\documentclass{article}

\pdfoutput=1

\usepackage[final]{neurips_2021}

\usepackage[utf8]{inputenc} 
\usepackage[T1]{fontenc}    
\usepackage{hyperref}       
\usepackage{url}            
\usepackage{booktabs}       
\usepackage{amsfonts}       
\usepackage{nicefrac}       
\usepackage{microtype}      

\usepackage{natbib}
\usepackage{algorithmic}
\usepackage{verbatim}
\usepackage{amsmath,amsthm,amsfonts,amssymb}
\usepackage[dvipsnames]{xcolor,colortbl}
\definecolor{pearThree}{HTML}{E74C3C}
\definecolor{pearcomp}{HTML}{B97E29}
\definecolor{pearDark}{HTML}{2980B9}
\definecolor{pearDarker}{HTML}{1D2DEC}
\usepackage{hyperref}
\hypersetup{
	colorlinks,
	citecolor=pearDark,
	linkcolor=pearThree,
	breaklinks=true,
	urlcolor=pearDarker}
\usepackage{cleveref}

\usepackage{mathrsfs}
\usepackage{enumitem}
\usepackage{bm}
\usepackage{multirow}
\usepackage{tikz}
\usetikzlibrary{positioning,angles,quotes}
\usepackage{booktabs}
\usepackage{array, makecell} 
\usepackage{graphicx}
\usepackage{subfigure}
\usepackage{caption}
\usepackage{thmtools}
\usepackage{thm-restate}
\usepackage{dsfont}
\usepackage{xspace}
\usepackage{empheq}
\usepackage{pifont}

\include{notations}
\usepackage[title]{appendix}
\allowdisplaybreaks[4]

\newtheorem{theorem}{Theorem}
\newtheorem{lemma}[theorem]{Lemma}
\newtheorem{corollary}[theorem]{Corollary}

\newtheorem{definition}[theorem]{Definition}

\newtheorem{assumption}[theorem]{Assumption}

\theoremstyle{definition}
\newtheorem{remark}{Remark}

\newenvironment{proofidea}{%
  \proof}{\endproof}

\newcommand{\II}{\mathbb{I}}

\usepackage[utf8]{inputenc}
\usepackage[T1]{fontenc}    
\usepackage{hyperref}       
\usepackage{url}            
\usepackage{booktabs}       
\usepackage{amsfonts}       
\usepackage{nicefrac}       
\usepackage{microtype}      

\usepackage{standalone}

\usepackage{mathtools}
\usepackage[linesnumbered,ruled,vlined]{algorithm2e}
\usepackage{amssymb}
\usepackage{bbm}
\usepackage{bm}
\usepackage{enumitem}
\usepackage{dirtytalk}
\usepackage{pifont}
\usepackage{tikz}
\usetikzlibrary{arrows.meta}
\usepackage{float}
\usepackage{cleveref}
\usepackage{hyperref}
\usepackage{soul}
\usepackage{dsfont}
\usepackage{amssymb}
\usepackage[thinc]{esdiff}
\usepackage{fontawesome}

\newcommand{\VISGO}{\texttt{VISGO}\xspace}
\newcommand{\ALGO}{\texttt{EB-SSP}\xspace}
\newcommand{\PHASE}{\texttt{PHASE}\xspace}

\newcommand{\smh}{s_h^m\xspace}
\newcommand{\smhp}{s_{h+1}^m\xspace}
\newcommand{\samh}{s_h^m, a_h^m\xspace}
\newcommand{\cmh}{c(s_h^m, a_h^m)\xspace}

\usepackage{natbib}

\definecolor{vert1}{RGB}{133,146,66} 
\definecolor{vert3}{RGB}{91,140,90} 
\definecolor{vert2}{RGB}{157,193,7} 
\definecolor{vert4}{RGB}{20,200,20} 

\newcommand{\calE}{\mathcal{E}}

\newcommand{\UCRLtwo}{{\small\textsc{UCRL2}}\xspace}

\newcommand{\epsVI}{\epsilon_{\textsc{{\tiny VI}}}\xspace}

\newcommand{\wtcalL}{\wt{\mathcal{L}}\xspace}

\newcommand\myeqi{\mathrel{\stackrel{\makebox[0pt]{\mbox{\normalfont\tiny (i)}}}{=}}}

\newcommand\myineqi{\mathrel{\stackrel{\makebox[0pt]{\mbox{\normalfont\tiny (i)}}}{\leq}}}
\newcommand\myineqii{\mathrel{\stackrel{\makebox[0pt]{\mbox{\normalfont\tiny (ii)}}}{\leq}}}
\newcommand\myineqiii{\mathrel{\stackrel{\makebox[0pt]{\mbox{\normalfont\tiny (iii)}}}{\leq}}}
\newcommand\myineqiv{\mathrel{\stackrel{\makebox[0pt]{\mbox{\normalfont\tiny (iv)}}}{\leq}}}
\newcommand\myineqv{\mathrel{\stackrel{\makebox[0pt]{\mbox{\normalfont\tiny (v)}}}{\leq}}}
\newcommand\myineqvi{\mathrel{\stackrel{\makebox[0pt]{\mbox{\normalfont\tiny (vi)}}}{\leq}}}
\newcommand\myineqvii{\mathrel{\stackrel{\makebox[0pt]{\mbox{\normalfont\tiny (vii)}}}{\leq}}}

\newcommand\mygei{\mathrel{\stackrel{\makebox[0pt]{\mbox{\normalfont\tiny (i)}}}{\ge}}}

\DeclareMathOperator*{\argmin}{arg\,min}

\newcommand{\cS}{\mathcal{S}}
\newcommand{\cA}{\mathcal{A}}
\newcommand{\SA}{\mathcal{S} \times \mathcal{A}}

\makeatletter
\newcommand\footnoteref[1]{\protected@xdef\@thefnmark{\ref{#1}}\@footnotemark}
\makeatother

\usepackage[colorinlistoftodos, textwidth=18mm]{todonotes}

\newcommand{\wt}[1]{\widetilde{#1}}

\newcommand{\wh}[1]{\widehat{#1}}
\newcommand{\ov}[1]{\overline{#1}}

\DeclarePairedDelimiter\abs{\lvert}{\rvert}%
\DeclarePairedDelimiter\norm{\lVert}{\rVert}%

\let\originalleft\left
\let\originalright\right
\renewcommand{\left}{\mathopen{}\mathclose\bgroup\originalleft}
\renewcommand{\right}{\aftergroup\egroup\originalright}

\usepackage{tcolorbox}
\newtcolorbox{redbox}{colback=red!5!white,colframe=red!75!black}
\newtcolorbox{bluebox}{colback=blue!5!white,colframe=blue!75!black}
\newtcolorbox{yellowbox}{colback=yellow!5!white,colframe=yellow!75!black}

\newcount\Comments  
\definecolor{darkgreen}{rgb}{0,0.5,0}
\definecolor{darkred}{rgb}{0.7,0,0}
\definecolor{teal}{rgb}{0.3,0.8,0.8}

\usetikzlibrary{matrix, positioning, arrows.meta}
\newlength{\myheight}
\tikzset{labels/.style={font=\sffamily\scriptsize},
    circuit/.style={draw,minimum width=2cm,minimum height=\myheight,very thick,inner sep=1mm,outer sep=0pt,cap=round,font=\sffamily\bfseries},
    triangle 45/.tip={Triangle[angle=45:8pt]}
}

\makeatletter
\def\@fnsymbol#1{\ensuremath{\ifcase#1\or *\or \dagger\or \ddagger\or
  \mathsection\or \mathparagraph\or \|\or \diamond \or **\or \dagger\dagger
  \or \ddagger\ddagger \else\@ctrerr\fi}}
\makeatother

\makeatletter
\newcommand{\printfnsymbol}[1]{%
  \textsuperscript{\@fnsymbol{#1}}%
}
\makeatother

\usepackage{tabularx}
\usepackage{footnote}
\usepackage{hhline}
\usepackage{multirow}

\definecolor{HighlightColor}{gray}{0.97}

\definecolor{colorm}{rgb}{0.2, 0.2, 0.6}
\definecolor{colorn}{rgb}{0.77, 0.12, 0.23}

\newcommand{\PreserveBackslash}[1]{\let\temp=\\#1\let\\=\temp}
\newcolumntype{C}[1]{>{\PreserveBackslash\centering}p{#1}}

\newcommand{\algocomment}[1]{\textcolor{gray}{\textbackslash\textbackslash \textit{#1}}}

\title{Stochastic Shortest Path: Minimax, Parameter-Free and Towards Horizon-Free Regret}

\author{%
 Jean Tarbouriech\thanks{equal contribution} \\
 Facebook AI Research \& Inria Lille\\
 \texttt{jean.tarbouriech@gmail.com} \\
  \And
 Runlong Zhou\printfnsymbol{1} \\
 Tsinghua University\\
 \texttt{zhourunlongvector@gmail.com}\\
  \AND
  \hspace{-5em}Simon S. Du\\
    \hspace{-5em}University of Washington \& Facebook AI Research \\ 
    \hspace{-5em}\texttt{ssdu@cs.washington.edu} \\
    \And
  \hspace{-3.5em}Matteo Pirotta \\
  \hspace{-3.5em}Facebook AI Research Paris\\
  \hspace{-3.5em}\texttt{pirotta@fb.com}
  \AND
  \hspace{0.5em}Michal Valko \\
  \hspace{0.5em}DeepMind Paris\\
  \hspace{0.5em}\texttt{valkom@deepmind.com} \\
  \And
  \hspace{2em}Alessandro Lazaric \\
  \hspace{2em}Facebook AI Research Paris\\
  \hspace{2em}\texttt{lazaric@fb.com}
}

\usepackage{minitoc}
\begin{document}

\maketitle

\doparttoc 
\faketableofcontents 

\begin{abstract}
  We study the problem of learning in the stochastic shortest path (SSP) setting, where an agent seeks to minimize the expected cost accumulated before reaching a goal state. We design a novel model-based algorithm \ALGO that carefully skews the empirical transitions and perturbs the empirical costs with an exploration bonus to induce an optimistic SSP problem whose associated value iteration scheme is guaranteed to converge. We prove that \ALGO achieves the minimax regret rate $\wt{O}(B_{\star} \sqrt{S A K})$, where $K$ is the number of episodes, $S$ is the number of states, $A$ is the number of actions, and $B_{\star}$ bounds the expected cumulative cost of the optimal policy from any state, thus closing the gap with the lower bound. Interestingly, \ALGO obtains this result while being parameter-free, i.e., it does not require any prior knowledge of~$B_{\star}$, nor of~$T_{\star}$, which bounds the expected time-to-goal of the optimal policy from any state. Furthermore, we illustrate various cases (e.g., positive costs, or general costs when an order-accurate estimate of~$T_{\star}$ is available) where the regret only contains a logarithmic dependence on~$T_{\star}$, thus yielding the first (nearly) horizon-free regret bound beyond the finite-horizon MDP setting. 
\end{abstract}

\captionsetup[table]{labelfont=small,font=small}

\section{Introduction}\label{sect_introduction}

Stochastic shortest path (SSP) is a goal-oriented reinforcement learning (RL) setting where the agent aims to reach a predefined goal state while minimizing its total expected cost \citep{bertsekas1995dynamic}. In particular, the interaction between the agent and the environment ends \textit{only} when (and if) the goal state is reached, so the length of an episode is not predetermined (nor bounded) and it is influenced by the agent's behavior. SSP includes both finite-horizon and discounted Markov Decision Processes (MDPs) as special cases. Moreover, many common RL problems can be cast under the SSP formulation, such as game playing (e.g., Atari games) or navigation (e.g., Mujoco mazes).

We study the online learning problem in the SSP setting (online SSP in short), where both the transition dynamics and the cost function are initially unknown and the agent interacts with the environment through multiple episodes. The learning objective is to achieve a performance as close as possible to the optimal policy $\pi^{\star}$, that is, the agent should achieve low \textit{regret} (i.e., the cumulative difference between the total cost accumulated across episodes by the agent and by the optimal policy). We identify three desirable properties for a learning algorithm in online~SSP.

\vspace{-0.05in}

\begin{itemize}[leftmargin=2mm]
 \setlength\itemsep{-0.2em}
    \item \textbf{Desired property 1: Minimax.} The information-theoretic lower bound on the regret is $\Omega(B_{\star} \sqrt{S A K})$ \citep{rosenberg2020near}, where $K$~is the number of episodes, $S$~is the number of states, $A$~is the number of actions, and $B_{\star}$~bounds the total expected cost of the optimal policy starting from any state (assuming for simplicity that $B_{\star} \geq 1$). 

\vspace{-0.04in}
\begin{center}
    \textit{An algorithm for online SSP is (nearly) minimax optimal if its regret is bounded by $\wt{O}(B_{\star} \sqrt{S A K})$, up to logarithmic factors and lower-order terms.}
\end{center}

\item \textbf{Desired property 2: Parameter-free.} Another relevant dimension is the amount of prior knowledge required by the algorithm. While the knowledge of $S$, $A,$ and the cost (or reward) range $[0, 1]$ is standard across regret-minimization settings (e.g., finite-horizon, discounted, average-reward), the complexity of learning in SSP problems may be linked to SSP-specific quantities such as $B_{\star}$ and $T_{\star}$, which denotes the expected time-to-goal of the optimal policy from any state. 

\vspace{-0.04in}
\begin{center}
\textit{An algorithm for online SSP is parameter-free if it relies neither on $T_{\star}$ nor $B_{\star}$ prior knowledge.}
\end{center}

\vspace{0.05in}

\item \textbf{Desired property 3: Horizon-free.} A core challenge in SSP is to trade off between minimizing costs and quickly reaching the goal state. This is accentuated when the instantaneous costs are small, i.e., when there is a mismatch between $B_{\star}$ and $T_{\star}$. Indeed, while $B_{\star} \leq T_{\star}$ always holds since the cost range is $[0, 1]$, the gap between the two may be arbitrarily large (see e.g., the simple example of App.\,\ref{app_Tstar}). The lower bound stipulates that the regret does depend on $B_{\star}$, while the ``time horizon'' of the problem, i.e., $T_{\star}$ should a priori not impact the regret, even as a lower-order term. 

\vspace{-0.04in}
\begin{center}
\textit{An algorithm for online SSP is (nearly)\,horizon-free if its regret depends only logarithmically on~$T_{\star}$.}
\end{center}

\vspace{-0.03in}

Our definition extends the property of so-called horizon-free bounds recently uncovered in finite-horizon MDPs with total reward bounded by~$1$ \citep{wang2020long,zhang2020reinforcement,zhang2021variance}. These bounds depend only logarithmically on the horizon~$H$, which is the number of time steps by which \textit{any} policy terminates. Such notion of horizon would clearly be too strong in the more general class of SSP, where some (even most) policies may never reach the goal, thus having unbounded time horizon. A more adequate notion of horizon in SSP is~$T_{\star}$, which bounds the \textit{expected} time of the \textit{optimal} policy to terminate the episode starting from any state.

\end{itemize}

\vspace{-0.05in}

Finally, while the previous properties focus on the learning aspects of the algorithm, another important consideration is computational efficiency. It is desirable that a learning algorithm has run-time complexity at most polynomial in~$K, S, A, B_{\star},$ and $ T_{\star}$. All existing algorithms for online SSP, including the one proposed in this paper, meet such requirement.



\newcommand{\wtO}{{\scalebox{0.92}{$\wt O$}}}
\newcommand{\Bstarexp}{{\scalebox{0.92}{$B_{\star}^{3/2}$}}}
\newcommand{\Kexp}{{\scalebox{0.92}{$K^{3/2}$}}}

\vspace{-0.04in}

\paragraph{Related Work.} Table \ref{tab:comparisons} reviews the existing work on online learning in SSP. The setting 
was first studied by \citet{tarbouriech2019no} who gave a parameter-free algorithm with a $\wtO (\Kexp)$ regret guarantee.  \citet{rosenberg2020near} then improved this result by deriving the first order-optimal algorithm with regret $\wtO (\Bstarexp S \sqrt{A K})$ in the parameter-free case and $\wtO (B_{\star} S \sqrt{A K})$ if $B_{\star}$ is known (to tune cost perturbation appropriately). 
Both approaches 
are model-optimistic,\footnote{\label{footnote1}We refer the reader to \citet{neu2020unifying} for details on the differences and interplay between model-optimistic and value-optimistic~approaches.} drawing inspiration from the ideas behind the \UCRLtwo algorithm \citep{jaksch2010near} for average-reward~MDPs. 


Concurrently to our work, \citet{cohen2021minimax} propose an algorithm for online SSP based on a black-box reduction from SSP to finite-horizon MDPs. It successively tackles finite-horizon problems with horizon set to $H = \Omega(T_{\star})$ and costs augmented by a terminal cost set to $c_H(s) = \Omega(B_{\star} \mathds{I}(s \neq g))$, where $g$ denotes the goal state. This finite-horizon construction guarantees that its optimal policy has a similar value function to the optimal policy in the original SSP instance up to a lower-order bias. Their algorithm comes with a regret bound of $O(B_{\star} \sqrt{ S A K} L + T_{\star}^4 S^2 A L^5)$, with $L = \log(K T_{\star} S A \delta^{-1})$ (with probability at least $1-\delta$). It achieves a nearly minimax-optimal rate, however it relies on both~$T_{\star}$ and $B_{\star}$ prior knowledge to tune the horizon and terminal cost in the reduction, respectively.\footnote{As mentioned by \citet[][Remark 2]{cohen2021minimax}, in the case of positive costs lower bounded by $c_{\min} > 0$, their knowledge of~$T_{\star}$ can be bypassed by replacing it with the upper bound $T_{\star} \leq B_{\star} / c_{\min}$. However, when generalizing from the $c_{\min}$ case to general costs with a perturbation argument, their regret guarantee worsens from $\wt{O}(\sqrt{K} + c_{\min}^{-4})$ to $\wt{O}(K^{4/5})$, because of the poor additive dependence on $c_{\min}^{-1}$.} 

\vspace{-0.03in}

Finally, all existing bounds contain lower-order dependencies either on $T_{\star}$ in the case of general costs, or on $B_{\star} / c_{\min}$ in the case of positive costs lower bounded by $c_{\min} > 0$ (note that $T_{\star} \leq B_{\star} / c_{\min}$, which is one of the reasons why $c_{\min}$ can show up in existing bounds). As such, no existing analysis satisfies horizon-free properties for online SSP.

\vspace{-0.04in}


\colorlet{colorast}{pearThree!80!black}

\begin{table}[t!]
		\centering
		\small
		\resizebox{0.99\columnwidth}{!}{%
			\renewcommand{\arraystretch}{2.2}
			\begin{tabular}{|c|c|c|c|c|}
                \hline 
				\textbf{Algorithm} &  \textbf{Regret} & \textbf{Minimax} & \textbf{Parameters} & \makecell{\textbf{Horizon-} \\ \textbf{Free}} \\
				\hhline{|=|=|=|=|=|}
				\citep{tarbouriech2019no}  &  $\wt{O}_{\scalebox{1}{$\scriptscriptstyle K$}}(K^{2/3})$
				&  No & None  & No \\
            	\hhline{|=|=|=|=|=|} \multirow{2}{*}{\citep{rosenberg2020near}} & $ \widetilde{O}\left( B_{\star} S \sqrt{ AK}+ T_{\star} ^{3/2} S^2 A \right)$ & No & $B_{\star}$ &  No   \\\hhline{~-|-|-|-}
				&  $ \widetilde{O}\left( B_{\star}^{3/2} S \sqrt{ AK}+ T_{\star} B_{\star} S^2 A \right)$  & No & None & No\\
				\hhline{|=|=|=|=|=|}	\multirow{1}{*}{\makecell{\citep{cohen2021minimax} \\ \textit{(concurrent work)}}} &  $ \widetilde{O}\left( B_{\star} \sqrt{S A K}+ T^4_{\star} S^2 A \right)$  & Yes & $B_{\star}$, $T_{\star}$ & No \\
				\hhline{|=|=|=|=|=|}
			 \rowcolor{HighlightColor}  & $ \widetilde{O}\left( B_{\star} \sqrt{S A K}+  B_{\star} S^2 A \right)$  & \textbf{Yes} & $B_{\star}$, $T_{\star}$ & \textbf{Yes} \\\hhline{~-|-|-|-}
			 \rowcolor{HighlightColor} & $  \widetilde{O}\left( B_{\star} \sqrt{S A K} +  B_{\star} S^2 A  + \frac{T_{\star}}{\textup{poly}(K)} \right)$ & \textbf{Yes} & $B_{\star}$ & No\textcolor{colorast}{$^{\ast}$} \\\hhline{~-|-|-|-}
		   \rowcolor{HighlightColor} \multirow{-2}{*}{This work} 	& $ \widetilde{O}\left( B_{\star} \sqrt{S A K}+ B_{\star}^3 S^3 A \right)$   & \textbf{Yes} & $T_{\star}$ & \textbf{Yes} \\\hhline{~-|-|-|-}
			\rowcolor{HighlightColor} & $ \widetilde{O}\left( B_{\star} \sqrt{S A K}+  B_{\star}^3 S^3 A + \frac{T_{\star}}{\textup{poly}(K)} \right)$   & \textbf{Yes} & \textbf{None} & No\textcolor{colorast}{$^{\ast}$} \\
				\hhline{|=|=|=|=|=|}
				Lower Bound  & $\Omega( B_{\star}\sqrt{S A K})$ & - & - & -\\
				\hline
			\end{tabular}
		}
		\vspace{0.1in}
		\caption{
			{\small Regret comparisons of algorithms for online SSP (we assume for simplicity that $B_{\star} \geq 1$). The notation
			$\widetilde{O}$ omits logarithmic factors and $\wt{O}_{\scalebox{1}{$\scriptscriptstyle K$}}$ only reports the dependence in $K$. 
			\textbf{Regret} is the performance metric of Eq.\,\ref{eq_ssp_regret_def}.
            \textbf{Minimax}: Whether the regret matches the $\Omega( B_{\star}\sqrt{S A K})$ lower bound \citep{rosenberg2020near}, up to logarithmic and lower-order terms.
            \textbf{Parameters}: The parameters  that the algorithm requires as input: either both $B_{\star}$ and $T_{\star}$, or one of them, or none (i.e., parameter-free).
			\textbf{Horizon-Free}: Whether the regret bound depends only logarithmically on~$T_{\star}$. \textcolor{colorast}{$^{\ast}$}If $K$ is known in advance, the additive term $T_{\star} / \textup{poly}(K)$ has a denominator that is polynomial in $K$, so it becomes negligible for large values of $K$ (if $K$ is unknown, the additive term is $T_{\star}$). See Sect.\,\ref{sect_main_results} for the full statements of our bounds.
			}
			\label{tab:comparisons}
		}
\end{table}

\paragraph{Contributions.} We summarize our main contributions as follows (see also Table \ref{tab:comparisons}):

\vspace{-0.05in}
\begin{itemize}[leftmargin=*]
    \item We propose \ALGO (\texttt{Exploration Bonus for SSP}), a new algorithm for online SSP. It introduces a value-optimistic scheme to efficiently compute optimistic policies for SSP, by both perturbing the empirical costs with an exploration bonus and slightly biasing the empirical transitions towards reaching the goal from \textit{each} state-action pair with positive probability. 
    Under these biased transitions, \textit{all} policies are in fact proper 
    (i.e., they eventually reach the goal with probability~$1$ starting from any state). We decay the bias over time in a way that it only contributes to a lower-order regret term. See Sect.\,\ref{sect_overview} for an overview of our algorithm and analysis. Note that \ALGO is \textit{not} based on a model-optimistic approach\footnoteref{footnote1} \citep{tarbouriech2019no, rosenberg2020near}, and it does \textit{not} rely on a reduction from SSP to finite-horizon \citep{cohen2021minimax} (i.e., we operate at the level of the non-truncated SSP model);
    \item \ALGO is the first algorithm to achieve the \textbf{minimax} regret rate of $\wt{O}(B_{\star} \sqrt{ S A K})$ while simultaneously being \textbf{parameter-free}: it does not require to know nor estimate $T_{\star}$, and it is able to bypass the knowledge of $B_{\star}$ at the cost of only logarithmic and lower-order contributions to the regret;
    \item \ALGO is the first algorithm to achieve \textbf{horizon-free} regret for SSP in various cases: i) positive costs, ii) no almost-sure zero-cost cycles, and iii) the general cost case when an order-accurate estimate of $T_{\star}$ is available (i.e., a value $\ov T_{\star}$ such that $\frac{T_{\star}}{\upsilon} \leq \ov T_{\star} \leq \lambda T_{\star}^{\zeta}$ for some unknown constants $\upsilon, \lambda, \zeta \geq 1$ is available). This property is especially relevant if $T_{\star}$ is much larger than $B_{\star}$, which can occur in SSP models with very small instantaneous costs. Moreover, \ALGO achieves its horizon-free guarantees while maintaining the minimax rate. For instance, under general costs when relying on $T_{\star}$ and $B_{\star}$, its regret is $\wt{O}(B_{\star} \sqrt{ S A K} + B_{\star} S^2 A)$.\footnote{We conjecture the optimal problem-independent regret in SSP to be $\wt{O}(B_{\star} \sqrt{ S A K} + B_{\star} S A)$ (by analogy with the conjecture of \citet{menard2021ucb} for finite-horizon MDPs), which shows the tightness of our bound up to an $S$ lower-order factor.} To the best of our knowledge, \ALGO yields the first set of (nearly) horizon-free bounds beyond the setting of finite-horizon~MDPs. 
\end{itemize}

\vspace{0.05in}

\textbf{Additional Related Work.} \textbf{\textit{~Planning in SSP:}} Early work by \citet{bertsekas1991analysis}, followed by \citep[e.g.,][]{bertsekas1995dynamic,bonet2007speed,kolobov2011heuristic,bertsekas2013stochastic, guillot2020stochastic}, examine the planning problem in SSP, i.e., how to compute an optimal policy when all parameters of the SSP model are known. Under mild assumptions, the optimal policy is deterministic and stationary and can be computed efficiently using standard planning techniques, e.g., value iteration, policy iteration or linear programming.

\textbf{\textit{Regret minimization in MDPs:}} The exploration-exploitation dilemma in tabular MDPs has been extensively studied in finite-horizon \citep[e.g.,][]{azar2017minimax,jin2018q, zanette2019tighter, efroni2019tight, simchowitz2019non, zhang2020almost, neu2020unifying, xu2021fine,menard2021ucb} and infinite-horizon \citep[e.g.,][]{jaksch2010near,bartlett2012regal, fruit2018efficient,wang2019q,jian2019exploration,wei2019model}.

\textbf{\textit{Other SSP-based settings:}} SSP with adversarial costs was investigated by \citet{rosenberg2020stochastic, chen2020minimax, chen2021finding}.
\footnote{A different line of work \citep[e.g.][]{neu2010online, neu2012adversarial,rosenberg2019onlinea, rosenberg2019onlineb, jin2020learning, jin2020simultaneously} studies finite-horizon MDPs with adversarial costs (sometimes called online loop-free SSP), where an episode ends after a fixed number of $H$ steps (as opposed to lasting as long as the goal is reached). 
} \citet{tarbouriech2021sample} study the sample complexity of SSP with a generative model, as a standard regret-to-PAC conversion may not hold in SSP (as opposed to finite-horizon). Exploration problems involving multiple goal states (i.e.,~multi-goal SSP or goal-conditioned RL) were analyzed by \citet{lim2012autonomous, tarbouriech2020improved}.

\section{Preliminaries}\label{sect_preliminaries}

An SSP problem is an MDP 
$M := \langle \mathcal{S}, \mathcal{A}, P, c, s_0, g \rangle$, where $\mathcal{S}$ is the finite state space with cardinality~$S$, $\mathcal{A}$ is the finite action space with cardinality $A$, and $s_0 \in \cS$ is the initial state. We denote by $g \notin \mathcal{S}$ the goal state, and we set $\cS' := \cS \cup \{g\}$ (thus $S' := S + 1$). Taking action $a$ in state $s$ incurs a cost drawn i.i.d.\,from a distribution on $[0, 1]$ with expectation $c(s, a)$, and the next state $s' \in \cS'$ is selected with probability $P(s' \vert s, a)$ (where $\sum_{s' \in \cS'} P(s' \vert s, a) = 1$). The goal state $g$ is absorbing and zero-cost, i.e., $P(g \vert g,a) = 1$ and $c(g,a)=0$ for any action $a$. 

For notational convenience, let $P_{s,a}:= P(\cdot \vert s,a)$, $P_{s,a,s'}:=P(s' \vert s,a)$. For any two vectors~$X,Y$ of size $S'$, we write their inner product as $X Y  := \sum_{s \in \cS'} X(s) Y(s)$, we denote by $X^2$ the vector $[X(1)^2, X(2)^2, \ldots, X(S')^2]^\top$, let $\norm{X}_{\infty} := \max_{s \in \cS'} \vert X(s) \vert$, $\norm{X}_{\infty}^{\neq g} := \max_{s \in \cS} \vert X(s) \vert$, and if~$X$ is a probability distribution on $\cS'$, then $\mathbb{V}(X, Y) := \sum_{s \in \cS'} X(s) Y(s)^2 - (\sum_{s \in \cS'} X(s) Y(s))^2$.

A stationary and deterministic policy $\pi : \cS \rightarrow \cA$ is a mapping from state $s$ to action $\pi(s)$. A policy $\pi$ is said to be proper if it reaches the goal with probability~$1$ when starting from any state in $\cS$ (otherwise it is improper). We denote by $\Pi_{\text{proper}}$ the set of proper, stationary and deterministic policies. 
We make the following basic assumption which ensures that the SSP problem is well-posed.

\begin{assumption}\label{asm_proper}
    There exists at least one proper policy, i.e., $\Pi_{\text{proper}} \neq \emptyset$.
\end{assumption}

\makeatletter
\newcommand*\mysizebis{%
   \@setfontsize\mysizebis{8.6}{9}%
}
\makeatother

The agent's objective is to minimize its expected cumulative cost incurred until the goal is reached. The value function (also called cost-to-go) of a policy $\pi$ and its associated $Q$-function are defined as 

\vspace{-0.2in}
\begin{mysizebis}
\begin{align*}
    V^{\pi}(s) := \lim_{T \rightarrow \infty} \mathbb{E}\bigg[ \sum_{t = 1}^{T} c_t(s_{t}, \pi(s_t)) \,\big\vert\,s_1 = s\bigg], ~\, Q^{\pi}(s,a) := \lim_{T \rightarrow \infty} \mathbb{E}\bigg[ \sum_{t = 1}^{T} c_t(s_{t}, \pi(s_t)) \,\big\vert\, s_1 = s, \pi(s_1) = a\bigg],
\end{align*}%
\end{mysizebis}%
where $c_t \in [0, 1]$ is the (instantaneous) cost incurred at time $t$ at state-action pair $(s_t, \pi(s_t))$, and the expectation is w.r.t.\,the random sequence of states generated by executing $\pi$ starting from state $s \in \cS$ (and taking action $a \in \cA$ in the second case). Note that $V^{\pi}$ may have unbounded components if $\pi$ never reaches the goal. For a proper policy $\pi$, $V^{\pi}(s)$ and $Q^{\pi}(s,a)$ are finite for any $s,a$. By definition of the goal, we set $V^{\pi}(g) = Q^{\pi}(g,a) = 0$ for all policies $\pi$ and actions $a$. Finally, we denote by~$T^{\pi}(s)$ the expected time that $\pi$ takes to reach $g$ starting at state $s$; in particular, if $\pi$ is proper then~$T^{\pi}(s)$ is finite for all $s$, yet if $\pi$ is improper there must exist at least one $s$ such that~$T^{\pi}(s) = \infty$.

Equipped with Asm.\,\ref{asm_proper} and an additional condition on improper policies defined below, one can derive important properties on the optimal policy $\pi^{\star}$ that minimizes the value function component-wise.

\begin{lemma}[\citealp{bertsekas1991analysis};\citealp{yu2013boundedness}] \label{lemma_wellposedproblem}
Suppose that Asm.\,\ref{asm_proper} holds and that for every improper policy $\pi'$ there exists at least one state $s \in \cS$ such that $V^{\pi'}(s) = + \infty$. Then the optimal policy $\pi^{\star}$ is stationary, deterministic, and proper. Moreover, $V^\star = V^{\scalebox{1.1}{$\scriptscriptstyle \pi^\star$}}$ is the unique solution of the optimality equations $V^\star = \mathcal{L} V^\star$ and $V^\star(s) < +\infty$ for any $s \in \mathcal{S}$, where for any vector $V \in \mathbb{R}^S$ the optimal Bellman operator $\mathcal{L}$ is defined as $\mathcal{L}V(s) := \min_{a \in \mathcal{A}} \big\{ c(s, a) + P_{s,a} V \big\}$. Also, the optimal $Q$-value, denoted by $Q^{\star}  = Q^{\scalebox{1.1}{$\scriptscriptstyle \pi^\star$}}$, is related to the optimal value function as follows: $Q^{\star}(s,a) = c(s,a) + P_{s,a} V^{\star}$ and $V^{\star}(s) = \min_{a \in \cA} Q^{\star}(s,a)$, for all $(s,a) \in \cS \times \cA$.
\end{lemma}

Since we will target the best proper policy, we will handle the second requirement of Lem.\,\ref{lemma_wellposedproblem} as follows \citep{bertsekas2013stochastic, rosenberg2020near}. First, the requirement is in particular verified if all instantaneous costs are strictly positive. To deal with the case of non-negative costs, we can introduce a small additive perturbation $\eta \in (0, 1]$ to all costs to yield a new (strictly positive) cost function $c_{\eta}(s,a) = \max \{ c(s,a), \eta\}$. In this cost-perturbed MDP, the conditions of Lem.\,\ref{lemma_wellposedproblem} hold so we get an optimal policy $\pi^{\star}_{\scalebox{1.1}{$\scriptscriptstyle \eta$}}$ that is stationary, deterministic and proper and has a finite value function $V^{\star}_{\scalebox{1.1}{$\scriptscriptstyle \eta$}}$. Taking the limit as $\eta \rightarrow 0$, we have that $\pi^{\star}_{\scalebox{1.1}{$\scriptscriptstyle \eta$}} \rightarrow \pi^{\star}$ and $V^{\star}_{\scalebox{1.1}{$\scriptscriptstyle \eta$}} \rightarrow V^{\scalebox{1.1}{$\scriptscriptstyle \pi^\star$}}$, where $\pi^{\star}$ is the optimal proper policy in the original model that is also stationary and deterministic, and $V^{\scalebox{1.1}{$\scriptscriptstyle \pi^\star$}}$ denotes its value function. This enables to circumvent the second condition of Lem.\,\ref{lemma_wellposedproblem} and only require Asm.\,\ref{asm_proper} to hold.

\paragraph{Learning formulation.} We consider the learning problem where the agent does not have any prior knowledge of the cost function $c$ or transition function $P$. Each episode starts at the initial state $s_0$ (the extension to any possibly unknown distribution of initial states is straightforward), and ends \textit{only} when the goal state $g$ is reached (note that this may never happen if the agent does not reach the goal). We evaluate the performance of the agent after $K$ episodes by its \textit{regret}, which is defined as
\begin{align}\label{eq_ssp_regret_def}
    R_K:= \sum_{k=1}^K \sum_{h=1}^{I^k} c_h^k ~ - K \cdot \min_{\pi \in \Pi_{\text{proper}}}  V^{\pi}(s_0),
\end{align}
where $I^k$ is the time needed to complete episode $k$ and $c_h^k$ is the cost incurred in the $h$-th step of episode $k$ when visiting $(s_h^k, a_h^k)$. If there exists $k$ such that $I^k$ is infinite, then we define~$R_K = \infty$. Throughout we denote the optimal proper policy by $\pi^{\star}$ and $V^{\star}(s) := V^{\pi^{\star}}(s) = \min_{\pi \in \Pi_{\text{proper}}}  V^{\pi}(s)$ and $Q^{\star}(s,a) := Q^{\pi^{\star}}(s,a) = \min_{\pi \in \Pi_{\text{proper}}}  Q^{\pi}(s,a)$ for all $(s,a)$. Let $B_{\star} > 0$ bound the values of~$V^{\star}$, i.e., $B_{\star} := \max_{s \in \cS} V^{\star}(s)$. Note that $Q^{\star}(s,a) \leq 1 + B_{\star}$. Also let $T_{\star} > 0$ bound the expected time-to-goal of the optimal policy, i.e., $T_{\star} := \max_{s \in \cS} T^{\pi^{\star}}(s)$. We see that $B_{\star} \leq T_{\star} < + \infty$.

\section{Main Algorithm}\label{sect_overview}

\newcommand{\wtP}{{\scalebox{0.92}{$\wt P$}}}
\newcommand{\whP}{{\scalebox{0.92}{$\wh P$}}}
\newcommand{\wtcalLs}{{\scalebox{0.92}{$\wtcalL$}}}

We introduce our algorithm \ALGO (\texttt{Exploration Bonus for SSP}) in Alg.\,\ref{algo}. It takes as input the state-action space $\cS \times \cA$ and confidence level $\delta \in (0, 1)$. For now it considers that an estimate $B$ such that $B \ge \max\{B_{\star}, 1\}$ is available, and we later handle the case of unknown $B_{\star}$ (Sect.\,\ref{subsect_unknownBstar} and App.\,\ref{sect_unknown_Bstar}). As explained in Sect.\,\ref{sect_preliminaries}, the algorithm enforces the conditions of Lem.\,\ref{lemma_wellposedproblem} to hold by adding a small cost perturbation $\eta \in [0, 1]$ (cf.\,lines \ref{line_input_eta}, \ref{line_cost_eta} in Alg.\,\ref{algo}) --- either $\eta = 0$ if the agent is aware that all costs are already positive, otherwise a careful choice of $\eta > 0$ is provided in Sect.\,\ref{sect_main_results}. 

Our algorithm builds on a value-optimistic approach by sequentially constructing optimistic lower bounds on the optimal $Q$-function and executing the policy that greedily minimizes them. 
Similar to the MVP algorithm of \citet{zhang2020reinforcement} designed for finite-horizon RL, we adopt the doubling update framework (first proposed by \citet{jaksch2010near}): whenever the number of visits of a state-action pair is doubled, the algorithm updates the empirical cost and transition probability of this state-action pair, and computes a new optimistic $Q$-estimate and optimistic greedy policy. Note that this slightly differs from MVP which waits for the end of its finite-horizon episode to update the policy. In SSP, however, having this delay may yield linear regret as the episode has the risk of never terminating under the current policy (e.g., if it is improper), which is why we perform the policy update instantaneously when the doubling condition is met.

\makeatletter
\newcommand*\mysizethree{%
   \@setfontsize\mysizethree{7.8}{9}%
}
\makeatother

\begin{algorithm}[t!]
        \small
        \DontPrintSemicolon
        \textbf{Input:} $\cS,\ s_0\in \cS,\ g\not\in \cS,\ \cA,\ \delta$. \\
        \textbf{Input:} an estimate $B$ guaranteeing $B \ge \max\{B_{\star}, 1\}$ (see Sect.\,\ref{subsect_unknownBstar} and App.\,\ref{sect_unknown_Bstar} if not available). \label{line-B-Bstar} \\
        \textbf{Optional input:} cost perturbation $\eta \in [0, 1]$. \label{line_input_eta} \\
        \textbf{Specify:} Trigger set $\mathcal{N} \leftarrow \{ 2^{j-1}\ :\ j=1,2,\ldots\}$. Constants $c_1 = 6,\ c_2 = 36,\ c_3 = 2 \sqrt{2},\ c_4 = 2 \sqrt{2}$. \\ 
        For $(s,a,s') \in \cS \times \cA \times \cS'$, set $N(s,a) \leftarrow 0; ~ n(s,a) \leftarrow 0; ~ N(s,a,s') \leftarrow 0; ~ \wh P_{s,a,s'} \leftarrow 0;$ ~$\theta(s,a) \leftarrow 0; ~ \wh c (s,a) \leftarrow 0; ~ Q(s,a) \leftarrow 0; ~ V(s) \leftarrow 0$. \\
        Set initial time step $t \leftarrow 1$ and trigger index $j \leftarrow 0$. \\
        \For{\textup{episode} $k=1, 2, \ldots$}
        {Set $s_t \leftarrow s_0$ \\
        \While{$s_t \neq g$}{
        Take action $a_t = \argmin_{a \in \cA} Q(s_t, a)$, incur cost $c_t$ and observe next state $s_{t+1} \sim P(\cdot \vert s_t, a_t)$.\\
        Set $(s,a,s',c) \leftarrow (s_t, a_t, s_{t+1}, \max\{c_t, \eta\})$ and $t \leftarrow t + 1$.\\
        Set $N(s,a) \leftarrow N(s,a) + 1$, $\theta(s,a) \leftarrow \theta(s,a) + c$,  $N(s,a,s') \leftarrow N(s,a,s') + 1$. \label{line_cost_eta} \\
        \If{$N(s,a) \in \mathcal{N}$ \label{line_update}}{
        \algocomment{~Update triggered: \textup{\VISGO} procedure.} \label{begin_trigger} \\
        Set $\wh{c}(s,a) \leftarrow \mathds{I}[N(s,a) \geq 2] \frac{2 \theta(s,a)}{N(s,a)} + \mathds{I}[N(s,a) = 1] \theta(s,a)$ and $\theta(s,a) \leftarrow 0$. \label{line_emp_costs} \\
        For $s' \in \mathcal{S}'$, set $\wh P_{s,a,s'} \leftarrow N(s,a,s') / N(s,a)$, $n(s,a) \leftarrow N(s,a)$, and $\wt P_{s,a,s'}$ as in Eq.\,\ref{eq_wtP}. \\
        Set $j \leftarrow j + 1$, $\epsVI \leftarrow 2^{-j} / (SA)$ and $i \leftarrow 0$, $V^{(0)} \leftarrow 0$, $V^{(-1)} \leftarrow +\infty$. \\
         For all $(s,a) \in \cS \times \cA$, set $n^+(s,a) \leftarrow \max\{n(s,a), 1\}$ and $\iota_{s,a} \gets \ln \left( \frac{12 S A S' [n^+ (s,a)]^2}{\delta} \right)$.\\
        \While{$\norm{ V^{(i)} - V^{(i-1)}}_{\infty} > \epsVI$}
        { 
        For all $(s,a) \in \cS \times \cA$, set 
        \vspace{-0.05in}
        \begin{align}
            &b^{(i+1)}(s,a) \, \leftarrow \, b(V^{(i)}, s, a), \quad \quad \textrm{\algocomment{~see Eq.\,\ref{eq_bonus} for bonus expression}} \label{update_bonus}  \\
            &Q^{(i+1)}(s,a) \, \leftarrow \, \max\big\{ \wh{c}(s,a) \,+ \, \wt P_{s,a} V^{(i)}  ~ - \,b^{(i+1)}(s,a), ~ 0 \big\}, \label{update_Q} \\
            &V^{(i+1)}(s) \, \leftarrow \, \min_a Q^{(i+1)}(s,a). \label{update_V} 
        \end{align} \\
        \vspace{-0.1in}
        Set $V^{(i+1)}(g) = 0$ and $i \leftarrow i+1$.
        }
        \vspace{-0.05in}
        Set $Q \leftarrow Q^{(i)}$, $V \leftarrow V^{(i)}$. \label{output_trigger}
        }
        }
        }
        \vspace{-0.06in}
\caption{Algorithm \ALGO}
\label{algo}
\end{algorithm}

The main algorithmic component lies in how to compute the $Q$-values (w.r.t.\,which the policy is greedy) when a doubling condition is met. To this purpose, we introduce a procedure called~\VISGO, for \texttt{Value Iteration with Slight Goal Optimism}. Starting with optimistic values $V^{(0)} = 0$, it iteratively computes $V^{(i+1)} = \wtcalLs V^{(i)}$ for a carefully defined operator $\wtcalLs$. It ends when a stopping condition is met, specifically once $\norm{V^{(i+1)} - V^{(i)}}_{\infty} \leq \epsVI$ for a precision level $\epsVI > 0$ (specified later), and it outputs the values $V^{(i+1)}$ (and $Q$-values $Q^{(i+1)}$). We now explain how we design $\wtcalLs$ and then provide some intuition. Let $\whP$ and $\wh c$ be the current empirical transition probabilities and costs, and let $n(s,a)$ be the current number of visits to state-action pair $(s,a)$ (and $n^+(s,a) = \max\{n(s,a), 1\}$). We first define transition probabilities $\wtP$ that are slightly skewed towards the goal w.r.t.\,$\whP$, as follows
%
\begin{align}\label{eq_wtP}
    \wt P_{s,a,s'} := \frac{n(s,a)}{n(s,a)+1} \wh P_{s,a,s'} + \frac{\mathds{I}[s'=g]}{n(s,a)+1}.
\end{align}
Given the estimate $B$, specific positive constants $c_1, c_2, c_3, c_4$ and a state-action dependent logarithmic term~$\iota_{s,a}$, we then define the exploration bonus function, for any state-action pair $(s,a) \in \SA$ and vector $V \in \mathbb{R}^{S'}$ such that $V(g) = 0$, as follows
\begin{align}\label{eq_bonus}
    b(V,s,a) := \max\Big\{ c_1 \sqrt{ \frac{\mathbb{V}(\wt P_{s,a}, V) \iota_{s,a} }{n^+(s,a)}} , \, c_2 \frac{B \iota_{s,a}}{n^+(s,a)}  \Big\} + c_3 \sqrt{\frac{\wh{c}(s,a) \iota_{s,a}}{n^+(s,a)}} + c_4 \frac{B \sqrt{S' \iota_{s,a}}}{n^+(s,a)}.
\end{align}%
Note that the last term in Eq.\,\ref{eq_bonus} accounts for the skewing of $\wtP$ w.r.t.\,$\whP$ (see Lem.\,\ref{lemma_properties_wtP}). Given the transitions $\wtP$ and exploration bonus $b$, we are ready to define the operator $\wtcalLs$ as
\begin{align}\label{eq_wtcalL_operator}
    \wtcalL V(s) := \max  \Big\{   \min_{a \in \cA}\big\{ \wh{c}(s,a) + \wt P_{s,a} V - b(V,s,a) \big\} , \, 0 \Big\}.
\end{align}
We see that $\wtcalLs$ promotes optimism in two different ways: 
\begin{itemize}[leftmargin=.2in,topsep=-1.5pt,itemsep=1pt,partopsep=0pt, parsep=0pt]
    \item[\textbf{(i)}] On the empirical cost function $\wh c$, via the bonus $b$ (Eq.\,\ref{eq_bonus}) that intuitively lowers the costs to $\wh c - b$;
    \item[\textbf{(ii)}] On the empirical transition function $\whP$, via the transitions $\wtP$ (Eq.\,\ref{eq_wtP}) that slightly bias $\whP$ with the addition of a non-zero probability of reaching the goal from \textit{every} state-action pair.
\end{itemize}
While the first feature~\textbf{(i)} is standard in finite-horizon approaches, the second~\textbf{(ii)} is SSP-specific, and is required to cope with the fact that the empirical model $\whP$ may \textit{not} admit any proper policy, meaning that executing value iteration for SSP on $\whP$ may diverge. Our simple transition skewing actually guarantees that \textit{all} policies are proper in $\wtP$, for any fixed and bounded cost function.\footnote{In fact this transition skewing implies that an SSP problem defined on $\wtP$ is equivalent to a discounted RL problem, with a varying state-action dependent discount factor. Also note that for different albeit mildly related purposes, a perturbation trick is sometimes used in regret minimization for average-reward MDPs \citep[e.g.,][]{ fruit2018efficient,jian2019exploration}, where a non-zero probability of reaching an arbitrary state at each state-action is added to guarantee that all policies are unichain and that value iteration variants 
nearly converge in finite-time.} 
By decaying the extra goal-reaching probability inversely with $n(s,a)$, we can tightly control the gap between $\wtP$ and $\whP$ and ensure that it only accounts for a lower-order regret term (cf.\,last term of Eq.\,\ref{eq_bonus}). 

Equipped with these two sources of optimism, as long as $B \geq B_{\star}$, we are able to prove that a \VISGO procedure verifies the following two key properties:
\begin{itemize}[leftmargin=.2in,topsep=-1.5pt,itemsep=1pt,partopsep=0pt, parsep=0pt]
    \item[\textbf{(1)}] \textbf{\textit{Optimism:}} \VISGO outputs an optimistic estimator of the optimal $Q$-function at each iteration step, i.e., $Q^{(i)}(s,a) \leq Q^{\star}(s,a), \forall i \geq 0$,
    \item[\textbf{(2)}] \textbf{\textit{Finite-time near-convergence:}} \VISGO terminates within a finite number of iteration steps (note that the final iterate $V^{(j)}$ approximates the fixed point of $\wtcalLs$ up to an error scaling with $\epsVI$). 
\end{itemize}
To satisfy~\textbf{(1)}, we derive similarly to MVP \citep{zhang2020reinforcement} a \textit{monotonicity} property for the operator $\wtcalLs$, which is achieved by carefully tuning the constants $c_1, c_2, c_3, c_4$ in the bonus of Eq.\,\ref{eq_bonus}. On the other hand, the requirement~\textbf{(2)} is SSP-specific, since it is not needed in finite-horizon where value iteration requires exactly $H$ backward induction steps. \textit{Without} bonuses, the design of $\wtP$ would have directly entailed that $\wtcalLs$ is contractive and convergent \citep{bertsekas1995dynamic}. However, our variance-aware exploration bonuses introduce a subtle correlation between value iterates (i.e., $b$ depends on $V$ in Eq.\,\ref{eq_bonus}), which leads to a cost function that varies across iterates. By directly analyzing $\wtcalLs$, we establish that it is contractive with modulus $\rho := 1 - \nu < 1$, where $\nu := \min_{s,a} \wtP_{s,a,g} > 0$. This \textit{contraction} property guarantees a polynomially bounded number of iterations before terminating, i.e.,~\textbf{(2)}. 

\vspace{0.05in}
\begin{remark}[Computational complexity]
    Denote by $T$ the accumulated time within the $K$ episodes. By the stopping condition $\vert \vert V^{(i+1)} - V^{(i)} \vert \vert_{\infty} \leq \epsilon_{\textup{\tiny{VI}}}$, the choice of $\epsilon_{\textup{\tiny{VI}}}$ and the $\rho$-contraction of the operator $\wtcalLs$ with $\rho \leq 1 - 1/T$, any \VISGO procedure is guaranteed to stop at an iteration $i \leq \log(\max\{B_{\star},1\} / \epsilon_{\textup{\tiny{VI}}}) / (1-\rho) = O(T S A \log(T \max\{B_{\star},1\} ))$. Since there are at most $O(S A \log T)$ \VISGO procedures, we see that the total computational complexity of \ALGO is near-linear in $T$, where $T$ is bounded polynomially w.r.t.\,$K$ as shown in the various cases of Sect.\,\ref{subsect_bounding_T} (see App.\,\ref{app_comput_complex} for details). Therefore \ALGO is computationally efficient. Note that its $\textrm{poly}(K)$ complexity is a limitation shared by all existing parameter-free algorithms in SSP. On the other hand, the algorithm of \citet{cohen2021minimax} can obtain a $\log(K)$ computational complexity but only with $T_{\star}$ prior knowledge: without it, using the upper bound $T_{\star} \leq B_{\star} / c_{\min}$, where $c_{\min}^{-1}$ becomes $\textrm{poly}(K)$ when applying the cost perturbation trick, also leads to $\textrm{poly}(K)$ complexity. It is an interesting open question whether it is possible in SSP to have $\log(K)$ computational complexity while staying parameter-free. \label{remark_comput_complexity_main}
\end{remark}


\section{Main Results} \label{sect_main_results}


Besides ensuring the computational efficiency of \ALGO, the properties of \VISGO lay the foundations for our regret analysis (App.\,\ref{sect_main_proof}) to yield the following general guarantee.

\begin{theorem}\label{key_lemma_Tdependent}
Assume that $B \ge \max\{B_{\star}, 1\}$ and that the conditions of Lem.\,\ref{lemma_wellposedproblem} hold. Then with probability at least $1-\delta$ the regret of \textup{\ALGO} (Alg.\,\ref{algo} with $\eta = 0$) can be bounded by 
\begin{align*}
    R_K = O\left( \sqrt{ (B_{\star}^2 + B_{\star}) S A K } \log\left( \frac{\max\{B_{\star},1\} S A T }{\delta} \right) + B S^2 A \log^2\left( \frac{\max\{B_{\star},1\} S A T }{\delta} \right) \right),
\end{align*}
with $T$ the accumulated time within the $K$ episodes.
\end{theorem}

Thm.\,\ref{key_lemma_Tdependent} is an intermediate result for the regret of \ALGO, as it depends on the \textit{random and possibly unbounded} total number of steps $T$ executed over $K$ episodes, it requires the possibly restrictive second condition of Lem.\,\ref{lemma_wellposedproblem}, and it relies on the parameter $B$ being properly tuned. Nonetheless, it already displays interesting properties: \textbf{1)} The dependence on $T$ is limited to logarithmic terms; \textbf{2)} The parameter $B$ only affects the lower order term, while the main order term naturally scales with the exact range $B_\star$; \textbf{3)} Up to dependence on $T$, the main order term displays minimax optimal dependencies on $B_\star$, $S$, $A$, and $K$.

Throughout the rest of the section, we consider for ease of exposition that $B_{\star} \geq 1$.\footnote{Otherwise, all later bounds hold by replacing $B_{\star}$ with $\max\{ B_{\star}, 1\}$, except for the $B_{\star}$ factor in the leading term that becomes $\sqrt{B_{\star}}$. This matches the lower bound of \cite{cohen2021minimax} of $\Omega(\sqrt{B_{\star} S A K})$ for $B_{\star} < 1$.} For simplicity, when tuning the cost perturbations later, we assume as in prior works \citep[e.g.,][]{rosenberg2020near, chen2020minimax, chen2021finding} that the total number of episodes $K$ is known to the agent (this knowledge can be eliminated with the standard doubling~trick).

\textbf{\textit{Proof idea of Thm.\,\ref{key_lemma_Tdependent}.}} We decompose the regret into three parts: $X_1$ (error on the optimistic $V$-values), $X_2$ (Bellman error) and $X_3$ (cost estimation error), and among them the major part is~$X_2$. Later, $X_1$ and $X_2$ introduce the intermediate quantities $X_4$ (variance of the optimistic $V$-values) and $X_5$ (variance of the differences $V^{\star} - V$), which are bounded using the recursion technique generalized from \citet{zhang2020reinforcement}, where we normalize the values by~$1/B_{\star}$ to avoid an exponential blow-up in the recursions. At a high-level, the key idea is to calculate errors of different orders, $F(1), F(2), \ldots, F(d), \ldots$ (see Lem.\,\ref{lemma_bound_X4} and \ref{lemma_bound_X5}), and recursively bound $F(i)$'s variance by a sublinear function of $F(i + 1)$. Throughout the proof, we bound quantities by solving inequalities that contain the unknown quantities on both sides, such as $X_3 \leq \widetilde{O}(\sqrt{X_3 + C_{K}})$ or $X_2 \leq \widetilde{O}(\sqrt{X_2 + C_{K}})$, where the random variable $C_K$ denotes the cumulative cost over the $K$ episodes. Indeed, the analysis at each time step $t$ brings out the instantaneous cost $c_t$ and it is important to combine them so that we can make $C_K$ appear explicitly. Ultimately, we obtain a regret bound scaling as $R_K = \wt{O}((\sqrt{B_{\star}} + 1) \sqrt{S A C_K})$. Since the regret in SSP is defined as $R_K = C_K - K V^{\star}(s_0)$, we obtain a quadratic inequality in $C_K$, which we solve to 
get the $\wt{O}(\sqrt{(B_{\star}^2 + B_{\star}) S A K})$ regret bound. 

\subsection{Regret Bounds for $B = B_{\star}$}\label{subsect_bounding_T}


First we assume that $B = B_{\star}$ (i.e., the agent has prior knowledge of $B_{\star}$) and we instantiate the regret achieved by \ALGO under various conditions on the SSP model.


$\Box$ \, \textbf{\textit{Positive Costs.}} We first focus on the case of positive costs.

\begin{assumption}\label{asm1_cminpositive}
    All costs are lower bounded by a constant $c_{\min} > 0$ which is unknown to the agent.
\end{assumption}

Asm.\,\ref{asm1_cminpositive} guarantees that the conditions of Lem.\,\ref{lemma_wellposedproblem} hold. Moreover, denoting by $C$ the cumulative cost over $K$ episodes, the total time satisfies $T \leq C / c_{\min}$. By simplifying the bound of Thm.\,\ref{key_lemma_Tdependent} as $C \le B_{\star} K + R_K \le O(B_{\star} S^2 A K \cdot \sqrt{B_{\star} T S A / \delta})$, we loosely obtain that $T = O(B_{\star}^3 S^5 A^3 K^2 / (c_{\min}^2 \delta))$.

\begin{corollary}\label{thm_regret_bound_cmin_positive}
Under Asm.\,\ref{asm1_cminpositive}, running \textup{\ALGO} (Alg.\,\ref{algo}) with $B = B_{\star}$ and $\eta = 0$ gives the following regret bound with probability at least $1-\delta$
\begin{align*}
    R_K = O\left( B_{\star} \sqrt{ S A K } \log\left( \frac{K B_{\star} S A }{c_{\min} \delta} \right) + B_{\star} S^2 A \log^2\left( \frac{K B_{\star} S A }{c_{\min} \delta} \right) \right).
\end{align*}
\end{corollary}

The bound of Cor.\,\ref{thm_regret_bound_cmin_positive} only depends polynomially on $K, S, A, B_{\star}$. We note that $T_{\star} \leq B_{\star} / c_{\min}$ and that this upper bound only appears in the logarithms. Under positive costs, the regret of \ALGO is thus (nearly) \textbf{minimax} and \textbf{horizon-free}. Furthermore, in App.\,\ref{app_alt_asm_no_cycle_zero_cost} we introduce an alternative assumption on the SSP problem (which is weaker than Asm.\,\ref{asm1_cminpositive}) that considers that there are no almost-sure zero-cost cycles. In this case also, the regret of \ALGO is (nearly) minimax and horizon-free.

$\Box$ \, \textbf{\textit{General Costs and $T_{\star}$ Unknown.}} Now we handle the case of non-negative costs, with no~assumption other than Asm.\,\ref{asm_proper}. We use a cost perturbation argument to generalize the results from positive to general costs (similar to \citet{tarbouriech2019no, rosenberg2020near}). As reviewed in Sect.\,\ref{sect_preliminaries}, this circumvents the second condition of Lem.\,\ref{lemma_wellposedproblem} (which holds in the cost-perturbed MDP) and target the optimal proper policy in the original MDP up to a bias scaling with the cost perturbation. Indeed, running \ALGO with costs $c_{\eta}(s,a) \leftarrow \max\{c(s,a), \eta\}$ for $\eta \in (0, 1]$ gives the bound of Cor.\,\ref{thm_regret_bound_cmin_positive} with $c_{\min} \leftarrow \eta$, $B_{\star} \leftarrow B_{\star} + \eta T_{\star}$ and an additive bias of $\eta T_{\star} K$. We then pick $\eta$ to balance these terms.

\begin{corollary}\label{thm_regret_bound_cmin_zero_unknown_Tstar}
Let $L := \log \left( K T_{\star} S A \delta^{-1} \right)$. Running \textup{\ALGO} (Alg.\,\ref{algo}) with $B = B_{\star}$ and $\eta = K^{-n} $ for \textbf{any} choice of constant $n > 1$ gives the following regret bound with probability at least $1-\delta$
\begin{align*}
    R_K = O\Big( n B_{\star} \sqrt{ S A K } L ~+~ \frac{ T_{\star}}{K^{n - 1}} + \frac{ n T_{\star} \sqrt{SA} L}{K^{n - 1/2}}   ~+~ n^2 B_{\star} S^2 A L^2 \Big).
\end{align*}
\end{corollary}
%
This bound can be decomposed as (i) a $\sqrt{K}$ leading term and (ii) an additive term that depends on $T_{\star}$ and vanishes as $K \rightarrow + \infty$ (we omit the last term that does not depend polynomially on either $K$ or $T_{\star}$). Note that the second term (ii) can be made as small as possible by increasing the choice of exponent $n$ in the cost perturbation, at the cost of the multiplicative constant $n$ in (i). Equipped only with Asm.\,\ref{asm_proper}, the regret of \ALGO is thus (nearly) \textbf{minimax}, and it may be dubbed as \textit{horizon-vanishing} when $K$ is given in advance, insofar as it contains an additive term that depends on~$T_{\star}$ and that becomes negligible for large values of $K$ (if $K$ is unknown in advance, the application of the doubling trick yields an additive term (ii) scaling as $T_{\star}$). We now show that the trade-off between (i) and (ii) can be resolved with loose knowledge of $T_{\star}$ and leads to a horizon-free bound.



$\Box$ \, \textbf{\textit{General Costs and Order-Accurate Estimate of $T_{\star}$ Available.}} We now consider that an order-accurate estimate of $T_{\star}$ is available. It may be a constant lower-bound approximation away from~$T_{\star}$, or a polynomial upper-bound approximation away from~$T_{\star}$.

\begin{assumption}\label{asm_priorknowledge_Tstar}
    The agent has prior knowledge of a quantity $\ov T_{\star}$ that verifies $\frac{T_{\star}}{\upsilon} \leq \ov T_{\star} \leq \lambda T_{\star}^{\zeta}$ for some unknown constants $\upsilon, \lambda, \zeta \geq 1$. (Note that $\upsilon = \lambda = \zeta = 1$ when $T_{\star}$ is known.)
\end{assumption}

We now tune the cost perturbation $\eta$ using~$\ov T_{\star}$. Specifically, selecting $\eta := (\ov T_{\star} K)^{-1}$ ensures that the bias satisfies $\eta T_{\star} K \leq \upsilon = O(1)$. We thus obtain the following guarantee (see App.\,\ref{app_cor_explicit_dep_constants_estimate_Tstar} for the explicit dependencies on the \textit{constant} terms $\upsilon, \lambda, \zeta$ which only appear as multiplicative and additive factors).

\begin{corollary}\label{thm_regret_bound_cmin_zero_priorknowledge_Tstar}
Under Asm.\,\ref{asm_priorknowledge_Tstar}, running \textup{\ALGO} (Alg.\,\ref{algo}) with $B = B_{\star}$ and $\eta = (\ov T_{\star} K)^{-1}$ gives the following regret bound with probability at least $1-\delta$
\begin{align*}
   R_K = O\left( B_{\star} \sqrt{ S A K } \log\left( \frac{K T_{\star} S A }{ \delta} \right) + B_{\star} S^2 A \log^2\left( \frac{K T_{\star} S A }{ \delta} \right) \right).
\end{align*}
\end{corollary}

This bound depends polynomially on $K, S, A, B_{\star}$, and only logarithmically on $T_{\star}$. Thus under general costs with an order-accurate estimate of $T_{\star}$, \ALGO's regret is (nearly) \textbf{minimax} and \textbf{horizon-free}.


We can compare Cor.\,\ref{thm_regret_bound_cmin_zero_priorknowledge_Tstar} with the concurrent result of \citet{cohen2021minimax}. Their regret bound scales as $O(B_{\star} \sqrt{ S A K} L + T_{\star}^4 S^2 A L^5)$ with $L = \log(K T_{\star} S A \delta^{-1})$ under the assumptions of known~$T_{\star}$ and~$B_{\star}$ (or tight upper bounds of them), which imply that the conditions of Cor.\,\ref{thm_regret_bound_cmin_zero_priorknowledge_Tstar} hold. The bound of Cor.\,\ref{thm_regret_bound_cmin_zero_priorknowledge_Tstar} is strictly tighter, since it always holds that~$B_{\star} \leq T_{\star}$ and the gap between the two may be arbitrarily large (see e.g., App.\,\ref{app_Tstar}), especially when some instantaneous costs are very small.


\subsection{Regret Bounds for Unknown $B_{\star}$ with Parameter-Free \ALGO} \label{subsect_unknownBstar}


We now introduce a parameter-free version of \ALGO that bypasses the requirement of $B \geq B_{\star}$ (line~\ref{line-B-Bstar} of Alg.\,\ref{algo}). Note that 
the challenge of not knowing the range of the optimal value function does not appear in finite-horizon MDPs, where the bound $H$ (or $1$ for \citet{zhang2020reinforcement}) is assumed to be known to the agent. 
In SSP, if the agent does not have a valid estimate $B \geq B_{\star}$, then it may design an under-specified exploration bonus which cannot guarantee optimism. The case of unknown~$B_{\star}$ is non-trivial: it appears impossible to properly estimate $B_{\star}$ (since some states may never be visited) and it is unclear how a standard doubling trick may be used.
\footnote{Note that \citet{jian2019exploration} raised an open question whether it is possible to design an exploration bonus strategy in a setting where no prior knowledge of the ``optimal range'' is available. Indeed their approach in average-reward MDPs relies on prior knowledge of an upper bound on the optimal bias span.}

\newcommand{\wtB}{{\scalebox{0.9}{$\wt B$}}}

Parameter-free \ALGO initializes a proxy $\wtB = 1$ and increases it over the learning interaction according to a carefully defined schedule. We need to ensure that the proxy $\wtB$ does not remain below $B^{\star}$ for too long, since in this case, the regret may keep growing linearly. Thus, our \textit{first condition} to increase $\wtB$ is whenever a new episode $k$ begins, specifically we set $\wtB \leftarrow \max\{ \wtB, \, \sqrt{k} / (S^{3/2} A^{1/2}) \}$, which ensures that $\wtB \geq B^{\star}$ for large enough episodes. However, this is not enough: indeed notice that when $\wtB < B^{\star}$, the agent may never reach the goal and thus get \textit{stuck} in the episode, so we cannot exclusively rely on the end of an episode as a trigger for increasing $\wtB$. Our \textit{second condition} to increase $\wtB$ is to set $\wtB \leftarrow 2 \wtB$ whenever the cumulative cost exceeds a carefully defined threshold (that depends on $\wtB$, $S$, $A$,~$\delta$ and the current episode and time indexes $k$ and $t$, which are all computable quantities). Since the regret is upper bounded by the cumulative cost, this second condition prevents the learner from accumulating too large regret when $\wtB < B^{\star}$. Finally, we introduce a \textit{third condition} to increase $\wtB$ in order to ensure the computational efficiency, since \VISGO may diverge when $\wtB < B^{\star}$ (specifically, we track the range of the value $V^{(i)}$ at each \VISGO iteration $i$ and if $\norm{V^{(i)}}_{\infty} > \wtB$, then we terminate \VISGO and increase $\wtB \leftarrow 2 \wtB$). At a high-level, the analysis of the scheme proceeds as follows: we bound the regret by the cumulative cost when $\wtB < B^{\star}$ (first regime), and by the regret bound of Thm.\,\ref{key_lemma_Tdependent} when $\wtB \geq B^{\star}$ (second regime). Note that this two-regime decomposition is only implicit (i.e., at the level of analysis), since the agent is unable to know in which regime it is (since $B^{\star}$ is unknown). The full pseudo-code and analysis of parameter-free \ALGO is deferred to App.\,\ref{sect_unknown_Bstar}.

\newcommand{\thmkeyunknownBstar}{
    Assume the conditions of Lem.\,\ref{lemma_wellposedproblem} hold. Then with probability at least $1-\delta$ the regret of parameter-free \textup{\ALGO} (Alg.\,\ref{algobstar}, App.\,\ref{sect_unknown_Bstar}) can be bounded~by
\begin{align*}
    R_K 
    &= O \left( R^{\star}_K \log\left( \frac{B_{\star} S A T }{\delta} \right) + B_{\star}^3 S^3 A \log^{3}\left( \frac{B_{\star} S A T }{\delta} \right) \right),
\end{align*}

where $T$ is the cumulative time within the $K$ episodes and $R^{\star}_K$ bounds the regret after $K$ episodes of~\textup{\ALGO} in the case of known $B_{\star}$ (i.e., the bound of Thm.\,\ref{key_lemma_Tdependent} with $B = B_{\star}$).
}
\begin{theorem}[Extension of Theorem \ref{key_lemma_Tdependent} to unknown $B_{\star}$]
  \label{thm_key_unknownBstar}
  \thmkeyunknownBstar
\end{theorem}


Thm.\,\ref{thm_key_unknownBstar} implies that we can remove the condition of $B \ge \max\{B_{\star}, 1\}$ in Thm.\,\ref{key_lemma_Tdependent}, i.e., we make the statement \textbf{parameter-free}. Hence, \textit{all} the regret bounds from Sect.\,\ref{subsect_bounding_T} in the case of known~$B_{\star}$ (i.e., Cor.\,\ref{thm_regret_bound_cmin_positive}, \ref{thm_regret_bound_cmin_zero_unknown_Tstar}, \ref{thm_regret_bound_cmin_zero_priorknowledge_Tstar}, \ref{thm_regret_bound_nocycle_zerocost}) still hold up to additional logarithmic and lower-order terms when $B_{\star}$ is unknown.

\section{Conclusion}\label{sect_ccl}

We introduced \ALGO, the first algorithm for online SSP to be \textit{simultaneously} nearly minimax-optimal and parameter-free (i.e., it does not need to know $T_{\star}$ nor $B_{\star}$). Also in various cases its regret is nearly horizon-free with only a \textit{logarithmic} dependence on $T_{\star}$, thus exponentially improving over existing bounds w.r.t.\,the dependence on $T_{\star}$, which may be arbitrarily larger~than~$B_{\star}$~when~instantaneous costs are small. The horizon-free property is perhaps even more meaningful in the goal-oriented setting than in finite-horizon MDPs (with total reward bounded by~$1$) 
\citep[e.g.,][]{wang2020long, zhang2020reinforcement,zhang2021variance}, as we do \textit{not} impose a known constraint on the total cost of a trajectory.



An interesting question raised by our paper is whether it is possible to simultaneously achieve minimax, parameter-free and horizon-free regret for SSP under general costs. Another direction can be to build on our approach (e.g., the \VISGO procedure) to derive tight sample complexity bounds in SSP, which as explained by \citet{tarbouriech2021sample} do not directly ensue from regret guarantees.

\section*{Acknowledgement}
SSD gratefully acknowledges the funding from NSF Award’s IIS-2110170 and DMS-2134106.

\bibliographystyle{plainnat}
\bibliography{bibliography.bib}




\newpage

\appendix

\part{Appendix}

\parttoc

\vspace{0.2in}

\section{$T_{\star}$ can be arbitrarily larger than $B_{\star},\ S,\ A$} \label{app_Tstar}

\begin{minipage}{0.55\linewidth}


Here we provide a simple illustration that the inequality $B_{\star} \leq T_{\star}$ may be arbitrarily loose, which shows that scaling with $T_{\star}$ can be much worse than scaling with $B_{\star}$. Recall that~$B_{\star}$~bounds the total expected cost of the optimal policy starting from any state, and~$T_{\star}$ bounds the expected time-to-goal of the optimal policy from any state.

\vspace{0.08in}

Let us consider an SSP instance whose optimal policy induces the absorbing Markov chain depicted in Fig.\,\ref{toy_figure}. It is easy to see that $B_{\star} = 1$ and that $T_{\star} = \Omega(S\ p_{\min}^{-1})$. Hence, the gap between $B_{\star}$ and $T_{\star}$ can grow arbitrarily large as $p_{\min} \rightarrow 0$. 

\vspace{0.08in}

This simple example illustrates the benefit of having a bound that is (nearly) \textit{horizon-free} (cf.\,desired property~3 in Sect.\,\ref{sect_introduction}). Indeed, a bound that is not horizon-free scales polynomially with~$T_{\star}$ and thus with~$p_{\min}^{-1}$, which may be arbitrarily large if~$p_{\min} \rightarrow 0$. In contrast, a horizon-free bound only scales logarithmically with~$p_{\min}^{-1}$ and can therefore be much tighter.

\end{minipage}\hfill
\begin{minipage}{0.38\linewidth}
\centering
\begin{tikzpicture}[relative]
                        \begin{scope}[scale=0.6]
	\tikzset{VertexStyleDashed/.style = {draw, dashed,
									shape          = circle,
	                                text           = black,
	                                inner sep      = 2pt,
	                                outer sep      = 0pt,
	                                minimum size   = 20 pt}}
	\tikzset{VertexStyle/.style = {draw, 
									shape          = circle,
	                                text           = black,
	                                inner sep      = 2pt,
	                                outer sep      = 0pt,
	                                minimum size   = 20 pt}}
	\tikzset{VertexStyle2/.style = {shape          = circle,
	                                text           = black,
	                                inner sep      = 2pt,
	                                outer sep      = 0pt,
	                                minimum size   = 14 pt}}
	\tikzset{Action/.style = {draw, 
                					shape          = circle,
	                                text           = black,
	                                fill           = black,
	                                inner sep      = 2pt,
	                                outer sep      = 0pt}}

	\node[VertexStyle](s0) at (0,0) {$ s_{0} $};
	\node[VertexStyle](s1) at (2.5,0.5){$s_1$};
	\node[VertexStyle](sm1) at (3,-2){$s_{-1}$};
	\node[VertexStyle](s2) at (4,2){$s_{2}$};
	\node[VertexStyleDashed](s3) at (2.5,4){$\ldots$};
	\node[VertexStyle](s4) at (-0.25,3.75){$s_{\scalebox{1}{$\scriptscriptstyle S-3$}}$};
	\node[VertexStyle](s5) at (-1.5,1.5){$s_{\scalebox{1}{$\scriptscriptstyle S-2$}}$};
	\node[VertexStyle](g) at (6,-2){$g$};
	\node(s0s1) at (1,-1){};
	\node(sm1g) at (3.5,-2){};

    \draw[->,ForestGreen, >=latex](s0) to [out=-10, in=190, looseness=0.5] node[midway, xshift=0.7em,yshift=-0.8em]{{\color{blue}{$1-p_{\min}$}}} (s1);   
    \draw[->,ForestGreen, >=latex](s0) to [out=320, in=190, looseness=0.5] node[midway, yshift=-0.5em]{{\color{blue}{$p_{\min}$}}} (sm1);   
    \draw[->,red, >=latex](sm1) to [out=-10, in=190, looseness=0.5] node[midway, yshift=0.5em]{} (g);   
    \draw[->,ForestGreen, >=latex](s1) to [out=-10, in=190, looseness=0.5] node[midway, yshift=0.5em]{} (s2);   
    \draw[->,ForestGreen,dashed, >=latex](s2) to [out=-10, in=190, looseness=0.5] node[midway, yshift=0.5em]{} (s3);   
    \draw[->,ForestGreen,dashed, >=latex](s3) to [out=-10, in=190, looseness=0.5] node[midway, yshift=0.5em]{} (s4);   
    \draw[->,ForestGreen, >=latex](s4) to [out=-10, in=190, looseness=0.5] node[midway, yshift=0.5em]{} (s5);   
    \draw[->,ForestGreen, >=latex](s5) to [out=-10, in=190, looseness=0.5] node[midway, yshift=0.5em]{} (s0);   
    \node [above, xshift=0.75em,yshift=-0.6em] at (s0s1) {\color{ForestGreen}{\tiny $c=0$}};
    \node [above, xshift=0.3em,yshift=1.9em] at (s0s1) {\color{ForestGreen}{\tiny $c=0$}};
    \node [above, xshift=1.8em,yshift=-0.3em] at (sm1g) {\color{red}{\tiny $c=1$}};
    \node [above, xshift=3.2em,yshift=3.5em] at (s0s1) {\color{ForestGreen}{\tiny $c=0$}};
    \node [above, xshift=3.1em,yshift=6.2em] at (s0s1) {\color{ForestGreen}{\tiny $c=0$}};
    \node [above, xshift=0.4em,yshift=8.5em] at (s0s1) {\color{ForestGreen}{\tiny $c=0$}};
    \node [above, xshift=-4.2em,yshift=2em] at (s0s1) {\color{ForestGreen}{\tiny $c=0$}};
    \node [above, xshift=-4.2em,yshift=6.2em] at (s0s1) {\color{ForestGreen}{\tiny $c=0$}};
    \end{scope}
    
\end{tikzpicture}

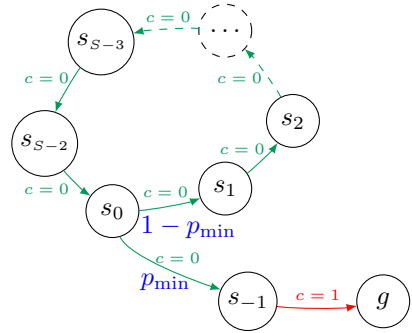
\captionof{figure}{\small Markov chain of the optimal policy of an SSP instance with $S$ states. Transitions in green incur a cost of $0$, while the transition in red leading to the goal state $g$ incurs a cost of $1$. All transitions are deterministic, apart from the one starting from $s_0$, which reaches state $s_{-1}$ with probability $p_{\min}$ and state $s_1$ with probability $1-p_{\min}$, where $p_{\min} > 0$.}
\label{toy_figure}
\end{minipage}

\vspace{0.2in}


\section{An Alternative Assumption on the SSP Problem: No Almost-Sure Zero-Cost Cycles} \label{app_alt_asm_no_cycle_zero_cost}

Here we complement Sect.\,\ref{subsect_bounding_T} by introducing an alternative assumption on the SSP problem (which is weaker than Asm.\,\ref{asm1_cminpositive}) and we analyze the regret bound achieved by \ALGO (under the set-up of Sect.\,\ref{subsect_bounding_T}). We draw inspiration from the common assumption in the deterministic shortest path setting that the transition graph does not possess any cycle of zero costs \citep{bertsekas1991linear}. In the following we introduce a ``stochastic'' counterpart of this assumption.

\begin{assumption}\label{asm_cycle_zerocost_stochastic}
    There exist unknown constants $c^{\dagger} > 0$ and $q^{\dagger} > 0$ such that:
    \begin{align*}
        \mathbb{P}\Bigg( \bigcap_{s' \in \cS} \bigcap_{\omega \in \Omega_{s'}} \Big\{ \sum_{i=1}^{\vert \omega \vert} c_i  \geq c^{\dagger} \Big\} \Bigg) \geq q^{\dagger},
    \end{align*}
where for every state $s' \in \cS$ we denote by $\Omega_{s'}$ the set of all possible trajectories in the SSP-MDP that start from state $s'$ and end in state $s'$, and we denote by $c_1, \ldots, c_{\vert \omega \vert}$ the sequence of costs incurred during a trajectory $\omega$.
\end{assumption}%
Asm.\,\ref{asm_cycle_zerocost_stochastic} is strictly weaker than the assumption of positive costs (Asm.\,\ref{asm1_cminpositive}) and it guarantees that the conditions of Lem.\,\ref{lemma_wellposedproblem} hold. Intuitively, it implies that the agent has a non-zero probability of gradually accumulating some positive cost as its trajectory length increases. In particular, under Asm.\,\ref{asm_cycle_zerocost_stochastic}, any trajectory of length~$S+1$ that does not reach the goal must accumulate costs of at least $c^{\dagger}$ with probability at least $q^{\dagger}$.

When $z \geq \ln (T/\delta) / q^{\dagger} \ge \frac{\ln (T / \delta)}{- \ln (1 - q^{\dagger})}$, it is guaranteed that $(1 - q^{\dagger})^z \le \delta / T$. Repeatedly applying this argument means that with probability at least $1-\delta/T$, for $z \geq \ln (T/\delta) / q^{\dagger}$ it holds that either $\sum_{i=1}^{z (S+1)} c_i \geq c^{\dagger}$, or the agent has reached the goal in the trajectory indexed by the time steps $[1, z (S+1)]$. Denote $z_0 := \lceil \ln (T/\delta) / q^{\dagger} \rceil$. For each episode, divide time steps in it into chunks with length $z_0 (S + 1)$, with the exception that the last chunk in it may have length less than or equal to $z_0 (S + 1)$ (just like taking modulo). So in each episode, the agent accumulates cost of at least $c^{\dagger}$ in each chunk except for the last one, and in the last chunk the agent reaches $g$. If we define $Z$ as the total number of chunks with cost at least $c^{\dagger}$ in all episodes, then $Z \ge \frac{T - K z_0 (S + 1)}{z_0 (S + 1)}$. Thus from $C \ge Z c^{\dagger}$ we have $T \le O \left( \frac{S \log(T/\delta)}{q^{\dagger}} \left( \frac{C}{c^{\dagger}} + K \right) \right) \le O ( S (T / \delta)^{1/4} C K / ( q^{\dagger} c^{\dagger} ) )$, with $C$ the cumulative cost. Using the loose bound $C \le O (B_{\star} S^2 A K \cdot \sqrt{B_{\star} T S A / \delta})$ and isolating~$T$ (with the same reasoning as in the case of positive costs in Sect.\,\ref{subsect_bounding_T}) gives that~$T \le O ( B_{\star}^6 S^{14} A^6 K^8 / ( (q^{\dagger} c^{\dagger})^4 \delta^3 ) )$ and thus that $\log T = O(\log(K B_{\star} S A / (c^{\dagger} q^{\dagger} \delta)))$. Plugging this in Thm.\,\ref{key_lemma_Tdependent} yields the following.

\begin{corollary}\label{thm_regret_bound_nocycle_zerocost}
Under Asm.\,\ref{asm_cycle_zerocost_stochastic}, running \textup{\ALGO} (Alg.\,\ref{algo}) with $B = B_{\star} \geq 1$ and $\eta = 0$ gives the following regret bound with probability at least $1-\delta$
\begin{align*}
     R_K = O\left( B_{\star} \sqrt{ S A K } \log\left( \frac{K B_{\star} S A }{c^{\dagger} q^{\dagger} \delta} \right) + B_{\star} S^2 A \log^2\left( \frac{K B_{\star} S A }{c^{\dagger} q^{\dagger} \delta} \right) \right).
\end{align*}
\end{corollary}

The regret bound of Cor.\,\ref{thm_regret_bound_nocycle_zerocost} is (nearly) \textbf{minimax} and \textbf{horizon-free} (and it can be made \textbf{parameter-free} by executing Alg.\,\ref{algobstar} instead of Alg.\,\ref{algo}). The bound depends logarithmically on the inverse of the constants~$c^{\dagger}$,~$q^{\dagger}$. We observe that i)~it no longer becomes relevant if one constant is exponentially small, ii)~spelling out $c^{\dagger}$, $q^{\dagger}$ satisfying Asm.\,\ref{asm_cycle_zerocost_stochastic} is challenging as they subtly depend on both the cost function and the transition dynamics, although iii)~the agent does not need to know nor estimate $c^{\dagger}$ and $q^{\dagger}$ to achieve the regret bound of Cor.\,\ref{thm_regret_bound_nocycle_zerocost}.

\vspace{0.2in}

\section{Full Statement of Corollary \ref{thm_regret_bound_cmin_zero_priorknowledge_Tstar}}\label{app_cor_explicit_dep_constants_estimate_Tstar}

Here we make explicit the \textit{constant} terms $\upsilon, \lambda, \zeta$ in the regret bound of Cor.\,\ref{thm_regret_bound_cmin_zero_priorknowledge_Tstar}. 

Recall that Asm.\,\ref{asm_priorknowledge_Tstar} considers that the agent has prior knowledge of a quantity $\ov T_{\star}$ that verifies $T_{\star}/ \upsilon \leq \ov T_{\star} \leq \lambda T_{\star}^{\zeta}$ for some unknown constants $\upsilon, \lambda, \zeta \geq 1$ (note that $\upsilon= \lambda= \zeta = 1$ when $T_{\star}$ is known). Under Asm.\,\ref{asm_priorknowledge_Tstar}, running \textup{\ALGO} (Alg.\,\ref{algo}) with $B = B_{\star}$ and $\eta = (\ov T_{\star} K)^{-1}$ gives the following regret bound with probability at least $1-\delta$
\begin{align*}
   R_K = O\left( \left(B_{\star} + \frac{\nu}{K}\right) \sqrt{ S A K } \zeta \log\left( \frac{\lambda  K T_{\star} S A }{ \delta} \right) + \left(B_{\star} + \frac{\nu}{K}\right) S^2 A \zeta^2 \log^2\left( \frac{\lambda K T_{\star} S A }{ \delta} \right) + \nu \right).
\end{align*}

\newpage

\section{Proof of Theorem \ref{key_lemma_Tdependent}} \label{sect_main_proof}

In this section, we present the proof of Thm.\,\ref{key_lemma_Tdependent} (the missing proofs of the intermediate results within the section are deferred to App.\,\ref{app_missing_proofs}). We recall that throughout App.\,\ref{sect_main_proof} we analyze Alg.\,\ref{algo} without cost perturbation (i.e., $\eta = 0$) and we assume that \textbf{1)} the estimate verifies $B \geq \max\{B_{\star}, 1\}$ and \textbf{2)} the conditions of Lem.\,\ref{lemma_wellposedproblem} hold. 

\subsection{High-Probability Event}

\begin{definition}[High-probability event]\label{def_hp_event}
We define the event $\mathcal{E} := \mathcal{E}_1 \cap \mathcal{E}_2 \cap \mathcal{E}_3$, where
\makeatletter
\newcommand*\mysize{%
   \@setfontsize\mysize{9.2}{9.2}%
}
\makeatother
\begin{mysize}
\begin{align}
    \calE_1 &:= \left\{ \forall (s,a) \in \SA, \forall n(s,a) \geq 1 \, : ~  \vert (\wh P_{s,a} - P_{s,a}) V^{\star} \vert \leq 2 \sqrt{\frac{ \mathbb{V}(\wh P_{s,a}, V^{\star}) \iota_{s,a}}{ n(s,a) }} + \frac{14 B_{\star} \iota_{s,a}}{ 3 n(s,a)}   \right\}, \label{hp_event_1} \\
    \calE_2 &:= \left\{ \forall (s,a) \in \SA, \forall n(s,a) \geq 1 \,: ~ \vert \wh{c}(s,a) - c(s,a) \vert \leq 2 \sqrt{\frac{ 2 \wh{c}(s,a) \iota_{s,a}}{ n(s,a) }} + \frac{28 \iota_{s,a}}{ 3 n(s,a)} \right\}, \label{hp_event_2} \\
    \calE_3 &:= \left\{ \forall (s,a,s') \in \SA \times \cS', ~ \forall n(s,a) \geq 1 \,: ~  |P_{s,a,s'}-\wh P_{s,a,s'}| \leq \sqrt{\frac{2P_{s,a,s'} \iota_{s,a}}{n(s,a)}}+\frac{\iota_{s,a}}{n(s,a)} \right\}, \label{hp_event_3}
\end{align}
\end{mysize}
where $\iota_{s,a} := \ln \left( \frac{12 S A S' [n^+(s,a)]^2}{\delta} \right)$. 
\end{definition}

\begin{lemma}
    It holds that $\mathbb{P}(\mathcal{E}) \geq 1 - \delta$.
\end{lemma}
\begin{proof}
The events $\calE_1$ and $\calE_2$ hold with probability at least $1-2\delta/3$ by the concentration inequality of Lem.\,\ref{concen_bound_empirical} and by union bound over all $(s,a) \in \SA$. The event $\calE_3$ holds with probability at least $1-\delta/3$ by Bennett's inequality (Lem.\,\ref{bennett}, anytime version), by Lem.\,\ref{var_le_exp} and by union bound over all $(s,a,s') \in \SA \times \cS'$.
\end{proof}

\subsection{Analysis of a \VISGO Procedure} 

\makeatletter
\newcommand*\mysize{%
   \@setfontsize\mysize{8.5}{9.0}%
}
\makeatother

A \VISGO procedure in Alg.\,\ref{algo} computes iterates of the form $V^{(i+1)} = \wtcalL V^{(i)}$, where $\wtcalL$ is an operator that we define as follows. For any $U \in \mathbb{R}^{S'}$ such that $U(g) = 0$, we set $\wtcalL U(g) :=0$ and for $s \in \cS$ we set $\wtcalL U(s) :=  \min_{a \in \cA} \wtcalL U(s,a)$, where 
\begin{align} 
    \wtcalL U(s,a) := \max &\Bigg\{ \wh{c}(s,a) + \wt P_{s,a} U - \max \Big\{ c_1 \sqrt{ \frac{\mathbb{V}(\wt P_{s,a}, U) \iota_{s,a} }{n^+(s,a)}} , \ c_2 \frac{B \iota_{s,a}}{n^+(s,a)}  \Big\} \nonumber \\ 
    & \quad - c_3 \sqrt{\frac{\wh{c}(s,a) \iota_{s,a}}{n^+(s,a)}} - c_4 \frac{B \sqrt{S' \iota_{s,a}}}{n^+(s,a)}, \ 0 \Bigg\}. \label{eq_wtcalL}
\end{align}%
Starting from an optimistic initialization $V^{(0)} = 0$ at each state, we show the following two properties:
\begin{itemize}[leftmargin=*]
    \item \textit{Optimism:} with high probability, $Q^{(i)}(s,a) \leq Q^{\star}(s,a), \forall i \geq 0$;
    \item \textit{Finite-time near-convergence:} Given any error $\epsVI > 0$, the procedure stops at a \textit{finite} iteration $j$ such that $\norm{ V^{(j)} - V^{(j-1)}}_{\infty} \leq \epsVI$, which implies that the vector $V^{(j)}$ verifies some fixed point equation for~$\wtcalL$ up to an error scaling with $\epsVI$. 
\end{itemize}

\subsubsection{Properties of the slightly skewed transitions $\wt P$}

Lem.\,\ref{lemma_properties_wtP} shows that the bias introduced by replacing $\wh P_{s,a}$ with $\wt P_{s,a}$ decays inversely with $n(s,a)$, the number of visits to state-action pair $(s,a)$.

\newcommand{\lemmapropertieswtP}{
    For any non-negative vector $U \in \mathbb{R}^{S'}$ such that $U(g) = 0$, for any $(s,a) \in \cS \times \cA$, it holds that
    \begin{align*}
        \wt P_{s,a} U \leq \wh P_{s,a} U \leq \wt P_{s,a} U + \frac{\norm{U}_{\infty}}{n(s,a) + 1}, \quad \quad \quad \quad \big\vert \mathbb{V}(\wt P_{s,a}, U) -  \mathbb{V}(\wh P_{s,a}, U) \big\vert \leq \frac{2 \norm{U}_{\infty}^2 S'}{n(s,a) + 1}.
    \end{align*}
}
\begin{lemma}
  \label{lemma_properties_wtP}
  \lemmapropertieswtP
\end{lemma}

\noindent Denote by $\nu$ the probability of reaching the goal from any state-action pair in $\wt P$, i.e.,
\begin{align}\label{eq_def_nu}
    \nu_{s,a} := \wt{P}_{s,a,g}, \quad \quad \nu := \min_{s,a} \nu_{s,a}.
\end{align}
By construction of $\wt P$, the quantity $\nu$ is strictly positive. This immediately implies the following result. 
\begin{lemma}\label{lemma_all_policies_proper}
    In the SSP-MDP associated to $\wt P$ with any bounded cost function, \textit{all} policies are proper.
\end{lemma}

\begin{remark}[Mapping to a discounted problem] \label{remark_discounted}
    In an SSP problem with only proper policies, the (optimal) Bellman operator is usually contractive only w.r.t.\,a weighted-sup norm \citep{bertsekas1995dynamic}. Here, the construction of $\wt P$ entails that any SSP defined on it with fixed bounded costs has a (optimal) Bellman operator that is a sup-norm contraction. In fact, the SSP problem on $\wt P$ can be cast as a discounted problem with a (state-action dependent) discount factor $\gamma_{s,a} := 1 - \nu_{s,a} < 1$ (we recall that discounted MDPs are a subclass of SSP-MDPs). Intuitively, at insufficiently visited state-action pairs, the agent behaves optimistically which increases the chance of reaching the goal and terminating the trajectory. Equivalently, we can interpret the agent as being uncertain about its future predictions and it is thus encouraged to act more myopically, which is connected to lowering the discount factor in the discounted RL setting.
\end{remark}


\subsubsection{Important auxiliary function $f$ and its properties}

Lem.\,\ref{lemma_properties_f} examines an auxiliary function $f$ that plays a key role in the analysis. Indeed, we see that an instantiation of $f$ surfaces in the definition of the operator $\wtcalL$ (Eq.\,\ref{eq_wtcalL}). While the first property (monotonicity) is similar to the one required in \citet{zhang2020reinforcement}, the third property (contraction) is SSP-specific and is crucial to guarantee the (finite-time) near-convergence of a \VISGO procedure. 

\newcommand{\lemmapropertiesf}{
    Let $\Upsilon := \{ v \in \mathbb{R}^{S'}: v \geq 0, ~v(g) = 0, ~ \norm{v}_{\infty} \leq B \}$. Let $f: \Delta^{S'} \times \Upsilon \times \mathbb{R} \times \mathbb{R} \times \mathbb{R} \rightarrow \mathbb{R}$ with $f(p,v,n,B, \iota) := p v - \max\Big\{ c_1 \sqrt{\frac{\mathbb{V}(p,v) \iota}{n}}, \, c_2 \frac{B \iota }{n} \Big\}$, with $c_1 = 6$ and $c_2 = 36$ (here taking any pair of constants such that $c_1^2 \le c_2$ works). Then $f$ satisfies, for all $p \in \Delta^{S'}$, $v \in \Upsilon$ and $n, \iota > 0$,
\begin{enumerate}[leftmargin=*]
    \item $f(p,v,n,B, \iota)$ is non-decreasing in $v(s)$, i.e.,
    \begin{align*}
        \forall (v, v') \in \Upsilon^2, ~ v \leq v' ~ \implies ~ f(p,v,n,B, \iota) \leq f(p,v',n,B, \iota);
    \end{align*}
    \item $f(p,v,n,B, \iota) \leq p v - \frac{c_1}{2} \sqrt{\frac{\mathbb{V}(p,v) \iota}{n}} - \frac{c_2}{2} \frac{ B \iota}{n} \leq p v - 2 \sqrt{\frac{\mathbb{V}(p,v) \iota}{n}} - 14 \frac{ B \iota}{n}$; 
    \item 
    If $p(g) > 0$,
    then $f(p,v,n,B, \iota)$ is $\rho_p$-contractive in $v(s)$, with
    $\rho_p := 1 - p(g) < 1$
    ,~i.e.,
    \begin{align*}
        \forall (v, v') \in \Upsilon^2, ~ \abs{ f(p,v,n,B, \iota) - f(p,v',n,B, \iota) } \leq \rho_p \norm{ v - v'}_{\infty}.
    \end{align*}
\end{enumerate}
}

\begin{lemma}
  \label{lemma_properties_f}
  \lemmapropertiesf
\end{lemma}


\subsubsection{Optimism of \VISGO}\label{subsection_optimism}

We now show that with the bonus defined in Eq.\,\ref{update_bonus}, the $Q$-function is always optimistic with high probability.

\newcommand{\lemmaoptimism}{
Conditioned on the event $\mathcal{E}$, for any output $Q$ of the \VISGO procedure (line \ref{output_trigger} of Alg.\,\ref{algo}) and for any state-action pair $(s,a) \in \SA$, it holds that
\begin{align*}
    Q(s,a) \leq Q^{\star}(s,a).
\end{align*}
}

\begin{lemma}
  \label{lemma_optimism}
  \lemmaoptimism
\end{lemma}

\begin{proofidea}
We prove the result by induction on the inner iterations $i$ of \VISGO, i.e., $Q^{(i)}(s,a) \leq Q^{\star}(s,a)$. We use the update of the $Q$-value (line \ref{update_Q}), Lem.\,\ref{lemma_properties_wtP}, the definition of event $\calE$ combined with the fact that $B \geq B_{\star}$, as well as the first two properties of Lem.\,\ref{lemma_properties_f} applied to $f(\wt P_{s,a}, V^{(i)}, n^+(s,a), B, \iota_{s,a})$.
\end{proofidea}

\subsubsection{Finite-time near-convergence of \VISGO}

\noindent \textbf{\emph{Warm-up: convergence with no bonuses.}} For the sake of discussion, let us first examine an idealized case where $n(s,a) \rightarrow + \infty$ for all $(s,a)$, which means $b(s,a) = 0$ for all $(s,a)$. In that case, the iterates verify $V^{(i+1)} = \wtcalL^{\star} V^{(i)}$, where $\wtcalL^{\star} U(s) := \min_a \big\{ c(s,a) + \wt P_{s,a} U \big\}$, $\forall U \in \mathbb{R}^S, s \in \cS$. Thus $\wtcalL^{\star}$ is the optimal Bellman operator of the SSP instance $\wt M$ with transitions $\wt P$ and cost function $c$. From Lem.\,\ref{lemma_all_policies_proper}, all policies are proper in $\wt M$. As a result, the operator $\wtcalL^{\star}$ is contractive (cf.\,Remark~\ref{remark_discounted}) and convergent \citep{bertsekas1995dynamic}. 

\vspace{0.1in}

\noindent \textbf{\emph{Convergence with bonuses.}} In \VISGO, however, we must account for the bonuses $b(s,a)$. Setting aside the truncation of each iterate $V^{(i)}$ (i.e., the lower bounding by $0$), we notice that a update for $V^{(i+1)}$ can be interpreted as the (truncated) Bellman operator of an SSP problem with cost function $c(s,a) - b^{(i+1)}(s,a)$. However, $b^{(i+1)}(s,a)$ depends on $V^{(i)}$, the previous iterate. This dependence means that the cost function is no longer fixed and the reasoning from the previous paragraph no longer holds. As a result, we directly analyze the properties of the operator $\wtcalL$ that defines the sequence of iterates $V^{(i+1)} = \wtcalL V^{(i)}$ in \VISGO (Eq.\,\ref{eq_wtcalL}).

\newcommand{\lemmawtcalLconvergent}{
The sequence $(V^{(i)})_{i \geq 0}$ is non-decreasing. Combining this with the fact that it is upper bounded by $V^{\star}$ from Lem.\,\ref{lemma_optimism}, the sequence must converge.
}

\begin{lemma}
  \label{lemma_wtcalL_convergent}
  \lemmawtcalLconvergent
\end{lemma}

While Lem.\,\ref{lemma_wtcalL_convergent} states that $\wtcalL$ ultimately converges starting from a vector of zeros, the following result guarantees that it can approximate in finite time its fixed point within any (arbitrarily small) positive component-wise accuracy.

\newcommand{\lemmawtcalLcontraction}{
Denote by $\nu > 0$ the probability of reaching the goal from any state-action pair in~$\wt P$, i.e., $\nu := \min_{s,a} \wt{P}_{s,a,g}$. Then $\wtcalL$ is a $\rho$-contractive operator with modulus 
$\rho := 1 - \nu < 1$.
}

\begin{lemma}
  \label{lemma_wtcalL_contraction}
  \lemmawtcalLcontraction
\end{lemma}

\begin{proofidea}
We can apply the third property (contraction) of Lem.\,\ref{lemma_properties_f} to $f(\wt P_{s,a}, V^{(i)}, n^+(s,a), B, \iota_{s,a})$, for any state-action pair $(s, a)$. Taking the maximum over $(s,a)$ pairs yields the contraction property of $\wtcalL$.
\end{proofidea}

\begin{remark}\label{rmk_comput_complexity_contraction} 
Lem.\,\ref{lemma_wtcalL_contraction} guarantees that $\norm{ V^{(i+1)} - V^{(i)}}_{\infty} \leq \epsVI$ for $i \geq \frac{\log(\max\{B_{\star}, 1\} / \epsVI)}{1 - \rho}$, which yields the desired property of finite-time near-convergence of \VISGO (i.e., it always stops at a finite iteration~$i$). Moreover, by definition of $\epsVI$ we have $\log(1/\epsVI) = O(SA \log(T))$, the (possibly loose) lower bound $1 - \rho = \nu \ge \frac{1}{T + 1}$, 
and there are at most $O(S A \log T)$ \VISGO procedures in total, thus we see that \ALGO has a polynomially bounded computational complexity. 
\end{remark}

\subsection{Interval Decomposition and Notation} \label{subsect_intervaldecomp_notation}

\paragraph{Interval decomposition.} In the analysis we split the time steps into \textit{intervals}. The first interval begins at the first time step, and an interval ends once either (1) the goal state $g$ is reached; (2) or the trigger condition holds (i.e., the visit to a state-action pair is doubled). We see that an update is triggered (line \ref{line_update} of Alg.\,\ref{algo}) whenever condition (2) is met.

\paragraph{Notation.} We index intervals by $m = 1, 2, \ldots$ and the length of interval $m$ is denoted by~$H^m$ (it is bounded almost surely). The trajectory visited in interval $m$ is denoted by $U^m = (s_1^m, a_1^m, \ldots, s_{H^m}^m, a_{H^m}^m, s^m_{H^m+1})$, where $a_h^m$ is the action taken in state $s_h^m$. The concatenation of the trajectories of the intervals up to and including interval $m$ is denoted by $\ov{U}^m$, i.e., $\ov{U}^m = \bigcup_{m'=1}^m U^{m'}$. Moreover, $c_h^m$ denotes the cost in the $h$-th step of interval $m$. We use the notation $Q^m(s,a)$, $V^m(s)$, $\wh P_{s,a}^m$, $\wt P_{s,a}^m$ and $\epsVI^m$ to denote the values (computed in lines \ref{begin_trigger}-\ref{output_trigger}) of $Q(s,a)$, $V(s)$, $\wh P_{s,a}$, $\wt P_{s,a}$ and $\epsVI$ in the beginning of interval~$m$. Let $n^m(s,a)$ and $\wh c^m(s,a)$ denote the values of $\max\{n(s,a), 1\}$ and $\wh c(s,a)$ used for computing $Q^m(s,a)$. Finally, we set 
\begin{align*}
    b^m(s,a) := \max \left\{ c_1 \sqrt{ \frac{\mathbb{V}(\wt P_{s,a}, V^{m}) \iota_{s,a} }{n^m(s,a)}} ,\ c_2 \frac{B \iota_{s,a}}{n^m(s,a)}  \right\} + c_3 \sqrt{\frac{\wh{c}^m(s,a) \iota_{s,a}}{n^m(s,a)}} + c_4 \frac{B \sqrt{S' \iota_{s,a}}}{n^m(s,a)}.
\end{align*}

\subsection{Bounding the Bellman Error}\label{subsect_bounding_bellman_error}

\newcommand{\lemmaboundingbellmanerror}{
Conditioned on the event $\mathcal{E}$, for any interval $m$ and state-action pair $(s,a) \in \SA$,
\begin{align*}
    \abs{c(s,a) + P_{s,a} V^m - Q^{m}(s,a)} \leq  \min\big\{ \beta^m(s,a), B_{\star} + 1 \big\},
\end{align*}
where we define
\begin{align*}
    \beta^m(s,a) &:= 4 b^m(s,a) + \sqrt{ \frac{2\mathbb{V}(P_{s,a}, V^{\star}) \iota_{s,a}}{n^m(s,a)}} + \sqrt{\frac{2 S' \mathbb{V}(P_{s,a}, V^{\star} - V^m) \iota_{s,a}  }{n^m(s,a)}} \\
    &~~~ +  \frac{3 B_{\star} S' \iota_{s,a}}{ n^m(s,a)} + \left(1 + c_1 \sqrt{\iota_{s,a} / 2}\right) \epsVI^m.
\end{align*}
}

\begin{lemma}
  \label{lemma_bounding_bellman_error}
  \lemmaboundingbellmanerror
\end{lemma}

\begin{proofidea}
We use that $V^m$ approximates the fixed point of $\wtcalL$ up to an error scaling with $\epsVI$. We end up decomposing and bounding the difference $P_{s,a} V^m - \wt{P}_{s,a} V^m \leq (\wh P_{s,a} - \wt P_{s,a}) V^m + (P_{s,a} - \wh{P}_{s,a}) V^{\star}  + (P_{s,a} - \wh{P}_{s,a})(V^m - V^{\star})$, where the first term is bounded by Lem.\,\ref{lemma_properties_wtP} and \ref{lemma_optimism}, while the second and third terms are bounded using the definition of the event $\calE$.
\end{proofidea}

\subsection{Regret Decomposition} \label{subsect_regret_decomposition}

\newcommand{\Mzero}{\mathcal{M}_{0}\xspace}
\newcommand{\Izero}{\mathds{I}[m \in \Mzero]\xspace}
\newcommand{\Izeroplusone}{\mathds{I}[m+1 \in \Mzero]\xspace}

We assume that the event $\mathcal{E}$ defined in Def.\,\ref{def_hp_event} holds. In particular it guarantees that Lem.\,\ref{lemma_optimism} and Lem.\,\ref{lemma_bounding_bellman_error} hold for all intervals $m$ simultaneously.

We denote by $M$ the total number of intervals in which the first $K$ episodes elapse. For any $M' \le M$, we denote by $\Mzero(M')$ the set of intervals which are among the first $M'$ intervals, and constitute the first intervals in each episode (i.e., either it is the first interval or its previous interval ended in the goal state). We also denote by $K_{M'} := |\Mzero(M')|$, $T_{M'} := \sum_{m=1}^{M'} H^m$ and $C_{M'} := \sum_{m=1}^{M'} \sum_{h=1}^{H^m} c_h^m$. Note that $K$ and $T$ are equivalent to $K_M$ and $T_M$, respectively, and $C_{M'}$ is the cumulative cost in the first $M'$ intervals.

Instead of bounding the regret $R_K$ from Eq.\,\ref{eq_ssp_regret_def}, we bound $\wt R_{M'} := C_{M'} - K_{M'} V^{\star}(s_0)$ for any fixed choice of $M' \le M$, as done in \citet{rosenberg2020near}. We see that $\wt R_M = R_K$, the true regret within $K$ episodes. To derive Thm.\,\ref{key_lemma_Tdependent}, we will show that $M$ is finite and instantiate $M'=M$. In the following we do the analysis for arbitrary $M' \leq M$ as it will be useful for the parameter-free case studied in App.\,\ref{sect_unknown_Bstar} (i.e., when no estimate $B \geq B_{\star}$ is available).

We decompose $\wt R_{M'}$ as follows
\begin{align*}
    \wt R_{M'} &\myineqi \sum_{m=1}^{M'} \sum_{h=1}^{H^m} c_h^m - \sum_{m \in \Mzero(M')} V^{m}(s_0), \\
    &\myineqii  \sum_{m=1}^{M'} \sum_{h=1}^{H^m} c_h^m +  \sum_{m=1}^{M'} \left( \sum_{h=1}^{H^m} V^m(s_{h+1}^m) - V^m(s_{h}^m) \right) + 2 S A \log_2(T_{M'}) \max_{1 \leq m \leq M'} \norm{V^m}_{\infty} \\
    &\myineqiii \sum_{m=1}^{M'} \sum_{h=1}^{H^m} \left[ c_h^m + P_{\samh} V^m - V^m(\smh) \right] + \sum_{m=1}^{M'} \sum_{h=1}^{H^m} \left[ V^m(\smhp) - P_{\samh} V^m \right] \\ 
    &\quad + 2 B_{\star} S A \log_2(T_{M'}) \\
    &\myineqiv \underbrace{ \sum_{m=1}^{M'} \sum_{h=1}^{H^m} \left[ V^m(\smhp) - P_{\samh} V^m \right]}_{:= X_1(M')} + \underbrace{ \sum_{m=1}^{M'} \sum_{h=1}^{H^m} \beta^m(\samh)  }_{:= X_2(M')} + \underbrace{ \sum_{m=1}^{M'} \sum_{h=1}^{H^m} c_h^m - \cmh}_{:= X_3(M')} \\
    &\quad + 2 B_{\star} S A \log_2(T_{M'}),
\end{align*}
where (i) uses the optimism property of Lem.\,\ref{lemma_optimism}, (ii) stems from the construction of intervals (Lem.\,\ref{lemma_bound_sum_diff_Vm}), (iii) uses that $\max_{1 \leq m \leq M'} \norm{V^m}_{\infty} \leq B_{\star}$ (from Lem.\,\ref{lemma_optimism}), and (iv) comes from Lem.\,\ref{lemma_bounding_bellman_error}. We now focus on bounding the terms $X_1(M')$, $X_2(M')$ and $X_3(M')$. To this end, we introduce the following useful quantities
\begin{align*}
    X_4(M') := \sum_{m=1}^{M'} \sum_{h=1}^{H^m} \mathbb{V}(P_{\samh}, V^m), \quad \quad \quad \quad 
    X_5(M') := \sum_{m=1}^{M'} \sum_{h=1}^{H^m} \mathbb{V}(P_{\samh}, V^{\star} - V^m).
\end{align*}

\subsubsection{The $X_1(M')$ term}\label{subsection_X1_term}

$X_1(M')$ could be viewed as a martingale, so by taking $c=\max\{B_{\star},1\}$ in the technical Lem.\,\ref{martingale_bound}, we have with probability at least $1 - \delta$, 
\begin{align*}
    |X_1(M')| \leq ~ & 2\sqrt{ 2 X_4(M') (\log_2 ((\max\{B_{\star},1\})^2 T_{M'}) + \ln (2 / \delta) )} \\
    &+ 5 (\max\{B_{\star},1\}) (\log_2 ((\max\{B_{\star},1\})^2 T_{M'}) + \ln (2 / \delta)).
\end{align*}
To bound $X_1(M')$, we only need to bound $X_4(M')$.

\subsubsection{The $X_3(M')$ term}\label{subsection_X3_term}

Taking $c=1$ in the technical Lem.\,\ref{martingale_bound}, we have
\begin{align*}
    \mathbb{P}\left[|X_3(M')| \ge 2 \sqrt{2 \sum_{m=1}^{M'} \sum_{h=1}^{H^m} \mathrm{Var}(\samh) (\log_2(T_{M'}) + \ln (2 / \delta))} + 5 (\log_2(T_{M'}) + \ln (2 / \delta)) \right]\le \delta,
\end{align*}
where $\mathrm{Var}(s_t,a_t) := \mathbb{E}[(c_t - c(s_t,a_t))^2]$ ($c_t$ denotes the cost incurred at time step $t$). By Lem.\,\ref{var_le_exp},
\begin{align*}
    \sum_{m=1}^{M'} \sum_{h=1}^{H^m} \mathrm{Var}(\samh) &\le \sum_{m=1}^{M'} \sum_{h=1}^{H^m} c(\samh) \\
    &= \sum_{m=1}^{M'} \sum_{h=1}^{H^m} (c(\samh) - c_h^m) + C_{M'} \\
    &\le |X_3(M')| + C_{M'}.
\end{align*}
Therefore we have
\begin{align*}
    \mathbb{P}\Big[|X_3(M')| \ge & ~ 2 \sqrt{2 ( |X_3(M')| + C_{M'} ) (\log_2(T_{M'}) + \ln (2 / \delta))} + 5 (\log_2(T_{M'}) + \ln (2 / \delta)) \Big]\le \delta,
\end{align*}
which implies that $|X_3(M')|\le O\left(\log_2(T_{M'}) + \ln (2 / \delta) + \sqrt{C_{M'} (\log_2(T_{M'}) + \ln (2 / \delta))}\right)$ with probability at least $1 - \delta$.

\subsubsection{The $X_2(M')$ term}\label{subsection_X2_term}

The full proof of the bound on $X_2(M')$ is deferred to App.\,\ref{full_proof_bound_X2}. Here we provide a brief sketch. First, we bound~$\beta^m$ and apply a pigeonhole principle to obtain 
\begin{align*}
    X_2(M') \leq O \Bigg( & \sqrt{S A \log_2(T_{M'}) \iota_{M'} X_4(M')} + \sqrt{S^2 A \log_2(T_{M'}) \iota_{M'} X_5(M')} \\
    & + \sqrt{S A \log_2(T_{M'}) \iota_{M'} \sum_{m=1}^{M'} \sum_{h=1}^{H^m} \wh c^m (\samh)} \\
    & + B_{\star} S^2 A \log_2(T_{M'}) + B S^{3/2} A \log_2(T_{M'}) \iota_{M'} + \sum_{m=1}^{M'} \sum_{h=1}^{H^m} (1 + c_1 \sqrt{\iota_{M'}/2}) \epsVI^m \Bigg)
\end{align*}
with the logarithmic term $\iota_{M'} := \ln \left(\frac{ 12 S A S' T_{M'}^2 }{ \delta }\right)$ which is the upper-bound of $\iota_{s,a}$ when considering only time steps in the first $M'$ intervals.
The regret contributions of the estimated costs and the \VISGO precision errors are respectively
\begin{align*}
    &\sum_{m=1}^{M'} \sum_{h=1}^{H^m} \wh c^m (\samh) \le 2 S A (\log_2 (T_{M'}) + 1) + 2 C_{M'}, \\
    &\sum_{m=1}^{M'} \sum_{h=1}^{H^m} (1 + c_1 \sqrt{\iota_{M'}/2}) \epsVI^m  = O(S A \log_2(T_{M'}) \sqrt{\iota_{M'}}).
\end{align*}
To bound $X_4(M')$ and $X_5(M')$, we perform a recursion-based analysis on the value functions normalized by~$1/B_{\star}$. We split the analysis on the intervals, and not on the episodes as done in \citet{zhang2020reinforcement}. In Lem.\,\ref{lemma_bound_X4} and \ref{lemma_bound_X5} we establish that with overwhelming probability, 
\begin{align*}
    X_4(M') &\le O \left( B_{\star} ( C_{M'} + X_2(M') ) + (B_{\star}^2 S A + B_{\star}) (\log_2 (T_{M'}) + \ln(2 / \delta))  \right), \\
    X_5(M') &\le O \left( B_{\star}^2 S A (\log_2 (T_{M'}) + \ln(2 / \delta)) + B_{\star} X_2(M')\right).
\end{align*}
As a result, we obtain 
\begin{align*}
    X_2(M') &\le  O\Big( \sqrt{S A X_4(M')} \ov \iota_{M'} + \sqrt{S^2 A X_5(M') } \ov \iota_{M'} \\
    &\quad \quad ~ + SA \ov{\iota}_{M'}^{3/2} + \sqrt{S A C_{M'}} \ov \iota_{M'} + B_{\star} S^2 A  \ov{\iota}_{M'}^2 + B S^{3/2} A \ov{\iota}_{M'}^2 \Big), \\
    X_4(M') &\le O \left( B_{\star} ( C_{M'} + X_2(M') ) + (B_{\star}^2 S A + B_{\star}) \ov \iota_{M'} \right), \\
    X_5(M') &\le O \left( B_{\star}^2 S A \ov \iota_{M'} + B_{\star} X_2(M') \right).
\end{align*}
with the logarithmic term $\ov \iota_{M'} := \ln \left(\frac{ 12 S A S' T_{M'}^2 }{ \delta }\right) + \log_2 ((\max\{B_{\star},1\})^2 T_{M'}) + \ln \left(\frac{ 2 }{ \delta }\right)$. Isolating the $X_2(M')$ term finally yields
\begin{align*}
    X_2(M') &\le O( (\sqrt{B_{\star}} + 1 ) \sqrt{ S A C_{M'} } \ov \iota_{M'} + B S^2 A \ov{\iota}_{M'}^2 ).
\end{align*}

\subsubsection{Putting Everything Together} \label{app_putting_everything_together}

Ultimately, with probability at least $1 - 6 \delta$ we have
\begin{align*}
    \wt R_{M'} &\le X_1(M')+X_2(M')+X_3(M')+ 2 B_{\star} S A \log_2(T_{M'}) \\ 
    &\le O( (\sqrt{B_{\star}} + 1 ) \sqrt{ S A C_{M'} } \ov \iota_{M'} + B S^2 A \ov{\iota}_{M'}^2 ).
\end{align*}
Noting that $\wt R_{M'} = C_{M'} - K_{M'} V^{\star}(s_0)$, we have
\begin{align*}
    C_{M'} &\le K_{M'} V^{\star}(s_0) + O( (\sqrt{B_{\star}} + 1 ) \sqrt{ S A C_{M'} } \ov \iota_{M'} + B S^2 A \ov{\iota}_{M'}^2 ), \\
    C_{M'} &\myineqi \left(O \left( (\sqrt{B_{\star}} + 1 ) \sqrt{ S A } \ov \iota_{M'} \right) + \sqrt{K_{M'} V^{\star}(s_0) + O(B S^2 A \ov{\iota}_{M'}^2)}\right)^2 \\
    &\le K_{M'} V^{\star}(s_0) + O\left( (\sqrt{B_{\star}} + 1) \sqrt{ V^{\star}(s_0) S A K_{M'}} \ov{\iota}_{M'} + B S^2 A \ov{\iota}_{M'}^2 \right) \\
    &\le K_{M'} V^{\star}(s_0) + O\left( (B_{\star} + \sqrt{B_{\star}}) \sqrt{ S A K_{M'}} \ov{\iota}_{M'} + B S^2 A \ov{\iota}_{M'}^2 \right), 
\end{align*}
where (i) uses Lem.\,\ref{quadratic_ineq_bound}, $V^{\star}(s_0)\le B_{\star}$ and $\sqrt{B_{\star}} + 1 \le O( \sqrt{B_{\star} + 1} ) \le O(\sqrt{B})$. Hence
\begin{align*}
    \wt R_{M'} \le O\left(\sqrt{(B_{\star}^2 + B_{\star}) S A K_{M'}} \ov{\iota}_{M'} + B S^2 A \ov{\iota}_{M'}^2 \right).
\end{align*}
By scaling $\delta \gets \delta / 6$ we have the following important bound
\begin{align}
    \wt R_{M'} &\le O\Bigg( \sqrt{ (B_{\star}^2 + B_{\star}) S A K_{M'} } \log\left( \frac{\max\{B_{\star},1\} S A T_{M'} }{\delta} \right) \nonumber \\ &\quad \quad + B S^2 A \log^2\left( \frac{\max\{B_{\star},1\} S A T_{M'} }{\delta} \right) \Bigg). \label{regret_bound_interval}
\end{align}
The proof of Thm.\,\ref{key_lemma_Tdependent} is concluded by taking $M' = M$, where $M$ denotes the number of intervals in which the first~$K$ episodes elapse.

\section{Missing Proofs}\label{app_missing_proofs}

\subsection{Proofs of Lemmas \ref{lemma_properties_wtP}, \ref{lemma_properties_f}, \ref{lemma_optimism}, \ref{lemma_wtcalL_convergent}, \ref{lemma_wtcalL_contraction}, \ref{lemma_bounding_bellman_error}}

\newtheorem*{T1}{Restatement of Lemma \ref{lemma_properties_wtP}}
\begin{T1}
  \lemmapropertieswtP
\end{T1}

\begin{proof}
The proof uses the definition of $\wt P$ (Eq.\,\ref{eq_wtP}) and simple algebraic manipulation. For any $s'\neq g$, we have $\wt P_{s,a,s'} \leq \wh P_{s,a,s'}$ and $U(s') \geq 0$, as well as $U(g) = 0$, so $\wt P_{s,a} U \leq \wh P_{s,a} U$, and
\begin{align*}
     (\wh P_{s,a} - \wt P_{s,a}) U  = \Big( 1 - \frac{n(s,a)}{n(s,a)+1} \Big) \wh P_{s,a} U \leq \frac{\norm{U}_{\infty}}{n(s,a) + 1}.
\end{align*}
In addition, for any $s' \in \cS'$,
\begin{align*}
    \abs{\wt P_{s,a,s'} - \wh P_{s,a,s'} } \leq \Big\vert \frac{n(s,a)}{n(s,a)+1} - 1 \Big\vert \wh P_{s,a,s'} + \frac{\mathds{I}[s'=g]}{n(s,a)+1} \leq \frac{2}{n(s,a) + 1}.
\end{align*}
Therefore we have that
\begin{align*}
    \mathbb{V}(\wh P_{s,a}, U) &= \sum_{s' \in \cS'} \wh P_{s,a,s'} (U(s') - \wh P_{s,a} U)^2 \leq  \sum_{s' \in \cS'} \wh P_{s,a,s'} (U(s') - \wt P_{s,a} U)^2 \\
    &\leq  \sum_{s' \in \cS'} \left( \wt P_{s,a,s'} + \frac{2}{n(s,a) + 1} \right) ( U(s') - \wt P_{s,a} U) ^2
    \leq \mathbb{V}(\wt P_{s,a}, U) + \frac{2 \norm{U}_{\infty}^2 S'}{n(s,a) + 1},
\end{align*}
where the first inequality is by the fact that $z^{\star}=\sum_i p_i x_i$ minimizes the quantity $\sum_i p_i(x_i-z)^2$. Conversely, 
\begin{align*}
    \mathbb{V}(\wt P_{s,a}, U) &= \sum_{s' \in \cS'} \wt P_{s,a,s'} (U(s') - \wt P_{s,a} U)^2 
    \leq  \sum_{s' \in \cS'} \wt P_{s,a,s'} (U(s') - \wh P_{s,a} U)^2 \\
    &\leq  \sum_{s' \in \cS'} \left( \wh P_{s,a,s'} + \frac{2}{n(s,a) + 1} \right) ( U(s') - \wh P_{s,a} U) ^2 
    \leq \mathbb{V}(\wh P_{s,a}, U) + \frac{2 \norm{U}_{\infty}^2 S'}{n(s,a) + 1}.
\end{align*}
\end{proof}

\newtheorem*{T2}{Restatement of Lemma \ref{lemma_properties_f}}
\begin{T2}
  \lemmapropertiesf
\end{T2}    

\newcommand{\diffv}{{\scalebox{0.9}{$\diffp{f}{{v(s)}}$}}}

\begin{proof} The second claim holds by $\max \{x, y \} \geq (x+y)/2, \forall x, y$, by the choices of $c_1, c_2$ and because both $\sqrt{\frac{\mathbb{V}(p,v) \iota}{n}}$ and $\frac{B \iota}{n}$ are non-negative.
To verify the first and third claims, we fix all other variables but $v(s)$ and view $f$ as a function in $v(s)$. Because the derivative of $f$ in $v(s)$ does not exist only when $c_1 \sqrt{\frac{\mathbb{V}(p,v) \iota}{n}} = c_2 \frac{B \iota }{n}$, where the condition has at most two solutions, it suffices to prove that $\diffv \geq 0$ when $c_1 \sqrt{\frac{\mathbb{V}(p,v) \iota}{n}} \neq c_2 \frac{B \iota }{n}$. Direct computation gives
\begin{align*}
    \diffp{f}{{v(s)}} &= p(s) - c_1 \mathds{I}\left[  c_1 \sqrt{ \frac{\mathbb{V}(p,v) \iota}{n} } \geq c_2 \frac{B \iota}{n} \right] \frac{p(s)(v(s) - pv) \iota}{ \sqrt{n \mathbb{V}(p,v) \iota } } \\
    &\geq \min \big\{ p(s) , ~ p(s) - \frac{c_1^2}{c_2 B} p(s) \big( v(s) - p v \big) \big\} \\
    &\mygei  \min \big\{ p(s) , ~ p(s) - \frac{c_1^2}{c_2 } p(s) \big\} \\
    &\geq p(s) \Big( 1 - \frac{c_1^2}{c_2} \Big) = 0.
\end{align*}
Here (i) is by $v(s) - p v \le v(s) \le B$.
For the third claim, we perform a distinction of cases. If $c_1 \sqrt{\frac{\mathbb{V}(p,v) \iota}{n}} = c_2 \frac{B \iota }{n}$, where the condition has at most two solutions, then $f(v) = p v - c_2 \frac{B \iota }{n}$, which corresponds to a $\rho_p$-contraction since
\begin{align*}
    |f(v_1) - f(v_2)| = \left | \sum_{s \in \cS} p(s) (v_1(s) - v_2(s)) \right | \le \sum_{s \in \cS} p(s) \cdot \norm{v_1 - v_2}_{\infty} = (1 - p(g)) \norm{v_1 - v_2}_{\infty}.
\end{align*}
Otherwise $c_1 \sqrt{\frac{\mathbb{V}(p,v) \iota}{n}} \neq c_2 \frac{B \iota }{n}$, then the derivative of $f$ in $v(s)$ exists and it verifies 
\begin{align*}
    \left \| \diffp{f}{{v}} \right \|_1 &=  \sum_{s \in \cS} \left | \diffp{f}{{v(s)}} \right | = \sum_{s \in \cS} \diffp{f}{{v(s)}} \\
    &= \sum_{s \in \cS} \left[ p(s) - c_1 \mathds{I}\left[  c_1 \sqrt{ \frac{\mathbb{V}(p,v) \iota}{n} } \geq c_2 \frac{B \iota}{n} \right] \frac{p(s)(v(s) - pv) \iota}{ \sqrt{n \mathbb{V}(p,v) \iota } } \right] \\
    &= 1 - p(g) - c_1 \mathds{I}\left[  c_1 \sqrt{ \frac{\mathbb{V}(p,v) \iota}{n} } \geq c_2 \frac{B \iota}{n} \right] \sqrt{\frac{\iota}{n \mathbb{V}(p,v)}} [pv - (1 - p(g))\cdot  pv] \big\} \\
    & \le 1 - p(g).
\end{align*}
In this case, by the mean value theorem we obtain that $f$ is $\rho_p$-contractive.
\end{proof}

\newtheorem*{T3}{Restatement of Lemma \ref{lemma_optimism}}
\begin{T3}
  \lemmaoptimism
\end{T3}

\begin{proof}
We prove by induction that for any inner iteration $i$ of \VISGO, $Q^{(i)}(s,a) \leq Q^{\star}(s,a)$. By definition we have $Q^{(0)} = 0 \leq Q^{\star}$. Assume that the property holds for iteration $i$, then
\begin{align*}
    Q^{(i+1)}(s,a) &= \max\big\{ \wh{c}(s,a) + \wt P_{s,a} V^{(i)} - b^{(i+1)}(s,a), 0 \big\},
\end{align*}
where
\makeatletter
\newcommand*\mysizeo{%
   \@setfontsize\mysizeo{9.2}{9.2}%
}
\makeatother
\begin{mysizeo}
\begin{align*}
    &\wh{c}(s,a) + \wt P_{s,a} V^{(i)} - b^{(i+1)}(s,a) \\
    &= \wh{c}(s,a) + \wt P_{s,a} V^{(i)} - \max \Big\{ c_1 \sqrt{ \frac{\mathbb{V}(\wt P_{s,a}, V^{(i)}) \iota_{s,a} }{n^+(s,a)}}, \, c_2 \frac{B \iota_{s,a}}{n^+(s,a)}  \Big\} - c_3 \sqrt{\frac{\wh{c}(s,a) \iota_{s,a}}{n^+(s,a)}} - c_4 \frac{B \sqrt{S' \iota_{s,a}}}{n^+(s,a)} \\
    &\myineqi c(s,a) + \wt P_{s,a} V^{(i)} - \max \Big\{ c_1 \sqrt{ \frac{\mathbb{V}(\wt P_{s,a}, V^{(i)}) \iota_{s,a} }{n^+(s,a)}}, \, c_2 \frac{B \iota_{s,a}}{n^+(s,a)}  \Big\} + \frac{28 \iota_{s,a}}{3 n^+(s,a)} - c_4 \frac{B \sqrt{S' \iota_{s,a}}}{n^+(s,a)} \\
    &= c(s,a) + f( \wt P_{s,a}, V^{(i)}, n^+(s,a), B, \iota_{s,a}) + \frac{28 \iota_{s,a}}{3 n^+(s,a)} - c_4 \frac{B \sqrt{S' \iota_{s,a}}}{n^+(s,a)} \\
    &\myineqii c(s,a) + f( \wt P_{s,a}, V^{\star}, n^+(s,a), B, \iota_{s,a}) + \frac{28 \iota_{s,a}}{3 n^+(s,a)} - c_4 \frac{B \sqrt{S' \iota_{s,a}}}{n^+(s,a)} \\
    &\myineqiii c(s,a) + \wt P_{s,a} V^{\star} - 2 \sqrt{ \frac{\mathbb{V}(\wt P_{s,a}, V^{\star}) \iota_{s,a} }{n^+(s,a)}}  - \frac{14 B \iota_{s,a}}{3 n^+(s,a)} - c_4 \frac{B \sqrt{S' \iota_{s,a}}}{n^+(s,a)} \\
    &\myineqiv c(s,a) + \wh P_{s,a} V^{\star} - 2 \sqrt{ \frac{\mathbb{V}(\wt P_{s,a}, V^{\star}) \iota_{s,a} }{n^+(s,a)}} - \frac{14 B \iota_{s,a}}{3 n^+(s,a)} - c_4 \frac{B \sqrt{S' \iota_{s,a}}}{n^+(s,a)} \\
    &\myineqv c(s,a) + P_{s,a} V^{\star} + 2 \sqrt{\frac{ \mathbb{V}(\wh P_{s,a}, V^{\star}) \iota_{s,a}}{ n^+(s,a) }} - 2 \sqrt{ \frac{\mathbb{V}(\wt P_{s,a}, V^{\star}) \iota_{s,a} }{n^+(s,a)}} - (B - B_{\star}) \frac{14 \iota_{s,a}}{3 n^+(s,a)} - c_4 \frac{B \sqrt{S' \iota_{s,a}}}{n^+(s,a)} \\
    &\myineqvi c(s,a) + P_{s,a} V^{\star} + 2 \sqrt{\frac{ \vert \mathbb{V}(\wh P_{s,a}, V^{\star}) - \mathbb{V}(\wt P_{s,a}, V^{\star}) \vert \iota_{s,a}}{ n^+(s,a) }} - (B - B_{\star}) \frac{14 \iota_{s,a}}{3 n^+(s,a)} - c_4 \frac{B \sqrt{S' \iota_{s,a}}}{n^+(s,a)} \\
    &\myineqvii \underbrace{c(s,a) + P_{s,a} V^{\star}}_{= Q^{\star}(s,a)} - (B - B_{\star}) \left( \frac{14 \iota_{s,a}}{3 n^+(s,a)} + \frac{2 \sqrt{ 2 S' \iota_{s,a}}}{n^+(s,a)} \right)\\
    &\le Q^{\star} (s,a),
\end{align*}
\end{mysizeo}
where (i) is by definition of $\mathcal{E}_2$ and choice of $c_3$, (ii) uses the first property of Lem.\,\ref{lemma_properties_f} and the induction hypothesis that $V^{(i)} \leq V^{\star}$, (iii) uses the second property of Lem.\,\ref{lemma_properties_f} and assumption $B\ge \max\{ B_{\star}, 1 \}$, (iv) uses Lem.\,\ref{lemma_properties_wtP}, (v) is by definition of $\mathcal{E}_1$, (vi) uses the inequality $\big| \sqrt{x}-\sqrt{y} \big| \leq \sqrt{ \vert x  - y \vert}, \forall x,y \geq 0$, and (vii) uses the second inequality of Lem.\,\ref{lemma_properties_wtP} and the choice of $c_4$. Ultimately, 
\begin{align*}
    Q^{(i+1)}(s,a) \leq \max\big\{ Q^{\star}(s,a), 0 \big\} = Q^{\star}(s,a).
\end{align*}
\end{proof}

\newtheorem*{T4}{Restatement of Lemma \ref{lemma_wtcalL_convergent}}
\begin{T4}
  \lemmawtcalLconvergent
\end{T4} 

\begin{proof}
We recognize that $V^{(i+1)}(s) \, \leftarrow \, \min_a Q^{(i+1)}(s,a)$, with
\makeatletter
\newcommand*\mysizeq{%
   \@setfontsize\mysizeq{9.5}{9.5}%
}
\makeatother
\begin{mysizeq}
\begin{align*}
    Q^{(i+1)}(s,a) \leftarrow \max \Big\{ \wh{c}(s,a) + \underbrace{ f\big( \wt P_{s,a}, V^{(i)}, n^+(s,a), B, \iota_{s,a} \big)}_{:= g_{s,a}(V^{(i)}) } - c_3 \sqrt{\frac{\wh{c}(s,a) \iota_{s,a}}{n^+(s,a)}} - c_4 \frac{B \sqrt{S' \iota_{s,a}}}{n^+(s,a)},\ 0 \Big\},
\end{align*}
\end{mysizeq}
where we introduce the function $g_{s,a}(V) := f\big( \wt P_{s,a}, V, n^+(s,a), B, \iota_{s,a} \big)$ for notational ease as all other parameters (apart from $V$) will remain the same throughout the analysis.

We prove by induction on the iterations indexed by $i$ that $Q^{(i)} \leq Q^{(i+1)}$. First, $Q^{(0)} = 0 \leq Q^{(1)}$. Now assume that $Q^{(i-1)} \leq Q^{(i)}$. Then
\begin{align*}
    Q^{(i+1)}(s,a) &= \max \left\{ \wh c(s,a) + g_{s,a}(V^{(i)}) - c_3 \sqrt{\frac{\wh{c}(s,a) \iota_{s,a}}{n^+(s,a)}} - c_4 \frac{B \sqrt{S' \iota_{s,a}}}{n^+(s,a)},\ 0 \right\} \\
    &\geq  \max \left\{ \wh c(s,a) + g_{s,a}(V^{(i-1)}) - c_3 \sqrt{\frac{\wh{c}(s,a) \iota_{s,a}}{n^+(s,a)}} - c_4 \frac{B \sqrt{S' \iota_{s,a}}}{n^+(s,a)},\ 0 \right\} \\
    &= Q^{(i)}(s,a),
\end{align*}
where the inequality uses the induction hypothesis $V^{(i)} \geq V^{(i-1)}$ and the fact that $g_{s,a}$ is non-decreasing from the first claim of Lem.\,\ref{lemma_properties_f}.
\end{proof}

\newtheorem*{T5}{Restatement of Lemma \ref{lemma_wtcalL_contraction}}
\begin{T5}
  \lemmawtcalLcontraction
\end{T5} 

\begin{proof}
Take any two vectors $U_1, U_2$, then for any state $s \in \cS$,
\begin{align*}
    \abs{ \wtcalL U_1(s) - \wtcalL U_2(s)} &=  \Big\vert \min_a \wtcalL U_1(s,a) - \min_a \wtcalL U_2(s,a)\Big\vert \\
    &\leq \Big\vert \max_a \Big\{ \wtcalL U_1(s,a) - \wtcalL U_2(s,a) \Big\}  \Big\vert,
\end{align*}
and we have that for any action $a \in \cA$,
\begin{align*}
    \abs{ \wtcalL U_1(s,a) - \wtcalL U_2(s,a)} &\leq \big\vert  \max \big\{ \wh c(s,a) + g_{s,a}(U_1) , ~ 0 \}  -  \max \big\{ \wh c(s,a) + g_{s,a}(U_2) , ~ 0 \} \big\vert \\
    &\leq \big\vert  g_{s,a}(U_1) - g_{s,a}(U_2) \big\vert \\
    &\myineqi \rho_{s,a} \norm{ U_1 - U_2}_{\infty}.
\end{align*}
The third claim of Lem.\,\ref{lemma_properties_f} is employed to justify inequality (i): $g_{s,a}$ is $\rho_{s,a}$-contractive (where $g_{s,a}$ is defined in the proof of Lem.\,\ref{lemma_wtcalL_convergent}) with (recall Eq.\,\ref{eq_def_nu})
\begin{align*}
    \rho_{s,a} := 1 - \wt P_{s,a,g} = 1 - \nu_{s,a}.
\end{align*}
Taking the maximum over $(s,a)$ pairs, $\wtcalL$ is thus $\rho$-contractive with modulus 
$\rho := 1 - \nu <~1$.
\end{proof}

\newtheorem*{T6}{Restatement of Lemma \ref{lemma_bounding_bellman_error}}
\begin{T6}
  \lemmaboundingbellmanerror
\end{T6}

\begin{proof} First we see that $c(s,a) + P_{s,a} V^m - Q^{m}(s,a) \leq c(s,a) + P_{s,a} V^{\star} = Q^{\star}(s,a) \leq B_{\star} + 1$ and that $Q^{m}(s,a) - c(s,a) - P_{s,a} V^m \leq Q^{\star}(s,a) \leq B_{\star} + 1$, from Lem.\,\ref{lemma_optimism} and the Bellman optimality equation (Lem.\,\ref{lemma_wellposedproblem}). Now we prove that $\abs{c(s,a) + P_{s,a} V^m - Q^{m}(s,a)} \leq \beta^m(s,a)$.

\paragraph{Bounding $c(s,a) + P_{s,a} V^m - Q^{m}(s,a)$.}

From the \VISGO loop of Alg.\,\ref{algo}, the vectors $Q^m$ and $V^m$ can be associated to a finite iteration $l$ of a sequence of vectors $(Q^{(i)})_{i \geq 0}$ and $(V^{(i)})_{i \geq 0}$ such that 
\begin{itemize}[leftmargin=.2in,topsep=-1.5pt,itemsep=1pt,partopsep=0pt, parsep=0pt]
    \item[(i)] $Q^m(s,a) := Q^{(l)}(s,a)$,
    \item[(ii)] $V^m(s) := V^{(l)}(s)$,
    \item[(iii)] $\norm{ V^{(l)} - V^{(l-1)} }_{\infty} \leq \epsVI^m$,
    \item[(iv)] {\footnotesize $b^m(s,a) := b^{(l+1)}(s,a) = \max \left\{ c_1 \sqrt{ \frac{\mathbb{V}(\wt P_{s,a}, V^{(l)}) \iota_{s,a} }{n^m(s,a)}} ,\ c_2 \frac{B \iota_{s,a}}{n^m(s,a)}  \right\} + c_3 \sqrt{\frac{\wh{c}^m(s,a) \iota_{s,a}}{n^m(s,a)}} + c_4 \frac{B \sqrt{S' \iota_{s,a}}}{n^m(s,a)}$. }
\end{itemize}
First, we examine the gap between the exploration bonuses at the final \VISGO iterations $l$ and $l+1$ as follows
\begin{align*}
    b^{(l)}(s,a) &\myineqi  c_1 \sqrt{ \frac{\mathbb{V}(\wt P_{s,a}, V^{(l-1)}) \iota_{s,a} }{n^+(s,a)}} + c_2 \frac{B \iota_{s,a}}{n^+(s,a)}  + c_3 \sqrt{\frac{\wh{c}(s,a) \iota_{s,a}}{n^+(s,a)}} + c_4 \frac{B \sqrt{S' \iota_{s,a}}}{n^+(s,a)} \\
    &\myineqii c_1 \sqrt{ 2 \frac{\mathbb{V}(\wt P_{s,a}, V^{(l)}) \iota_{s,a} }{n^+(s,a)}} + c_1 \sqrt{ 2 \frac{\mathbb{V}(\wt P_{s,a}, V^{(l-1)} - V^{(l)}) \iota_{s,a} }{n^+(s,a)}} + c_2 \frac{B \iota_{s,a}}{n^+(s,a)} \\
    &\quad + c_3 \sqrt{\frac{\wh{c}(s,a) \iota_{s,a}}{n^+(s,a)}} + c_4 \frac{B \sqrt{S' \iota_{s,a}}}{n^+(s,a)} \\
    &\myineqiii 2 \sqrt{2} b^{(l+1)}(s,a) + c_1 \sqrt{ \frac{(\epsVI^m)^2 \iota_{s,a} }{2 n^+(s,a)}} \\
    &\leq 2 \sqrt{2} b^{(l+1)}(s,a) + \epsVI^m c_1 \sqrt{\iota_{s,a} / 2},
\end{align*}
where (i) uses $\max\{x,\ y\} \le x+y$; (ii) uses $\mathbb{V}(P,X+Y)\le 2(\mathbb{V}(P,X)+\mathbb{V}(P,Y))$ and $\sqrt{x+y}\le\sqrt{x}+\sqrt{y}$; (iii) uses $x + y \le 2 \max\{x,\ y\}$ and Popoviciu's inequality (Lem.\,\ref{popoviciu}) applied to $V^{(l-1)} - V^{(l)} \in [- \epsVI^m, 0]$. Moreover, we have that $Q^{(l)}(s,a) \geq  \wh c(s,a) + \wt P_{s,a} V^{(l-1)} - b^{(l)}(s,a)$ from Eq.\,\ref{update_Q}. Combining everything yields
\begin{align*}
    - Q^m(s,a) &\leq -\wh c(s,a) - \wt P_{s,a}(V^m - \epsVI) + \epsVI c_1 \sqrt{\iota_{s,a} / 2} + 2 \sqrt{2} b^m(s,a) \\
    &\leq -\wh c(s,a) - \wt P_{s,a}V^m + 2 \sqrt{2} b^m(s,a) + \left(1 + c_1 \sqrt{\iota_{s,a} / 2}\right) \epsVI^m.
\end{align*}
Therefore, we have
\begin{align*}
    &c(s,a) + P_{s,a} V^m - Q^{m}(s,a) \\ &\leq c(s,a) + P_{s,a} V^m - \wh{c}^m(s,a) - \wt{P}_{s,a} V^m + 2 \sqrt{2} b^m(s,a) +  \left(1 + c_1 \sqrt{\iota_{s,a} / 2}\right) \epsVI^m \\
    &\myineqi P_{s,a} V^m - \wh{P}_{s,a} V^m + \frac{B_{\star}}{n^m(s,a)+1} + 4 b^m(s,a) +  \left(1 + c_1 \sqrt{\iota_{s,a} / 2}\right) \epsVI^m  \\
    &\leq \underbrace{ (P_{s,a} - \wh{P}_{s,a}) V^{\star} }_{:= Y_1} + \underbrace{(P_{s,a} - \wh{P}_{s,a})(V^m - V^{\star})}_{:= Y_2} + \frac{B_{\star}}{n^m(s,a)} + 4 b^m(s,a) +  \left(1 + c_1 \sqrt{\iota_{s,a} / 2}\right) \epsVI^m,
\end{align*}
where (i) comes from Lem.\,\ref{lemma_properties_wtP}, the event $\mathcal{E}_2$, Lem.\,\ref{lemma_optimism} and (loosely) bounding $\vert c(s,a) - \wh c(s,a) \vert \leq b^m(s,a)$.
It holds under the event $\mathcal{E}_1$ that
\begin{align*}
    |Y_1| \leq \sqrt{\frac{2\mathbb{V}(P_{s,a},V^{\star}) \iota_{s,a} }{n^m(s,a)}}+\frac{B_{\star} \iota_{s,a} }{n^m(s,a)}.
\end{align*}
Moreover, we have  
\begin{align*}
    \abs{Y_2} &\myeqi \left| \sum_{s'} (\wh P_{s,a,s'}-P_{s,a,s'})(V^m(s')-V^{\star}(s')-P_{s,a}(V^m-V^{\star})) \right| \\
    &\leq \sum_{s'} \abs{P_{s,a,s'}-\wh P_{s,a,s'}} \abs{V^m(s')-V^{\star}(s')-P_{s,a}(V^m-V^{\star})}\\
    &\myineqii \sum_{s'}\sqrt{\frac{2P_{s,a,s'} \iota_{s,a}}{n^m(s,a)}}|V^m(s')-V^{\star}(s')-P_{s,a}(V^m-V^{\star})| + \frac{B_{\star} S'\iota_{s,a}}{n^m(s,a)}\\
    &\myineqiii \sqrt{\frac{2S'\mathbb{V}(P_{s,a},V^m-V^{\star})\iota_{s,a}}{n^m(s,a)}} + \frac{B_{\star} S'\iota_{s,a}}{n^m(s,a)},
\end{align*}
where the shift performed in (i) is by $\sum_{s'}P_{s,a,s'}=\sum_{s'}\wh P_{s,a,s'}=1$; (ii) holds under the event $\mathcal{E}_3$ and Lem.~\ref{lemma_optimism} ($V^m(s) \in [0, B_\star]$); (iii) is by Cauchy-Schwarz inequality.

\paragraph{Bounding $Q^{m}(s,a) - c(s,a) - P_{s,a} V^m$.}

If $Q^m(s,a) = Q^{(l)}(s,a) = 0$, then $Q^{m}(s,a) - \wh c(s,a) - P_{s,a} V^m \leq 0 \leq  \min\big\{ \beta^m(s,a), B_{\star} \big\}$. Otherwise, we have $Q^m(s,a) = Q^{(l)}(s,a) = \wh c(s,a) + \wt P_{s,a} V^{(l-1)} - b^{(l)}(s,a)$. Using that $V^m\ge V^{(l-1)}$ (Lem.\,\ref{lemma_wtcalL_convergent}) and $\wh P_{s,a}V^m\ge \wt P_{s,a}V^m$ (Lem.\,\ref{lemma_properties_wtP}), we get
\begin{align*}
    Q^{m}(s,a) - c(s,a) - P_{s,a} V^m &\le Q^{m}(s,a) - \wh c(s,a) - P_{s,a} V^m + b^m(s,a)\\
    &= \wt P_{s,a} V^{(l-1)} - b^{(l)}(s,a) - P_{s,a} V^m + b^m(s,a) \\
    &\leq \wh P_{s,a} V^{m} - P_{s,a} V^m + b^m(s,a) \\
    &=  (\wh P_{s,a} - P_{s,a}) V^{\star} -  (\wh P_{s,a} - P_{s,a})(V^{\star} - V^m) + b^m(s,a) \\
    &\le \abs{Y_1} + \abs{Y_2} + b^m(s,a),
\end{align*}
which can be bounded as above.
\end{proof}


\subsection{Additional lemmas}

\begin{lemma}\label{tilde_Q}
    Let $\wt Q^m(s,a) := Q^{\star}(s,a) - Q^m(s,a)$ and $\wt V^m(s) := V^{\star}(s) - V^m(s)$. Then conditioned on the event $\mathcal{E}$, we have that for all $(s,a,m,h)$,
    \begin{align*}
        \wt V(s_h^m) - P_{s_h^m,a_h^m} \wt V(s_{h+1}^m) \leq \beta^m(s_h^m,a_h^m).
    \end{align*}
\end{lemma}

\begin{proof}
We write that
\begin{align*}
    \wt V^m(s_h^m) - P_{s_h^m,a_h^m} \wt V^m(s_{h+1}^m) &= V^{\star}(s_h^m) - P_{\samh} V^{\star} + P_{\samh} V^m - V^m(s_h^m) \\ 
    &\leq Q^{\star}(s_h^m,a_h^m) - P_{\samh} V^{\star} + P_{\samh} V^m - V^m(s_h^m) \\ 
    &\myeqi c(\samh) + P_{\samh} V^m - Q^m(s_h^m,a_h^m) \\ 
    &\myineqii \beta^m(s_h^m,a_h^m),
\end{align*}
where (i) uses the Bellman optimality equation (Lem.\,\ref{lemma_wellposedproblem}) and the fact that $V^m(s_h^m) = Q^m(s_h^m,a_h^m)$, and (ii) comes from Lem.\,\ref{lemma_bounding_bellman_error}.
\end{proof}

\begin{lemma}\label{lemma_bound_sum_diff_Vm}
For any $M' \le M$, it holds that 
\begin{align*}
    \sum_{m=1}^{M'} \left( \sum_{h=1}^{H^m} V^m(s_{h}^m) - V^m(s_{h+1}^m) \right) - \sum_{m \in \Mzero(M')} V^m(s_0) \leq 2 S A \log_2 (T_{M'}) \max_{1 \leq m \leq M'} \norm{V^m}_{\infty}.
\end{align*}
\end{lemma}

\begin{proof}
We recall that we denote by $\Mzero(M')$ the set of intervals among the first $M'$ intervals that constitute the first intervals in each episode. From the analytical construction of intervals, an interval $m < M'$ can end due to one of the following three conditions:
\begin{itemize}
    \item[(i)] If interval $m$ ends in the goal state, then 
    \begin{align*}
        V^{m+1}(s^{m+1}_{1}) - V^{m}(s^m_{H^m+1}) = V^{m+1}(s_0) - V^{m}(g) = V^{m+1}(s_0).
    \end{align*}
    This happens for all the intervals $m+1 \in \Mzero(M')$.
    \item[(ii)] If interval $m$ ends when the count to a state-action pair is doubled, then we replan with a \VISGO procedure. Thus we get
    \begin{align*}
        V^{m+1}(s^{m+1}_{1}) - V^{m}(s^m_{H^m+1}) \leq V^{m+1}(s^{m+1}_{1}) \leq \max_{1 \leq m \leq M'} \norm{V^m}_{\infty}.
    \end{align*}
    This happens at most $2 S A \log_2 (T_{M'})$ times.
\end{itemize}
Combining the three conditions above implies that
\begin{mysize}
\begin{align*}
    &\sum_{m=1}^{M'} \left( \sum_{h=1}^{H^m} V^m(s_{h}^m) - V^m(s_{h+1}^m) \right) \\
    &= \sum_{m=1}^{M'} V^m(s_1^m) - V^m(s_{H^m+1}^m) \\
    &= \sum_{m=1}^{M'-1} \left( V^{m+1}(s_1^{m+1}) - V^m(s_{H^m+1}^m) \right)  + \underbrace{ \sum_{m=1}^{M'-1} \left( V^m(s_1^m) - V^{m+1}(s_1^{m+1}) \right)}_{= V^1(s_1^1) - V^{M'}(s_1^{M'})} + V^{M'}(s_1^{M'}) \underbrace{ - V^{M'}(s_{H^{M'}+1}^{M'})}_{\le 0} \\
    &\le \sum_{m=1}^{M'-1} \left( V^{m+1}(s_1^{m+1}) - V^m(s_{H^m+1}^m) \right) + V^1(s_0) \\
    &\leq \sum_{m=1}^{M'-1} V^{m+1}(s_0) \II[m+1 \in \Mzero(M')] + 2 S A \log_2(T_{M'}) \max_{1 \leq m \leq M'} \norm{V^m}_{\infty} + V^1(s_0)\\
    &= \sum_{m \in \Mzero(M')} V^{m}(s_0) + 2 S A \log_2(T_{M'}) \max_{1 \leq m \leq M'} \norm{V^m}_{\infty}.
\end{align*}
\end{mysize}
\end{proof}

\subsection{Full proof of the bound on $X_2(M')$} \label{full_proof_bound_X2}

\paragraph{\ding{172} First, bound $\beta^m$.}~{}
\\

\noindent Recall that we assume that the event $\calE$ holds. From Lem.\,\ref{lemma_bounding_bellman_error}, we have for any $m, s,a$,
\begin{align*}
    \beta^m(s,a) = O \Bigg( & \sqrt{ \frac{\mathbb{V}(\wt P_{s,a}, V^m) \iota_{s,a}}{n^m(s,a)}} + \sqrt{ \frac{\mathbb{V}(P_{s,a}, V^{\star}) \iota_{s,a}}{n^m(s,a)}} + \sqrt{  \frac{S \mathbb{V}(P_{s,a}, V^{\star} - V^m) \iota_{s,a}  }{n^m(s,a)}} \\
    & + \sqrt{\frac{\wh c^m (s,a) \iota_{s,a}}{n^m(s, a)}} + \frac{B_{\star} S \iota_{s,a}}{ n^m(s,a)} + \frac{B \sqrt{S} \iota_{s,a}}{ n^m(s,a)} +  \left(1 + c_1 \sqrt{\iota_{s,a} / 2}\right) \epsVI^m \Bigg).
\end{align*}
Here we interchange $S'$ and $S$ since we use the $O()$ notation.
From Lem.\,\ref{lemma_properties_wtP} and Lem.\,\ref{lemma_optimism}, for any $m,s,a$,
\begin{align*}
    \mathbb{V}(\wt P_{s,a}, V^m)\le \mathbb{V}(\wh P_{s,a}, V^m)+\frac{2B_{\star}^2 S'}{n^m(s,a)+1}<\mathbb{V}(\wh P_{s,a}, V^m)+\frac{2B_{\star}^2 S'}{n^m(s,a)}.
\end{align*}
Under the event $\calE_3$, it holds that
\begin{align*}
    \wh P_{s,a,s'} \leq P_{s,a,s'} + \sqrt{\frac{2P_{s,a,s'}\iota_{s,a}}{n^m(s,a)}}+\frac{\iota_{s,a}}{n^m(s,a)} \leq \frac{3}{2}P_{s,a,s'}+\frac{2\iota_{s,a}}{n^m(s,a)}.
\end{align*}
Thus, it holds that for any $m,s,a$,
\begin{align*}
    \mathbb{V}(\wh P_{s,a}, V^m)&=\sum_{s'}\wh P_{s,a,s'}\left(V^m(s')-\wh P_{s,a}V^m\right)^2\\
    &\myineqi \sum_{s'}\wh P_{s,a,s'}\left(V^m(s')-P_{s,a}V^m\right)^2\\
    &\le \sum_{s'}\left(\frac{3}{2}P_{s,a,s'}+\frac{2\iota_{s,a}}{n^m(s,a)}\right)\left(V^m(s')-P_{s,a}V^m\right)^2\\
    &\le \frac{3}{2}\mathbb{V}(P_{s,a},V^m)+\frac{2B_{\star}^2 S'\iota_{s,a}}{n^m(s,a)}.
\end{align*}
(i) is by the fact that $z^{\star}=\sum_i p_i x_i$ minimizes the quantity $\sum_i p_i(x_i-z)^2$. As a result,
\begin{align*}
    \mathbb{V}(\wt P_{s,a}, V^m) < \frac{3}{2}\mathbb{V}(P_{s,a},V^m) + \frac{2B_{\star}^2 S'}{n^m(s,a)} + \frac{2B_{\star}^2 S' \iota_{s,a}}{n^m(s,a)}.
\end{align*}
Utilizing $\mathbb{V}(P,X+Y)\le 2(\mathbb{V}(P,X)+\mathbb{V}(P,Y))$ with $X = V^{\star} - V^m$ and $Y = V^m$ and $\sqrt{x+y}\le\sqrt{x}+\sqrt{y}$, finally we have
\begin{align*}
    \beta^m(s,a) \le O \Bigg( & \sqrt{ \frac{\mathbb{V}(P_{s,a}, V^m) \iota_{s,a}}{n^m(s,a)}} + \sqrt{  \frac{S \mathbb{V}(P_{s,a}, V^{\star} - V^m) \iota_{s,a}  }{n^m(s,a)}} \\
    & + \sqrt{\frac{\wh c (s,a) \iota_{s,a}}{n^m(s, a)}} + \frac{B_{\star} S \iota_{s,a}}{ n^m(s,a)} + \frac{B \sqrt{S} \iota_{s,a}}{ n^m(s,a)} + \left(1 + c_1 \sqrt{\iota_{s,a} / 2}\right) \epsVI^m \Bigg).
\end{align*}

\paragraph{\ding{173} Second, bound a special type of summation.}~{}
\\

\begin{lemma}\label{lemma_sum_sqrt}
Let $w=\{w_h^m\ge 0\ :\ 1\le m\le M,\ 1\le h\le H^m\}$ be a group of weights, then for any $M' \le M$, 
\begin{align*}
    \sum_{m=1}^{M'} \sum_{h=1}^{H^m} \sqrt{\frac{w_h^m}{n^m(\samh)}}\le O\left(\sqrt{SA\log_2(T_{M'}) \sum_{m=1}^{M'} \sum_{h=1}^{H^m} w_h^m}\right).
\end{align*}
\end{lemma}

\begin{proof}
For $m\le M'$, $n^m(s,a)\in \{2^i\ :\ i\in\mathbb{N}, i\le \log_2(T_{M'})\}$. We can count the occurrences of a fixed value of $n^m(s,a)$ by the doubling property of \VISGO: $\forall i,s,a$
\begin{align*}
    \sum_{m=1}^{M'} \sum_{h=1}^{H^m} \II[(\samh)=(s,a),n^m(s,a)=2^i]\le 2^i.
\end{align*}
Thus
\begin{align}
    \sum_{m=1}^{M'} \sum_{h=1}^{H^m} \frac{1}{n^m(\samh)}&=\sum_{s,a}\sum_{0\le i\le \log_2(T_{M'})} \sum_{m=1}^{M'} \sum_{h=1}^{H^m} \II[(\samh)=(s,a),n^m(s,a)=2^i] \frac{1}{2^i} \notag \\
    &=\sum_{s,a} \sum_{0\le i\le \log_2(T_{M'})} 1 \notag \\
    &\le S A (\log_2(T_{M'}) + 1) \label{sum_one_over_n} \\
    &\le O(S A \log_2(T_{M'})). \notag
\end{align}
By Cauchy-Schwarz inequality,
\begin{align*}
    \sum_{m=1}^{M'} \sum_{h=1}^{H^m} \sqrt{\frac{w_h^m}{n^m(\samh)}} &\le \sqrt{\left( \sum_{m=1}^{M'} \sum_{h=1}^{H^m} w_h^m\right)\left( \sum_{m=1}^{M'} \sum_{h=1}^{H^m} \frac{1}{n^m(\samh)}\right)} \\
    &\le O\left(\sqrt{S A \log_2(T_{M'}) \sum_{m=1}^{M'} \sum_{h=1}^{H^m} w_h^m}\right).
\end{align*}
\end{proof}

By setting successively $w_h^m=\mathbb{V}(P_{\samh}, V^m),\ \mathbb{V}(P_{\samh}, V^{\star} - V^m)$ and $\wh c(\samh)$, and relaxing $\iota_{\samh}$ to its upper-bound $\iota_{M'} = \ln \left( \frac{12 S A S' T_{M'}^2}{\delta} \right)$ we have
\begin{align*}
    X_2(M') \leq O \Bigg( & \sqrt{S A \log_2(T_{M'}) \iota_{M'} \underbrace{\sum_{m=1}^{M'} \sum_{h=1}^{H^m} \mathbb{V}(P_{\samh}, V^m)}_{:= X_4(M')}} \\
    & + \sqrt{S^2 A \log_2(T_{M'}) \iota_{M'} \underbrace{\sum_{m=1}^{M'} \sum_{h=1}^{H^m} \mathbb{V}(P_{\samh}, V^{\star} - V^m)}_{:=X_5(M')}} \\
    & + \sqrt{S A \log_2(T_{M'}) \iota_{M'} \sum_{m=1}^{M'} \sum_{h=1}^{H^m} \wh c (\samh)} + B_{\star} S^2 A \log_2(T_{M'}) \\
    & + B S^{3/2} A \log_2(T_{M'}) \iota_{M'} + \sum_{m=1}^{M'} \sum_{h=1}^{H^m} (1 + c_1 \sqrt{\iota_{M'}/2}) \epsVI^m \Bigg).
\end{align*}

\paragraph{\ding{174} Third, bound each summation separately.}~{}
\\

\noindent \textbf{\emph{Regret contribution of the estimated costs.}} From line \ref{line_emp_costs} in \ALGO, we have that $\wh c(s,a) \le \frac{2 \theta(s, a)}{N(s, a)}$. Let $\theta^m (s,a)$ denote the value of $\theta (s,a)$ for calculating $\wh c^m$. By definition,
\begin{align*}
    \theta^m (\samh) &= \sum_{m' = 1}^{M'} \sum_{h' = 1}^{H^{m'}} \II[(\samh) = (s_{h'}^{m'}, a_{h'}^{m'}),\ n^m(\samh) = 2 n^{m'}(s_{h'}^{m'}, a_{h'}^{m'})] c_{h'}^{m'} \\
    &\quad - {\scriptsize \II[\textup{first occurrence of ($m',h'$) such that $(\samh) = (s_{h'}^{m'}, a_{h'}^{m'}),\ n^m(\samh) = 2 n^{m'}(s_{h'}^{m'}, a_{h'}^{m'})$}]} c_{h'}^{m'} \\
    &\quad + {\scriptsize \II[\textup{first occurrence of ($m',h'$) such that $(\samh) = (s_{h'}^{m'}, a_{h'}^{m'}),\ n^m(\samh) = n^{m'}(s_{h'}^{m'}, a_{h'}^{m'})$}]} c_{h'}^{m'} \\
    & \le \sum_{m' = 1}^{M'} \sum_{h' = 1}^{H^{m'}} \II[(\samh) = (s_{h'}^{m'}, a_{h'}^{m'}),\ n^m(\samh) = 2 n^{m'}(s_{h'}^{m'}, a_{h'}^{m'})] c_{h'}^{m'} + 1.
\end{align*}
For any $M' \le M$ we have
\begin{align*}
    &\sum_{m=1}^{M'} \sum_{h=1}^{H^m} \wh c^m (\samh) \\
    &\le \sum_{m=1}^{M'} \sum_{h=1}^{H^m} \frac{2 \theta^m (\samh)}{n^m (\samh)} \\
    &=  \sum_{m=1}^{M'} \sum_{h=1}^{H^m} \sum_{m' = 1}^{M'} \sum_{h' = 1}^{H^{m'}} \II[(\samh) = (s_{h'}^{m'}, a_{h'}^{m'}),\ n^m(\samh) = 2 n^{m'}(s_{h'}^{m'}, a_{h'}^{m'})] \frac{2 c_{h'}^{m'}}{n^m (\samh)} \\
    &\quad + \sum_{m=1}^{M'} \sum_{h=1}^{H^m} \frac{2}{n^m (\samh)} \\
    &\myineqi \sum_{m' = 1}^{M'} \sum_{h' = 1}^{H^{m'}} \frac{c_{h'}^{m'}}{n^{m'}(s_{h'}^{m'}, a_{h'}^{m'})} \cdot  \sum_{m=1}^{M'} \sum_{h=1}^{H^m} \II[(\samh) = (s_{h'}^{m'}, a_{h'}^{m'}),\ n^m(\samh) = 2 n^{m'}(s_{h'}^{m'}, a_{h'}^{m'})] \\
    &\quad + 2 S A (\log_2 (T_{M'}) + 1) \\
    &\le 2 S A (\log_2 (T_{M'}) + 1) + \sum_{m' = 1}^{M'} \sum_{h' = 1}^{H^{m'}} \frac{c_{h'}^{m'}}{n^{m'}(s_{h'}^{m'}, a_{h'}^{m'})} \cdot 2 n^{m'}(s_{h'}^{m'}, a_{h'}^{m'}) \\
    &= 2 S A (\log_2 (T_{M'}) + 1) + 2 \sum_{m' = 1}^{M'} \sum_{h' = 1}^{H^{m'}} c_{h'}^{m'} \\
    &= 2 S A (\log_2 (T_{M'}) + 1) + 2 C_{M'},
\end{align*}
where (i) comes from Eq.\,\ref{sum_one_over_n}.

\vspace{0.05in}

\noindent \textbf{\emph{Regret contribution of the \textup{\VISGO} precision errors.}} For any $M'\le M$, denote by $J_{M'}$ the (unknown) total number of triggers in the first $M'$ intervals. For $1 \leq j \leq J_{M'}$, denote by $L_j$ the number of time steps elapsed between the $(j-1)$-th and the $j$-th trigger. The doubling condition implies that $L_j \leq 2^j S A$ and that there are at most $J_{M'} = O(S A \log_2 (T_{M'} / (SA)))$ triggers. 
Using that Alg.\,\ref{algo} selects as error $\epsVI^j = 2^{-j} / (SA)$, we have that
\begin{align*}
    \sum_{m=1}^{M'} \sum_{h=1}^{H^m} (1 + c_1 \sqrt{\iota_{M'}/2}) \epsVI^m  &\leq (1 + c_1 \sqrt{\iota_{M'}/2}) \sum_{j=1}^{J_{M'}} L_j \epsVI^j \\
    &\leq (1 + c_1 \sqrt{\iota_{M'}/2}) J_{M'} \\
    &= O\Big(S A \log_2(T_{M'}) \sqrt{\iota_{M'}}\Big).
\end{align*}

\begin{lemma}\label{lemma_bound_X4}
Conditioned on Lem.\,\ref{lemma_bounding_bellman_error}, for a fixed $M' \le M$ with probability $1 - 2 \delta$,
\begin{align*}
    X_4(M') \le O \left( B_{\star} ( C_{M'} + X_2(M') ) + (B_{\star}^2 S A + B_{\star}) (\log_2 (T_{M'}) + \ln(2 / \delta))  \right).
\end{align*}
\end{lemma}

\begin{proof}
We introduce the normalized value function $\ov V^m := V^m / B_{\star} \in [0, 1]$. Define
\begin{align*}
    F(d) := \sum_{m=1}^{M'} \sum_{h=1}^{H^m} ( P_{\samh} (\ov V^m)^{2^d} - (\ov V^m(s_{h+1}^m))^{2^d}),\ G(d) := \sum_{m=1}^{M'} \sum_{h=1}^{H^m} \mathbb{V}(P_{\samh}, (\ov V^m)^{2^d}).
\end{align*}
Then $X_4(M') = B_{\star}^2 G(0)$.
Direct computation gives that
\begin{align*}
    G(d)&= \sum_{m=1}^{M'} \sum_{h=1}^{H^m} \left(P_{\samh}(\ov{V}^m)^{2^{d+1}}-(P_{\samh}(\ov{V}^m)^{2^d})^2\right)\\
    &\myineqi \sum_{m=1}^{M'} \sum_{h=1}^{H^m} \left(P_{\samh}(\ov{V}^m)^{2^{d+1}}-(\ov{V}^m(s_{h+1}^m))^{2^{d+1}}\right) ~ \underbrace{+ \sum_{m=1}^{M'} (\ov{V}^m(s_{H^m + 1}^m))^{2^{d+1}}}_{\le M_1'}\\
    &\quad + \sum_{m=1}^{M'} \sum_{h=1}^{H^m} \left((\ov{V}^m(s_h^m))^{2^{d+1}}-(P_{\samh}\ov{V}^m)^{2^{d+1}}\right) ~ \underbrace{- \sum_{m=1}^{M'} (\ov{V}^m(s_1^m))^{2^{d+1}}}_{\le 0}\\
    &\myineqii F(d+1) +M_1' + 2^{d+1} \sum_{m=1}^{M'} \sum_{h=1}^{H^m} \max\{\ov{V}^m(s_h^m)-P_{\samh}\ov{V}^m,\ 0\}\\
    &= F(d+1) + M_1' + \frac{2^{d+1}}{B_{\star}} \sum_{m=1}^{M'} \sum_{h=1}^{H^m} \max\{Q^m(\samh)-P_{\samh}V^m,\ 0\}\\
    &\myineqiii F(d+1) + M_1' + \frac{2^{d+1}}{B_{\star}} \sum_{m=1}^{M'} \sum_{h=1}^{H^m} (c(\samh)+\beta^m(\samh))\\
    &= F(d + 1) + M_1' + \frac{2^{d + 1}}{B_{\star}} \sum_{m=1}^{M'} \sum_{h=1}^{H^m} (c_h^m + \beta^m(\samh) + (c(\samh) - c_h^m))\\
    &\le F(d + 1) + M_1' + \frac{2^{d + 1}}{B_{\star}} ( C_{M'} + X_2(M') + |X_3(M')|),
\end{align*}
where $M_1'$ denotes the number of intervals satisfying $\ov{V}^m(s_{H^m + 1}^m) \ne 0$; (i) is by convexity of $f(x)=x^{2^d}$; (ii) is by Lem.\,\ref{diff_of_exp}; (iii) is by Lem.\,\ref{lemma_bounding_bellman_error}.

For a fixed $d$, $F(d)$ is a martingale. By taking $c=1$ in Lem.\,\ref{martingale_bound}, we have
\begin{align*}
    \mathbb{P}\left[F(d)>2\sqrt{2G(d)(\log_2 (T_{M'}) + \ln(2 / \delta))} + 5 (\log_2 (T_{M'}) + \ln(2 / \delta))\right] \le \delta.
\end{align*}
Taking $\delta' = \delta / (\log_2 (T_{M'}) + 1)$, using $x\ge \ln (x) + 1$ and finally swapping $\delta$ and $\delta'$, we have that
\begin{align*}
    \mathbb{P}\left[F(d)>2\sqrt{2G(d)(2 \log_2 (T_{M'}) + \ln(2 / \delta))} + 5 (2 \log_2 (T_{M'}) + \ln(2 / \delta))\right] \le \frac{\delta}{\log_2 (T_{M'}) + 1}.
\end{align*}
Taking a union bound over $d=1,2,\ldots,\log_2(T_{M'})$, we have that with probability $1 - \delta$,
\begin{align*}
    F(d) \myineqi & 2\sqrt{2 (2 \log_2 (T_{M'}) + \ln(2 / \delta))}\cdot \sqrt{F(d + 1) + 2^{d + 1} \cdot \frac{C_{M'} + X_2(M') + |X_3(M')|}{B_{\star}}} \\
    &+ 5 (2 \log_2 (T_M) + \ln(2 / \delta)) + 2\sqrt{2 (2 \log_2 (T_{M'}) + \ln(2 / \delta)) M_1'}.
\end{align*}
From Lem.\,\ref{recursion_bound_order}, taking $\lambda_1 = T_{M'},\ \lambda_2 = 2\sqrt{2 (2 \log_2 (T_{M'}) + \ln(2 / \delta))},\ \lambda_3 = (C_{M'} + X_2(M') + |X_3(M')|) / B_{\star},\ \lambda_4 = 5 (2 \log_2 (T_M) + \ln(2 / \delta)) + 2\sqrt{2 (2 \log_2 (T_{M'}) + \ln(2 / \delta)) M_1'}$, we have that
\begin{align*}
     F(1) \le O \left( \log_2 (T_{M'}) + \ln(2 / \delta) + \frac{C_{M'} + X_2(M') + |X_3(M')|}{B_{\star}} + M_1' \right).
\end{align*}
Hence
\begin{align*}
    X_4(M') 
    \le O \left( B_{\star} (C_{M'} + X_2(M') + |X_3(M')|) + B_{\star}^2 ( \log_2 (T_{M'}) + \ln(2 / \delta) + M_1') \right).
\end{align*}
By definition, $M_1' \le O (S A \log_2 (T_{M'}))$ since only those intervals ending by triggering the doubling condition are taken into account. From the bound of $|X_3(M')|$, the following holds with probability $1 - 2 \delta$:
\begin{align*}
    X_4(M') \le O \left( B_{\star} ( C_{M'} + X_2(M') ) + (B_{\star}^2 S A + B_{\star}) (\log_2 (T_{M'}) + \ln(2 / \delta)) \right).
\end{align*}
Throughout the proof, the inequality $O(\sqrt{x y}) \le O(x + y)$ is utilized to simplify the bound.
\end{proof}

\begin{lemma}\label{lemma_bound_X5}
Conditioned on Lem.\,\ref{lemma_bounding_bellman_error}, for a fixed $M' \le M$ with probability $1 - \delta$,
\begin{align*}
    X_5(M')\le O \left( B_{\star}^2 S A (\log_2 (T_{M'}) + \ln(2 / \delta)) + B_{\star} X_2(M')\right).
\end{align*}
\end{lemma}

\begin{proof}
We introduce the normalized quantity $\ov{\wt V}^m := \wt V^m / B_{\star} \in [-1, 1]$ (recall the definition in Lem.\,\ref{tilde_Q}). Define
\begin{align*}
    \wt F(d) := \sum_{m=1}^{M'} \sum_{h=1}^{H^m} ( P_{\samh} (\ov{\wt V}^m)^{2^d} - (\ov{\wt V}^m(s_{h+1}^m))^{2^d}),\ \wt G(d) := \sum_{m=1}^{M'} \sum_{h=1}^{H^m} \mathbb{V}(P_{\samh}, (\ov{\wt V}^m)^{2^d}).
\end{align*}
Then $X_5(M') = \wt G(0) B_{\star}^2$.
Direct computation gives that
\begin{align*}
    \wt G(d)&= \sum_{m=1}^{M'} \sum_{h=1}^{H^m} \left(P_{\samh}(\ov{\wt V}^m)^{2^{d+1}}-(P_{\samh}(\ov{\wt V}^m)^{2^d})^2\right)\\
    &\le \sum_{m=1}^{M'} \sum_{h=1}^{H^m} \left(P_{\samh}(\ov{\wt V}^m)^{2^{d+1}}-(\ov{\wt V}^m(s_{h+1}^m))^{2^{d+1}}\right) ~ \underbrace{+ \sum_{m=1}^{M'} (\ov{\wt V}^m(s_{H^m + 1}^m))^{2^{d+1}}}_{\le \wt M_1'}\\
    &\quad + \sum_{m=1}^{M'} \sum_{h=1}^{H^m} \left((\ov{\wt V}^m(s_h^m))^{2^{d+1}}-(P_{\samh}\ov{\wt V}^m)^{2^{d+1}}\right) ~\underbrace{ - \sum_{m=1}^{M'} (\ov{\wt V}^m(s_1^m))^{2^{d+1}}}_{\le 0}\\
    &\le \wt F(d+1) + \wt M_1' + 2^{d+1} \sum_{m=1}^{M'} \sum_{h=1}^{H^m} \max\{\ov{\wt V}^m(s_h^m)-P_{\samh}\ov{\wt V}^m,\ 0\}\\
    &=\wt F(d+1) + \wt M_1' + \frac{2^{d+1}}{B_{\star}} \sum_{m=1}^{M'} \sum_{h=1}^{H^m} \max\{\wt V^m(s_h^m)-P_{\samh}\wt V^m,\ 0\}\\
    &\myineqi \wt F(d+1) + \wt M_1' + \frac{2^{d+1}}{B_{\star}} \sum_{m=1}^{M'} \sum_{h=1}^{H^m} \beta^m(\samh)\\
    &= \wt F(d+1) + \wt M_1' + \frac{2^{d+1}}{B_{\star}}X_2(M'),
\end{align*}
where $\wt M_1'$ denotes the number of intervals satisfying $\ov{\wt V}^m(s_{H^m + 1}^m) \ne 0$; (i) come from Lem.\,\ref{tilde_Q}.

For a fixed $d$, $\wt F(d)$ is a martingale. By taking $c=1$ in Lem.\,\ref{martingale_bound}, we have
\begin{align*}
    \mathbb{P}\left[\wt F(d) > 2\sqrt{2 \wt G(d) (\log_2 (T_{M'}) + \ln(2 / \delta))} + 5 (\log_2 (T_{M'}) + \ln(2 / \delta))\right] \le \delta.
\end{align*}
Taking $\delta' = \delta / (\log_2 (T_{M'}) + 1)$, using $x\ge \ln (x) + 1$ and finally swapping $\delta$ and $\delta'$, we have that
\begin{align*}
    \mathbb{P}\left[\wt F(d) > 2\sqrt{2 \wt G(d)(2 \log_2 (T_{M'}) + \ln(2 / \delta))} + 5 (2 \log_2 (T_{M'}) + \ln(2 / \delta))\right] \le \frac{\delta}{\log_2 (T_{M'}) + 1}.
\end{align*}
Taking a union bound over $d=1,2,\ldots,\log_2 (T_{M'})$, we have that with probability $1 - \delta$,
\begin{align*}
    \wt F(d) \le & 2\sqrt{2 (2 \log_2 (T_{M'}) + \ln(2 / \delta))}\cdot \sqrt{\wt F(d + 1) + 2^{d + 1} \frac{X_2(M')}{B_{\star}}} \\
    &+ 5 (2 \log_2 (T_{M'}) + \ln(2 / \delta)) + 2\sqrt{2 (2 \log_2 (T_{M'}) + \ln(2 / \delta)) \wt M_1'}.
\end{align*}
From Lem.\,\ref{recursion_bound_order}, taking $\lambda_1 = T_{M'},\ \lambda_2 = 2\sqrt{2 (2 \log_2 (T_{M'}) + \ln(2 / \delta))},\ \lambda_3 = X_2(M') / B_{\star},\ \lambda_4 = 5 (2 \log_2 (T_{M'}) + \ln(2 / \delta)) + 2\sqrt{2 (2 \log_2 (T_{M'}) + \ln(2 / \delta)) \wt M_1'}$, we have that
\begin{align*}
    \wt F(1) \le O \left( \log_2 (T_{M'}) + \ln(2 / \delta) + \frac{X_2(M')}{B_{\star}} + \wt M_1' \right).
\end{align*}
Since $V^{\star}(g) - V^m(g) = 0 - 0 = 0$, similar as bounding $M_1'$, we have $\wt M_1' \le O (S A \log_2 (T_{M'}))$. Hence with probability $1 - \delta$, we have
\begin{align*}
    X_5(M') \le O \left( B_{\star}^2 S A (\log_2 (T_{M'}) + \ln(2 / \delta)) + B_{\star} X_2(M')\right).
\end{align*}
Throughout the proof, the inequality $O(\sqrt{x y}) \le O(x + y)$ is utilized to simplify the bound.
\end{proof}

\paragraph{\ding{175} Finally, bind them together.}~{}
\\

\noindent Let $\ov \iota_{M'} := \ln \left(\frac{ 12 S A S' T_{M'}^2 }{ \delta }\right) + \log_2 ((\max\{B_{\star},1\})^2 T_{M'}) + \ln \left(\frac{ 2 }{ \delta }\right)$ be the upper bound of all previous log terms. 
\begin{align*}
    X_2(M') &\le  O\Big( \sqrt{S A X_4(M')} \ov \iota_{M'} + \sqrt{S^2 A X_5(M') } \ov \iota_{M'} \\
    &\quad \quad ~ + SA \ov{\iota}_{M'}^{3/2} + \sqrt{S A C_{M'}} \ov \iota_{M'} + B_{\star} S^2 A  \ov{\iota}_{M'}^2 + B S^{3/2} A \ov{\iota}_{M'}^2 \Big), \\
    X_4(M') &\le O \left( B_{\star} ( C_{M'} + X_2(M') ) + (B_{\star}^2 S A + B_{\star}) \ov \iota_{M'} \right), \\
    X_5(M') &\le O \left( B_{\star}^2 S A \ov \iota_{M'} + B_{\star} X_2(M') \right).
\end{align*}
This implies that
\begin{align*}
    X_2(M') &\myineqi O\left(\sqrt{B_{\star} S^2 A } \ov \iota_{M'} \cdot \sqrt{X_2(M')} + (\sqrt{B_{\star}} + 1 ) \sqrt{ S A C_{M'} } \ov \iota_{M'} + B S^2 A \ov{\iota}_{M'}^2\right)\\
    &\le O\left(\max\left\{ \sqrt{B_{\star} S^2 A } \ov \iota_{M'} \cdot \sqrt{X_2(M')},\ (\sqrt{B_{\star}} + 1 ) \sqrt{ S A C_{M'} } \ov \iota_{M'} + B S^2 A \ov{\iota}_{M'}^2 \right\}\right),
\end{align*}
where (i) uses the assumption $B \ge \max\{ B_{\star}, 1\}$ to simplify the bound.
Considering terms in $\max\{\}$ separately, we obtain two bounds:
\begin{align*}
    X_2(M') &\le O( B_{\star} S^2 A \ov{\iota}_{M'}^2 ),\\
    X_2(M') &\le O( (\sqrt{B_{\star}} + 1 ) \sqrt{ S A C_{M'} } \ov \iota_{M'} + B S^2 A \ov{\iota}_{M'}^2 ).
\end{align*}
By taking the maximum of these bounds, we have
\begin{align*}
    X_2(M') &\le O( (\sqrt{B_{\star}} + 1 ) \sqrt{ S A C_{M'} } \ov \iota_{M'} + B S^2 A \ov{\iota}_{M'}^2 ).
\end{align*}

\section{Technical Lemmas}

\begin{lemma}[Bennett's Inequality, anytime version]\label{bennett}
Let $Z,Z_1,\ldots,Z_n$ be i.i.d. random variables with values in $[0,b]$ and let $\delta>0$. Define $\mathbb{V}[Z]=\mathbb{E}[(Z-\mathbb{E}[Z])^2]$. Then we have
\begin{align*}
    \mathbb{P}\left[ \forall n \geq 1, ~ \left|\mathbb{E}[Z]-\frac{1}{n}\sum_{i=1}^n Z_i\right|>\sqrt{\frac{2\mathbb{V}[Z] \ln(4n^2/\delta)}{n}}+\frac{b \ln(4n^2/\delta)}{n}\right]\le \delta.
\end{align*}
\end{lemma}

\begin{proof}
From Bennett's inequality, if the variables have values in $[0, 1]$, then for a specific $n \geq 1$, 
\begin{align*}
    \mathbb{P}\left[\left|\mathbb{E}[Z]-\frac{1}{n}\sum_{i=1}^n Z_i\right|>\sqrt{\frac{2\mathbb{V}[Z] \ln(2/\delta)}{n}}+\frac{\ln(2/\delta)}{n}\right]\le \delta.
\end{align*}
We then choose $\delta \leftarrow \frac{\delta}{2 n^2}$ and take a union bound over all possible values of $n \geq 1$, and the result follows given that $\sum_{n\geq 1} \frac{\delta}{2 n^2} < \delta$. To account for the case $b \neq 1$ we apply the result to $(Z_n / b)$.
\end{proof}

\begin{lemma}[Theorem 4 in \cite{maurer2009empirical}, anytime version]\label{concen_bound_empirical}
Let $Z,Z_1,\ldots,Z_n\ (n\ge 2)$ be i.i.d. random variables with values in $[0,b]$ and let $\delta>0$. Define $\bar{Z}=\frac{1}{n}Z_i$ and $\hat V_n=\frac{1}{n}\sum_{i=1}^n (Z_i-\bar Z)^2$. Then we have
\begin{align*}
    \mathbb{P}\left[\forall n \geq 1, ~ \left|\mathbb{E}[Z]-\frac{1}{n}\sum_{i=1}^n Z_i\right|>\sqrt{\frac{2\hat V_n \ln(4n^2/\delta)}{n-1}}+\frac{7 b \ln(4n^2/\delta)}{3(n-1)}\right]\le \delta.
\end{align*}
\end{lemma}

\begin{lemma}[Popoviciu's Inequality]\label{popoviciu}
Let $X$ be a random variable whose value is in a fixed interval $[a,b]$, then $\mathbb{V}[X]\le \frac{1}{4}(b - a)^2$.
\end{lemma}

\begin{lemma}[Lemma 11 in \cite{zhang2020model}]\label{martingale_bound_primal}
Let $(M_n)_{n\ge 0}$ be a martingale such that $M_0=0$ and $|M_n-M_{n-1}|\le c$ for some $c>0$ and any $n\ge 1$. Let $\mathrm{Var}_n=\sum_{k=1}^n \mathbb{E}[(M_k-M_{k-1})^2|\mathcal{F}_{k-1}]$ for $n\ge 0$, where $\mathcal{F}_k=\sigma(M_1,\ldots,M_k)$. Then for any positive integer $n$ and any $\epsilon,\delta>0$, we have that
\begin{align*}
    \mathbb{P}\left[|M_n|\ge 2\sqrt{2\mathrm{Var}_n\ln(1/\delta)}+2\sqrt{\epsilon\ln(1/\delta)}+2c\ln(1/\delta)\right]\le 2\left(\log_2\left(\frac{nc^2}{\epsilon}\right)+1\right)\delta.
\end{align*}
\end{lemma}

\begin{lemma}\label{martingale_bound}
Let $(M_n)_{n\ge 0}$ be a martingale such that $M_0=0$ and $|M_n-M_{n-1}|\le c$ for some $c>0$ and any $n\ge 1$. Let $\mathrm{Var}_n=\sum_{k=1}^n \mathbb{E}[(M_k-M_{k-1})^2|\mathcal{F}_{k-1}]$ for $n\ge 0$, where $\mathcal{F}_k=\sigma(M_1,\ldots,M_k)$. Then for any positive integer $n$ and $\delta \in (0, 2 (n c^2)^{1 / \ln 2}]$, we have that
\makeatletter
\newcommand*\mysizetwo{%
   \@setfontsize\mysizetwo{9.3}{9.3}%
}
\makeatother
\begin{mysizetwo}
\begin{align*}
    \mathbb{P} \left[ |M_n| \ge 2\sqrt{2 \mathrm{Var}_n (\log_2 (n c^2) + \ln(2 / \delta))} + 2 \sqrt{ \log_2 (n c^2) + \ln(2 / \delta) } + 2 c (\log_2 (n c^2) + \ln(2 / \delta)) \right] \le \delta.
\end{align*}
\end{mysizetwo}
\end{lemma}

\begin{proof}
Take $\epsilon=1$ and $\delta' = 2 (\log_2(n c^2) + 1) \delta$ in Lem.\,\ref{martingale_bound_primal}. By $x \ge \ln(x) + 1$, we have
\begin{align*}
    \ln (1 / \delta) = \ln (2 (\log_2(n c^2) + 1) / \delta' )
    = \ln (\log_2 (n c^2) + 1) + \ln (2 / \delta')
    \le \log_2 (n c^2) + \ln (2 / \delta').
\end{align*}
Hence,
\makeatletter
\newcommand*\mysizetwo{%
   \@setfontsize\mysizetwo{9.3}{9.3}%
}
\makeatother
\begin{mysizetwo}
\begin{align*}
    &\mathbb{P} \left[ |M_n| \ge 2\sqrt{2 \mathrm{Var}_n (\log_2 (n c^2) + \ln(2 / \delta'))} + 2 \sqrt{ \log_2 (n c^2) + \ln(2 / \delta') } + 2 c (\log_2 (n c^2) + \ln(2 / \delta')) \right]\\
    &\le \mathbb{P}\left[|M_n|\ge 2\sqrt{2\mathrm{Var}_n\ln(1/\delta)}+2\sqrt{\ln(1/\delta)}+2c\ln(1/\delta)\right]\\
    &\le \delta'.
\end{align*}
\end{mysizetwo}
By swapping $\delta$ and $\delta'$ we complete the proof.
\end{proof}

\begin{lemma}[Lemma 11 in \cite{zhang2020reinforcement}]\label{recursion_bound}
Let $\lambda_1,\lambda_2,\lambda_4\ge 0,\ \lambda_3\ge 1$ and $i'=\log_2 \lambda_1$. Let $a_1,a_2,\ldots,a_{i'}$ be non-negative reals such that $a_i\le \lambda_1$ and $a_i\le \lambda_2\sqrt{a_{i+1}+2^{i+1}\lambda_3}+\lambda_4$ for any $1\le i\le i'$. Then we have that $a_1\le \max\{(\lambda_2+\sqrt{\lambda_2^2+\lambda_4})^2,\lambda_2\sqrt{8\lambda_3}+\lambda_4\}$.
\end{lemma}

\begin{lemma}\label{recursion_bound_order}
Let $\lambda_1,\lambda_2,\lambda_4\ge 0,\ \lambda_3\ge 1$ and $i'=\log_2 \lambda_1$. Let $a_1,a_2,\ldots,a_{i'}$ be non-negative reals such that $a_i\le \lambda_1$ and $a_i\le \lambda_2\sqrt{a_{i+1}+2^{i+1}\lambda_3}+\lambda_4$ for any $1\le i\le i'$. Then we have that $a_1\le O(\lambda_2^2 + \lambda_3 + \lambda_4)$.
\end{lemma}
\begin{proof}
Since $\max\{a,\ b\} \le a + b$ and $2 a b \le a^2 + b^2$ for any choice of non-negative $a$ and $b$, we can transform the result of Lem.\,\ref{recursion_bound} into 
\begin{align*}
    a_1 &\le \max \left \{ \left( \lambda_2+\sqrt{\lambda_2^2+\lambda_4} \right)^2,\lambda_2\sqrt{8\lambda_3}+\lambda_4 \right\} \\
    &\le O \left( \left( \lambda_2+\sqrt{\lambda_2^2+\lambda_4} \right)^2 + \lambda_2\sqrt{8\lambda_3}+\lambda_4 \right) \\
    &\le O ( \lambda_2^2 + \lambda_2^2+\lambda_4 + \lambda_2^2 + \lambda_3 + \lambda_4 ) \\
    &\le O (\lambda_2^2 + \lambda_3 + \lambda_4).
\end{align*}
\end{proof}

\begin{lemma}\label{var_le_exp}
For random variable $Z\in [0,1],\ \mathbb{V}[Z]\le \mathbb{E}[Z]$.
\end{lemma}
\begin{proof}
$\mathbb{V}[Z]=\mathbb{E}[Z^2]-(\mathbb{E}[Z])^2\le \mathbb{E}[Z^2]\le \mathbb{E}[Z].$
\end{proof}

\begin{lemma}\label{diff_of_exp}
For any $a,b\in [0,1]$ and $k\in \mathbb{N},\ a^k - b^k\le k \max\{a-b,\ 0\}$.
\end{lemma}
\begin{proof}
$a^k - b^k = (a - b) \sum_{i=0}^{k-1} a^i b^{k-1 - i}\le \max\{a - b,\ 0\} \cdot \sum_{i=0}^{k-1} 1 = k \max\{a-b,\ 0\}$.
\end{proof}

\begin{lemma}\label{quadratic_ineq_bound}
For $a,b,x\ge 0,\ x\le a\sqrt{x}+b$ implies $x\le (a+\sqrt{b})^2$.
\end{lemma}
\begin{proof}
$x\le a\sqrt{x}+b \Rightarrow x\le \left(\frac{a+\sqrt{a^2+b}}{2}\right)^2 \le (a+\sqrt{b})^2$. 
\end{proof}

\section{Computational Complexity of \ALGO}  \label{app_comput_complex}

Here we complement Remarks \ref{remark_comput_complexity_main} and \ref{rmk_comput_complexity_contraction} on the computational complexity of \ALGO (Alg.\,\ref{algo}).

The computational complexity of a \VISGO procedure can be bounded as $O(\frac{S^2 A}{1 - \rho} \log(B_{\star} / \epsVI))$ (assuming for simplicity that $B_{\star} \geq 1$, otherwise replace $\max\{B_{\star}, 1\} \leftarrow B_{\star}$). By the fact that total number of \VISGO procedure is bounded by $O(S A \log T)$, we derive $\log(B_{\star}/\epsVI) = O(S A \log(B_{\star} T))$ by choice of $\epsVI$. As a result, the total computational complexity for \ALGO is $O(T S^2 A \cdot S A \log(B_{\star} T) \cdot S A \log T)$, which is polynomially bounded and in particular near-linear in $T$. Also note that $T$ is bounded polynomially w.r.t.\,$K$ as shown in the various cases of Sect.\,\ref{subsect_bounding_T}. Indeed, in the case of positive costs lower bounded by $c_{\min} > 0$, Cor.\,\ref{thm_regret_bound_cmin_positive} entails that $T \leq c_{\min}^{-1} K V^{\star}(s_0) + c_{\min}^{-1} \wt{O}\big( B_{\star} \sqrt{ S A K } + B_{\star} S^2 A\big)$. In the general cost case, the cost perturbation trick is applied and the minimum cost becomes $K^{-n}$ for Cor.\,\ref{thm_regret_bound_cmin_zero_unknown_Tstar} or $(\ov T_{\star} K)^{-1}$ for Cor.\,\ref{thm_regret_bound_cmin_zero_priorknowledge_Tstar}, i.e., $c_{\min}^{-1}$ depends polynomially on $K$.

We note that the analysis of the computational complexity of \ALGO may likely be refined. Indeed, we see that i) on the one hand, if $n(s,a)$ is small, then the optimistic skewing of $\wt P_{s,a}$ is not too small so the probability of reaching the goal from $(s,a)$ is not too small (so the associated contraction modulus is bounded away from $1$) and ii) on the other hand, if $n(s,a) \rightarrow + \infty$, then $\wt P_{s,a} \rightarrow \wh P_{s,a} \rightarrow P_{s,a}$, so to the limit we should recover the convergence properties of VI of the optimal Bellman operator under the true model, which by assumption admits a proper policy in~$P$. Thus we see that studying further the ``intermediate regime'' may bring into the picture the computational complexity of running VI in the true model, yet this is not our main focus here, as our complexity analysis is sufficient to ensure the computational efficiency of \ALGO. 



\section{Unknown $B_{\star}$: Parameter-Free \ALGO}\label{sect_unknown_Bstar}

In this section, we relax the assumption that (an upper bound of) $B_{\star}$ is known to \ALGO. In Alg.\,\ref{algobstar} we propose a parameter-free \ALGO that bypasses the requirement $B \geq B_{\star}$ (line \ref{line-B-Bstar} of Alg.\,\ref{algo}) to tune the exploration bonus. As in Sect.\,\ref{sect_main_results} we consider for ease of exposition that $B_{\star} \geq 1$. We structure the section as follows: App.\,\ref{app_unknownBstar_scheme} presents our algorithm and provides intuition, App.\,\ref{app_unknownBstar_result} spells out its regret guarantee, and App.\,\ref{app_unknownBstar_proof} gives its proof.

\subsection{Algorithm and Intuition } \label{app_unknownBstar_scheme}

Parameter-free \ALGO (Alg.\,\ref{algobstar}) initializes an estimate $\wt B = 1$ and decomposes the time steps into \textit{phases}, indexed by $\phi$. The execution of a phase is reported in the subroutine \PHASE (Alg.\,\ref{algophase}). Given any estimate $\wt B$, a subroutine \PHASE has the same structure as Alg.\,\ref{algo}, up to two key differences:
\begin{itemize}[leftmargin=.2in,topsep=-1.5pt,itemsep=1pt,partopsep=0pt, parsep=0pt]
    \item \textit{\textbf{Halting due to exceeding cumulative cost.}} \PHASE tracks the cumulative cost within the current phase, and terminates whenever it exceeds a threshold $C_{\textup{bound}}$ (Eq.\,\ref{c_bound}) that depends on $\wt B$, $S$, $A$, $\delta$ and the current episode and time indexes $k$ and $t$, which are all computable quantities to the agent.
    \item \textit{\textbf{Halting due to exceeding \textup{\VISGO} range.}} During each \VISGO procedure, \PHASE tracks the range of the value function $V^{(i)}$ at each \VISGO iteration $i$, and terminates if $\norm{V^{(i)}}_{\infty} > \wt B$.
\end{itemize}
The estimate $\wt B$ can be incremented in two different ways and speeds:
\begin{itemize}[leftmargin=.2in,topsep=-1.5pt,itemsep=1pt,partopsep=0pt, parsep=0pt]
    \item \textit{\textbf{Doubling increment of $\wt B$.}} On the one hand, whenever a phase ends (i.e., one of the two halting conditions above is met), $\wt B$ is doubled ($\wt B \leftarrow 2 \wt B$). 
    \item \textit{\textbf{Episode-driven increment of $\wt B$.}} On the other hand, at the beginning of each new episode $k$, the estimate is automatically increased to $\wt B \leftarrow \max\{ \wt B, \, \sqrt{k} / (S^{3/2} A^{1/2}) \}$.
\end{itemize}
We now explain the rationale behind our scheme:
\begin{itemize}[leftmargin=.2in,topsep=-1.5pt,itemsep=1pt,partopsep=0pt, parsep=0pt]
    \item \textit{Reason for episode-driven increment of $\wt B$.} The fact that $\wt B$ grows as a function of $k$ implies that at some (unknown) point it will hold that $\wt B \geq B_{\star}$ for large enough $k$. This will enable us to recover the analysis and the regret bound of Thm.\,\ref{key_lemma_Tdependent}. 
    \item \textit{Reason for doubling increment of $\wt B$.} The doubling increment comes into play whenever a phase terminates due to an exceeding cumulative cost or \VISGO range. At this point, the agent becomes aware that $\wt B$ is too small and thus it doubles it. It is crucial to allow intra-episode increments of $\wt B$ to avoid getting \textit{stuck} in an episode with an underestimate $\wt B < B_{\star}$.
    \item \textit{Reason for cumulative cost halting.} The cost threshold $C_{\textup{bound}}$ is designed so that (w.h.p.) it can be exceeded at most once in the case of $\wt B \geq B_{\star}$, and so that it can serve as a tight enough bound on the regret in the case of $\wt B < B_{\star}$.
    \item \textit{Reason for \textup{\VISGO} range halting.} The threshold $\wt B$ on the range of the \VISGO value functions is chosen so that (w.h.p.) it is never exceeded in the case of $\wt B \geq B_{\star}$, and so that it can serve as a guarantee of finite-time near-convergence of a \VISGO procedure (i.e., the contraction property) in the case of $\wt B < B_{\star}$.
\end{itemize}

\subsection{Regret Guarantee of Parameter-Free \ALGO} \label{app_unknownBstar_result}

Parameter-free \ALGO satisfies the following guarantee (which extends Thm.\,\ref{key_lemma_Tdependent} to unknown $B_{\star}$).

\newtheorem*{T0}{Restatement of Theorem \ref{thm_key_unknownBstar}}
\begin{T0}
  \thmkeyunknownBstar
\end{T0}

As a result, parameter-free \ALGO is able to circumvent the knowledge of $B_{\star}$ at the cost of only logarithmic and lower-order terms.

\subsection{Proof of Theorem \ref{thm_key_unknownBstar}} \label{app_unknownBstar_proof}

We begin by defining notations and concepts exclusively used in this section:
\begin{itemize}[leftmargin=.2in,topsep=-1.5pt,itemsep=1pt,partopsep=0pt, parsep=0pt]
\item $C_t$ denotes the cumulative cost up to time step $t$ (included) that is accumulated in the execution of the subroutine \PHASE in which time step $t$ belongs. Importantly, note that the cumulative cost $C_t$ is initialized to $0$ at the beginning of each \PHASE (line \ref{line_reinitialize_cost_tracker} of Alg.\,\ref{algophase}). Also note that re-planning (i.e., a \VISGO procedure) occurs whenever the estimate $\wt B$ is changed.
\item Denote by $t_m$ the time step at the end of the current interval $m$, and by $k_m$ the episode in which the time step $t_m$ belongs. $\wt{B}_m$ denotes the value of $\wt B$ at time step $t_m$. $C_m$ denotes $C_{t_m}$, i.e., the cumulative cost up to interval $m$ (included) in the execution of the \PHASE in which interval $m$ belongs.
\end{itemize}

Unlike \ALGO of Alg.\,\ref{algo}, the parameter-free version has an increasing $\wt B$ throughout the process. To utilize the regret bounds (Thm.\,\ref{key_lemma_Tdependent} and Eq.\,\ref{regret_bound_interval}) in the case of $\wt B \geq B_{\star}$, slight modifications are needed to be applied to the algorithm and some lemmas.


\noindent \textbf{\emph{Modification to \textup{\ALGO}.}} Previously, \ALGO accepted a single value $B \ge \max\{B_{\star},\ 1\}$ to compute the bonuses in Eq.\,\ref{update_bonus}. To satisfy the same regret bound when $\wt B$ changes, we require \ALGO to accept a series of $B_k$ for $k \in \mathbb{N}^+$, such that $\max\{B_{\star},\ 1\} \le B_k \le B$ for any $k$. In any episode $k$, the analysis simply substitutes $B_k$ for $B$ in Eq.\,\ref{update_bonus}.

\noindent \textbf{\emph{Modifications to the proofs of Lem.\,\ref{lemma_optimism}, \ref{lemma_wtcalL_convergent} and \ref{lemma_bounding_bellman_error}.}} In the original version of the proofs, we proved the lemmas for any update of value functions, without mentioning any time relevant variables. Now since $B$ relies on episode $k$, the modified proofs need to incorporate the changes. Suppose that we are examining $Q(s,a),\ V(s),\ b(s,a)$ and $\beta(s,a)$ for any state-action pair $(s,a) \in \cS \times \cA$ in episode $k$. Lem.\,\ref{lemma_optimism} and Lem.\,\ref{lemma_wtcalL_convergent} utilize the property stated in Lem.\,\ref{lemma_properties_f}, and the $B$ in Lem.\,\ref{lemma_properties_f} is a parameter that is able to vary each time step we utilize Lem.\,\ref{lemma_properties_f}. Thus, in the proofs of Lem.\,\ref{lemma_optimism}, \ref{lemma_wtcalL_convergent} and \ref{lemma_bounding_bellman_error}, all the $B$'s are substituted with $B_k$'s to ensure that these lemmas are compatible with our modified setting. 

\noindent \textbf{\emph{Modification to the proof of bounding $\beta^m$ in App.\,\ref{full_proof_bound_X2}.}} Suppose that interval $m$ is in episode $k$ and recall that $B_k \le B$, then
\begin{align*}
    b^m(s,a) &= \max \left\{ c_1 \sqrt{ \frac{\mathbb{V}(\wt P_{s,a}, V^{(l)}) \iota_{s,a} }{n^m(s,a)}} ,\ c_2 \frac{B_k \iota_{s,a}}{n^m(s,a)}  \right\} + c_3 \sqrt{\frac{\wh{c}^m(s,a) \iota_{s,a}}{n^m(s,a)}} + c_4 \frac{B_k \sqrt{S' \iota_{s,a}}}{n^m(s,a)} \\
    &\le O \left( \sqrt{ \frac{\mathbb{V}(\wt P_{s,a}, V^{(l)}) \iota_{s,a} }{n^m(s,a)}} + \frac{B \iota_{s,a}}{n^m(s,a)} + \sqrt{\frac{\wh{c}^m(s,a) \iota_{s,a}}{n^m(s,a)}} + \frac{B \sqrt{S \iota_{s,a}}}{n^m(s,a)} \right).
\end{align*}
Combining the above bound of $b^m(s,a)$ with Lem.\,\ref{lemma_bounding_bellman_error}, we get that the bound of $\beta^m$ in App.\,\ref{full_proof_bound_X2} is unchanged.


Equipped with the slight modifications mentioned above, we now derive two key properties on which the analysis of parameter-free \ALGO relies:



\noindent \textbf{\emph{Property 1: Optimism avoids the first halting condition.}} Let us study any phase starting with estimate $\wt B \ge B_{\star}$. From Eq.\,\ref{regret_bound_interval} (which is the interval-generalization of Thm.\,\ref{key_lemma_Tdependent}), for a fixed initial state $s_0$ and a fixed interval $m$, the cumulative cost can be bounded with probability $1 - \delta$ by 
\begin{align}
    k_m V^{\star}(s_0) + x \left( B_{\star} \sqrt{ S A k_m } \log_2 \left( \frac{B_{\star} t_m S A}{\delta} \right) + \wt{B}_m S^2 A \log_2^2\left( \frac{B_{\star} t_m S A}{\delta} \right) \right),  \label{eqeqeq_cost}
\end{align}
where $x > 0$ is a large enough absolute constant (which can be retraced in the analysis leading to Eq.\,\ref{regret_bound_interval}). By scaling $\delta \gets \delta / (2 S t_m^2)$ for each $m \le M$, we have the following cumulative cost bound that holds for any initial state in $\cS$ and any interval $m \le M$, with probability $1 - \delta$,
\begin{align*} 
    C_m &\le k_m V^{\star}(s_0) + x \left( B_{\star} \sqrt{ S A k_m } \log_2 \left( \frac{B_{\star} t_m S A \cdot 2 S t_m^2}{\delta} \right) + \wt{B}_m S^2 A \log_2^2\left( \frac{B_{\star} t_m S A \cdot 2 S t_m^2}{\delta} \right) \right) \\
    &\le k_m B_{\star} + 3 x \left( B_{\star} \sqrt{ S A k_m } \log_2 \left( \frac{B_{\star} t_m S A}{\delta} \right) + \wt{B}_m S^2 A \log_2^2\left( \frac{B_{\star} t_m S A}{\delta} \right) \right).
\end{align*}
Since we are in the case of $\wt{B}_m \ge B_{\star}$, we have
\begin{align} \label{eq_bound_C_unknownB}
    C_m \le k_m \wt{B}_m + 3 x \left( \wt{B}_m \sqrt{S A k_m } \log_2 \left( \frac{\wt{B}_m t_m S A}{\delta} \right) + \wt{B}_m S^2 A \log_2^2\left( \frac{\wt{B}_m t_m S A}{\delta} \right) \right).
\end{align}
Since costs are non-negative, for any $t \le t_m$, we have $C_t \le C_m$ hence $C_t$ must also satisfy the bound of Eq.\,\ref{eq_bound_C_unknownB}. There remains to predict the values of $k_m,\ t_m,\ \wt{B}_m$, given the current $k_{\textup{cur}},\ t_{\textup{cur}},\ \wt{B}_{\textup{cur}}$. The upper bounds for $k_m$ and $\wt{B}_m$ are $k_{\textup{cur}}$ and $\wt{B}_{\textup{cur}}$ respectively, since they can only be incremented when reaching the goal $g$, which is a condition for ending the current interval. The upper bound for $t_m$ can be derived using the pigeonhole principle: since $t_{\textup{cur}} = \sum_{(s,a) \in \cS \times \cA} n(s,a)$, we know that $2t_{\textup{cur}} > \sum_{(s,a) \in \cS \times \cA} (2n(s,a) - 1)$. Thus by time step $2t_{\textup{cur}}$ there must exist a trigger condition, which is a condition for ending the current interval. Hence, by replacing $k_m \leftarrow k_{\textup{cur}}$, $\wt B_m \leftarrow \wt{B}_{\textup{cur}}$ and $t_m \leftarrow 2 t_{\textup{cur}}$ in Eq.\,\ref{eq_bound_C_unknownB}, we get, with probability at least $1-\delta$, that the cumulative cost within a phase that starts with $\wt B \geq B_{\star}$ has the following anytime upper bound 
\begin{align*}
    C_{t_{\textup{cur}}} \le k_{\textup{cur}} \wt{B}_{\textup{cur}} + 3 x \left( \wt{B}_{\textup{cur}} \sqrt{S A k_{\textup{cur}}} \log_2 \left( \frac{2 \wt{B}_{\textup{cur}} t_{\textup{cur}} S A}{\delta} \right) + \wt{B}_{\textup{cur}} S^2 A \log_2^2 \left( \frac{2 \wt{B}_{\textup{cur}} t_{\textup{cur}} S A}{\delta} \right) \right).
\end{align*}
Note that this bound corresponds exactly to the cumulative cost threshold $C_{\textup{bound}}$ in Eq.\,\ref{c_bound}. 
This means that with probability at least $1-\delta$, the first halting condition cannot be met in a phase that starts with $\wt B \geq B_{\star}$. 


\noindent \textbf{\emph{Property 2: Optimism avoids the second halting condition.}} Let us consider the case of $\wt B \ge B_{\star}$ whenever the algorithm re-plans (i.e., running \VISGO procedure). The proof of Lem.\,\ref{lemma_optimism} ensures that at any iteration, $\norm{V^{(i)}}_{\infty}\le B_{\star} \le \wt B$, so the second halting condition is never met under the same high-probability event as above.

\noindent \textbf{\emph{Implications.}} The two properties above indicate that, if a phase starts with estimate $\wt B \geq B_{\star}$, with probability at least $1 - \delta$, this phase will never halt due to the two halting conditions (it can only terminate if it completes the final episode $K$), and Alg.\,\ref{algobstar} will thus never enter a new phase. Due to the doubling increment of $\wt B$ every time a phase ends, we can therefore bound the total number of phases as $\Phi \le \lceil \log_2 (B_{\star}) \rceil + 1$.

\noindent \textbf{\emph{Analysis.}} We now split the analysis of the regret contributions of the episodes in two \textit{regimes}. To this end, let $\kappa_{\star} := \lceil B_{\star}^2 S^3 A \rceil$ denote a special episode (note that it is unknown to the learner since it depends on $B_{\star}$). We consider that the high-probability event mentioned above holds (which is the case with probability at least~$1-\delta$). Recall that at the beginning of each episode $k$, the algorithm sets $\wt B \gets \max\{\wt B,\ \sqrt{k} / (S^{3/2} A^{1/2}) \}$. 

\paragraph{\ding{172} Regret contribution in the first regime (i.e., episodes $k < \kappa_{\star}$).}~{}

\vspace{0.05in}
\noindent We denote respectively by $R_{1 \rightarrow \kappa_{\star}}$ and $C_{1 \rightarrow \kappa_{\star}}$ the cumulative regret and the cumulative cost incurred by the algorithm before episode $\kappa_{\star}$ begins. For any phase $\phi$, we denote by 
\begin{itemize}[leftmargin=.2in,topsep=-1.5pt,itemsep=1pt,partopsep=0pt, parsep=0pt]
    \item $C^{(\phi)}_{1 \to \kappa_{\star}}$ the cumulative cost incurred during the time steps that are \textit{both} in phase $\phi$ and in an episode~$k < \kappa_{\star}$;
    
    \item $k^{(\phi)}$ the episode when phase $\phi$ ends;
    
    \item $t^{(\phi)}$ the time step when phase $\phi$ ends;
    
    \item $\wt{B}^{(\phi)}$ the value of $\wt B$ at the end of phase $\phi$.

\end{itemize}
Observe that 
\vspace{-0.1in}
\begin{align*}
    C_{1 \rightarrow \kappa_{\star}} = \sum_{\phi=1}^{\Phi} C^{(\phi)}_{1 \to \kappa_{\star}}.
\end{align*}

\vspace{-0.16in}
Now, by definition of $\kappa^{\star}$, the episode-driven increment of $\wt B$ never exceeds $B_{\star}$, unless $\wt B$ is already larger or equal to $B_\star$ at the beginning of the phase. But Property 1 ensures that if $\wt B \ge B_{\star}$ in the beginning of a phase, then $\wt B$ will never be doubled afterwards. Hence, we are guaranteed that within the episodes $k < \kappa_{\star}$, the final value of the estimate $\wt B$ is at most $2 B_{\star}$.

Since \PHASE tracks the cumulative cost at each step using the threshold in Eq.\,\ref{c_bound} and since $c_t \le 1$, by the fact that $C_{\textup{bound}}$ is monotonously increasing with respect to $t$, we have that for any phase $\phi$,
\makeatletter
\newcommand*\mysizeu{%
   \@setfontsize\mysizeu{8.85}{9}%
}
\makeatother
\begin{mysizeu}
\begin{align*}
    C^{(\phi)}_{1 \to \kappa_{\star}} &\le k^{(\phi)} \wt{B}^{(\phi)} + 3 x \left( \wt{B}^{(\phi)} \sqrt{S A k^{(\phi)}} \log_2 \left( \frac{2 \wt{B}^{(\phi)} t^{(\phi)} S A}{\delta} \right) + \wt{B}^{(\phi)} S^2 A \log_2^2 \left( \frac{2 \wt{B}^{(\phi)} t^{(\phi)} S A}{\delta} \right) \right) + 1 \\
    &\le \kappa_{\star} (2 B_{\star}) + 3 x \left( (2 B_{\star}) \sqrt{S A \kappa_{\star}} \log_2 \left( \frac{2 (2 B_{\star}) T S A}{\delta} \right) + (2 B_{\star}) S^2 A \log_2^2 \left( \frac{2 (2 B_{\star}) T S A}{\delta} \right) \right) + 1 \\
    &\le O\left( B_{\star}^3 S^3 A +  B_{\star}^2 S^2 A \log \left( \frac{ B_{\star} T S A}{\delta} \right) + B_{\star} S^2 A \log^2 \left( \frac{B_{\star} T S A}{\delta} \right)  \right).
\end{align*}%
\end{mysizeu}%
In addition, we recall that $\Phi \le \lceil \log_2 (B_{\star}) \rceil + 1$. Hence, by plugging in the definition of $\kappa^{\star}$, we can bound the cost (and thus the regret) accumulated over the episodes $k < \kappa_{\star}$ as follows
\makeatletter
\newcommand*\mysizev{%
   \@setfontsize\mysizev{9.5}{9.5}%
}
\makeatother
\begin{mysizev}%
\begin{align*}
    R_{1 \to \kappa_{\star}} \leq C_{1 \to \kappa_{\star}}  &\leq \sum_{\phi=1}^{\lceil \log_2 (B_{\star}) \rceil + 1} O\left( B_{\star}^3 S^3 A +  B_{\star}^2 S^2 A \log \left( \frac{ B_{\star} T S A}{\delta} \right) + B_{\star} S^2 A \log^2 \left( \frac{B_{\star} T S A}{\delta} \right)  \right)  \\
    &\le O\Big( B_{\star}^{3} S^3 A \log(B_{\star}) + B_{\star}^{2} S^2 A \log \left( \frac{B_{\star} T S A}{\delta} \right) \log(B_{\star}) \\ 
    &\quad \quad + B_{\star} S^2 A \log^2 \left( \frac{B_{\star} T S A}{\delta} \right) \log(B_{\star}) \Big) \\
    &\le O\left( B_{\star}^{3} S^3 A \ov \iota + B_{\star}^{2} S^2 A \ov{\iota}^2 + B_{\star} S^2 A \ov{\iota}^3 \right).
\end{align*}%
\end{mysizev}%
\paragraph{\ding{173} Regret contribution in the second regime (i.e., episodes $k \geq \kappa_{\star}$).}~{}

\vspace{0.05in}

\noindent We denote respectively by $R_{\kappa_{\star} \to K}$ and $C_{\kappa_{\star} \to K}$ the cumulative regret and the cumulative cost incurred during the episodes $k \geq \kappa^{\star}$. By definition of $\kappa^{\star}$, the episode-driven increment of $\wt B$ ensures that $\wt B \ge B_{\star}$. During this second regime there may be at most two phases: one that started at an episode $k < \kappa_{\star}$ (i.e., in the first regime) and that overlaps the two regimes, and one starting after that (note that properties 1 and 2 ensure that at this point neither halting condition can end this phase since it started with estimate $\wt B \geq B_{\star}$, thus it lasts until the end of the learning interaction). In addition, we can upper bound $\wt B$ as follows
\vspace{-0.07in}
\begin{align*}
    \wt{B} \le \max\Big\{2B_{\star},\ \frac{2 \sqrt{K}}{S^{3/2} A^{1/2}} \Big\}.
\end{align*}
We now introduce a fourth condition of stopping an interval to the analysis performed in Sect.\,\ref{subsect_intervaldecomp_notation}: (4) an interval ends when a subroutine \PHASE ends. This implies that the policy always stays the same within an interval when running Alg.\,\ref{algobstar}. Condition (4) is met at most once in the second regime.

We now focus on only the second regime: we re-index intervals by $1,2,\ldots,M'$ and let $T_m$ denote the time step counting from the beginning of $\kappa_{\star}$ to the end of interval $m$. To bound $R_{\kappa_{\star} \to K}$, we need to adapt the proofs in App.\,\ref{subsect_regret_decomposition} and App.\,\ref{full_proof_bound_X2} to be compatible with our new interval decomposition. Concretely, there are two slight modifications in the analysis of the second regime:
\begin{itemize}[leftmargin=.2in,topsep=-1.5pt,itemsep=1pt,partopsep=0pt, parsep=0pt]
    
    \item Statistics: For any statistic (i.e., $N(s,a,s'),\ \theta(s,a)$ and $\wh c(s,a)$ for any $(s,a,s')\in \cS \times \cA \times \cS'$), instead of learning from scratch, \PHASE reuses all samples collected thus far. This difference does not affect the regret bound and the probability, since it can be viewed by taking a partial sum of terms in $\wt{R}_{M'}$.
    
    \item The regret decomposition: In the proof of Lem.\,\ref{lemma_bound_sum_diff_Vm}, we need to incorporate condition (4) which is met at most once during the second regime. It falls into case (ii) in the proof of Lem.\,\ref{lemma_bound_sum_diff_Vm}, which thus happens at most $2 S A \log_2 (T_{M'}) + 1$ times, and the regret decomposition should be
    \begin{align*}
        \wt{R}_{M'} \le X_1(M') + X_2(M') + X_3(M') + 2 B_{\star} S A \log_2 (T_{M'}) + B_{\star}.
    \end{align*}
    
\end{itemize}
Hence by incorporating these slight modifications in the proof of Thm.\,\ref{key_lemma_Tdependent}, we get probability at least~$1 - \delta$, 
\begin{align*}
    R_{\kappa_{\star} \to K} &\le O \left( B_{\star} \sqrt{S A K} \log\left( \frac{B_{\star} T S A}{\delta} \right) + S^2 A \wt{B}_{M'} \log^2\left( \frac{B_{\star} T S A}{\delta} \right) \right)\\
    &\le O \left( B_{\star} \sqrt{S A K} \log\left( \frac{B_{\star} T S A}{\delta} \right) + S^2 A \frac{\sqrt{K}}{S^{3/2} A^{1/2}} \log^2\left( \frac{B_{\star} T S A}{\delta} \right) \right)\\
    &\le O \left( B_{\star} \sqrt{S A K} \ov \iota + \sqrt{S A K} \ov{\iota}^2 \right).
\end{align*}

\paragraph{\ding{174} Combining the regret contributions in the two regimes.}~{}

\vspace{0.05in}

\noindent The overall regret is bounded with probability at least $1 - \delta$ by
\begin{align*}
    R_K = R_{1 \to \kappa_{\star}} + R_{\kappa_{\star} \to K} \le O\left( B_{\star} \sqrt{S A K} \ov \iota + \sqrt{S A K} \ov{\iota}^2 + B_{\star}^{3} S^3 A \ov \iota + B_{\star}^{2} S^2 A \ov{\iota}^2 + B_{\star} S^2 A \ov{\iota}^3 \right).
\end{align*}
There remains to plug in the definition of $\ov{\iota}$. Denote by $T$ the cumulative time within the $K$ episodes and by $R^{\star}_K$ the regret after $K$ episodes of~\textup{\ALGO} in the case of known $B_{\star}$ (i.e., the bound of Thm.\,\ref{key_lemma_Tdependent} with $B = B_{\star}$). Then with probability at least $1 - \delta$ the regret of parameter-free \ALGO can be bounded as 
\begin{align*}
    R_K &= O \left( R^{\star}_K + \sqrt{S A K} \log^{2}\left( \frac{B_{\star} S A T }{\delta} \right) + B_{\star}^3 S^3 A \log^{3}\left( \frac{B_{\star} S A T }{\delta} \right) \right) \\
    &= O \left( R^{\star}_K \log\left( \frac{B_{\star} S A T }{\delta} \right) + B_{\star}^3 S^3 A \log^{3}\left( \frac{B_{\star} S A T }{\delta} \right) \right).
\end{align*}

This concludes the proof of Thm.\,\ref{thm_key_unknownBstar}.

\vspace{0.1in}

\begin{remark}
    At a high level, our analysis to circumvent the knowledge of $B_{\star}$ boils down to the following argument: if the estimate is too small, we bound the regret by the cumulative cost; otherwise if it is large enough, we recover the regret bound under a known upper bound on $B_\star$. Interestingly, this somewhat resembles the reasoning behind the schemes for unknown SSP-diameter $D$ in the adversarial SSP algorithms of \citet[][App.\,I]{rosenberg2020stochastic} and \citet[][App.\,E]{chen2021finding} (recall that $D := \max_{s \in \cS} \min_{\pi \in \Pi_{\textup{proper}}} T^{\pi}(s)$ and that $B_{\star} \leq D \leq T_{\star}$). Note however that these schemes change their algorithms' structure: whenever the agent is in a state that is insufficiently visited, it executes the Bernstein-SSP algorithm of \citet{rosenberg2020near} with unit costs until the goal is reached. In other words, these schemes first learn to reach the goal (regardless of the costs) and then focus on minimizing the costs to goal. In contrast, our scheme for unknown $B_{\star}$ targets the original SSP objective from the start and it does \textit{not} fundamentally alter our algorithm \ALGO with known $B_{\star}$. Indeed, the only addition of parameter-free \ALGO is a \textit{dual tracking} of the cumulative costs and \VISGO ranges, and a \textit{careful increment} of the estimate $\wtB$ in the bonus. Finally, our scheme only adds ``horizon-free'' lower-order terms (i.e., $B_{\star}, S, A$) as shown in Thm.\,\ref{thm_key_unknownBstar}, as opposed to the aforementioned schemes that introduce a lower-order dependence on the SSP-diameter $D$, which may be much larger than $B_{\star}$. 
\end{remark}

\begin{algorithm}[!t]
        \small
        \DontPrintSemicolon
        \textbf{Input:} $\cS,\ s_0 \in \cS,\ g \not\in \cS,\ \cA,\ \delta$. \\
        \textbf{Optional input:} cost perturbation $\eta \in [0, 1]$. \\
        Set up \textbf{global constants:} $\cS,\ \cA,\ s_0 \in \cS,\ g \not\in \cS,\ \eta$.\\
        Set up \textbf{global variables:} $t,\ j,\ N(),\ n(),\ \wh P,\ \theta(),\ \wh c(),\ Q(),\ V()$. \\
        Set estimate $\wt B \gets 1$. \\
        Set current starting state $s_{\text{start}} \gets s_0$. \\
        Set $t\gets 1,\ k\gets 1,\ j \gets 0$. \\
        For $(s,a,s') \in \cS \times \cA \times \cS'$, set $N(s,a) \leftarrow 0; ~ n(s,a) \leftarrow 0; ~ N(s,a,s') \leftarrow 0; ~ \wh P_{s,a,s'} \leftarrow 0; ~ \theta(s,a) \leftarrow 0; ~ \wh c (s,a) \leftarrow 0; ~ Q(s,a) \leftarrow 0; ~ V(s) \leftarrow 0$. \\
        Set phase counter $\phi \gets 1$. \\
        \While{\textup{True}} {
            Set $s_{\textup{cur}},\ \wt{B}_{\textup{cur}},\ k_{\textup{cur}} \gets$ \PHASE($s_{\textup{start}},\ \wt B,\ k$) ~ (Alg.\,\ref{algophase}). \\
            \algocomment{~\textup{\PHASE} halts because of $B_{\star}$ underestimation, entering a new \textbf{phase}} \\
            Set $s_{\text{start}} \gets s_{\textup{cur}},\ k \gets k_{\textup{cur}},\ \wt B \gets 2 \wt{B}_{\text{cur}}$, and increment phase index $\phi \gets \phi + 1$. \\
        }
        \caption{Algorithm for unknown $B_{\star}$: Parameter-free \ALGO}
        \label{algobstar}
\end{algorithm}

\begin{algorithm}
        \small
        \DontPrintSemicolon
        \textbf{Input:} $s_{\text{start}} \in \cS,\ \wt B,\ k$. \\
        \textbf{Global constants:} $\cS,\ \cA,\ s_0 \in \cS,\ g \not\in \cS,\ \eta$.\\
        \textbf{Global variables:} $t,\ j,\ N(),\ n(),\ \wh P,\ \theta(),\ \wh c(),\ Q(),\ V()$. \\
        \textbf{Specify:} Trigger set $\mathcal{N} \leftarrow \{ 2^{j-1}\ :\ j=1,2,\ldots\}$. Constants $c_1 = 6,\ c_2 = 36,\ c_3 = 2 \sqrt{2},\ c_4 = 2 \sqrt{2}$. Large enough absolute constant $x > 0$ (so that Eq.\,\ref{eqeqeq_cost} holds, see App.\,\ref{app_unknownBstar_proof}). \\ 
        Set $C \gets 0$. \label{line_reinitialize_cost_tracker}  \algocomment{~Reinitialize cumulative cost tracker} \\
        \For{\textup{episode} $k_{\textup{cur}}=k, k + 1, \ldots$}
        {
            \If{$\sqrt{k_{\textup{cur}}} / (S^{3/2} A^{1/2}) > \wt B$} {
                Set $\wt B \gets \sqrt{k_{\textup{cur}}} / (S^{3/2} A^{1/2})$, and set $j\gets j + 1,\ \epsVI \gets 2^{-j} / (S A)$. \\
                Info, $Q,\ V \gets$ \VISGO($\wt{B},\ \epsVI$). \\
                \If{\textup{Info = Fail}} {
                    \algocomment{~Second halting condition: \textup{\VISGO} range exceeds threshold} \\
                    \Return $s_t,\ \wt B,\ k_{\textup{cur}}$.
                }
            }
            \vspace{-0.1in}
            Set $s_t \gets \left\{\begin{array}{ll}
                s_{\text{start}}, & k_{\textup{cur}} = k,\\
                s_0, & \textup{otherwise}.
            \end{array}\right.$ \\
        \While{$s_t \neq g$}{
        Take action $a_t = \argmin_{a \in \cA} Q(s_t, a)$, incur cost $c_t$ and observe next state $s_{t+1} \sim P(\cdot \vert s_t, a_t)$.\\
        Set $(s,a,s',c) \leftarrow (s_t, a_t, s_{t+1}, \max\{c_t, \eta\})$ and $t \leftarrow t + 1$.\\
        Set $N(s,a) \leftarrow N(s,a) + 1$, $\theta(s,a) \leftarrow \theta(s,a) + c$, $C \gets C + c$, $N(s,a,s') \leftarrow N(s,a,s') + 1$, and set
        \vspace{-0.1in}
        \begin{align}
            C_{\textup{bound}} \gets k_{\textup{cur}} \wt B + 3 x \left( \wt B \sqrt{S A k_{\textup{cur}}} \log_2 \left( \frac{2 \wt B t S A}{\delta} \right) + \wt B S^2 A \log_2^2 \left( \frac{2 \wt B t S A}{\delta} \right) \right). \label{c_bound}
        \end{align}\\
        \vspace{-0.1in}
        \If{$C > C_{\textup{bound}}$} {\label{first_hald_cond}
            \algocomment{~First halting condition: cumulative cost exceeds threshold} \\
            \Return $s_t,\ \wt B,\ k_{\textup{cur}}$.
        }
        \If{$N(s,a) \in \mathcal{N}$}{
        Set $\wh{c}(s,a) \leftarrow \mathds{I}[N(s,a) \geq 2] \frac{2 \theta(s,a)}{N(s,a)} + \mathds{I}[N(s,a) = 1] \theta(s,a)$ and $\theta(s,a) \leftarrow 0$. \\
        For all $s' \in \mathcal{S}$, set $\wh P_{s,a,s'} \leftarrow N(s,a,s') / N(s,a)$, $n(s,a) \leftarrow N(s,a)$, and set $j \leftarrow j + 1,\ \epsVI \leftarrow 2^{-j} / (S A)$. \\
        Info, $Q,\ V \gets$ \VISGO($\wt{B},\ \epsVI$). \\
        \If{\textup{Info = Fail}} {
            \algocomment{~Second halting condition: \textup{\VISGO} range exceeds threshold} \\
            \Return $s_t,\ \wt B,\ k_{\textup{cur}}$.
        }
        }
        }
        }
        \caption{Subroutine \PHASE}
        \label{algophase}
\end{algorithm}

\begin{algorithm}[b!]
        \small
        \DontPrintSemicolon
        \textbf{Inputs:} $\wt{B},\ \epsVI$. \\
        \textbf{Global constants:} $\cS,\ \cA,\ s_0 \in \cS,\ g \not\in \cS,\ \eta$.\\
        \textbf{Global variables:} $t,\ j,\ N(),\ n(),\ \wh P,\ \theta(),\ \wh c(),\ Q(),\ V()$. \\
        For all $(s,a,s') \in \cS \times \cA \times \cS'$, set
        \begin{align*} 
            \wt P_{s,a,s'} &\leftarrow \frac{n(s,a)}{n(s,a)+1} \wh P_{s,a,s'} + \frac{\mathds{I}[s'=g]}{n(s,a)+1}.
        \end{align*}\\
        For all $(s,a) \in \cS \times \cA$, set $n^+(s,a) \leftarrow \max\{n(s,a), 1\},\ \iota_{s,a} \gets \ln \left( \frac{12 S A S' [n^+ (s,a)]^2}{\delta} \right)$.\\
        Set $i \leftarrow 0$, $V^{(0)} \leftarrow 0$, $V^{(-1)} \leftarrow +\infty$. \\
        \While{$\norm{ V^{(i)} - V^{(i-1)}}_{\infty} > \epsVI$}
        { 
        For all $(s,a) \in \cS \times \cA$, set 
        \begin{align}
            &b^{(i+1)}(s,a) \, \leftarrow \, \max \Big\{ c_1 \sqrt{ \frac{\mathbb{V}(\wt P_{s,a}, V^{(i)}) \iota_{s,a} }{n^+(s,a)}} , \, c_2 \frac{\wt B \iota_{s,a}}{n^+(s,a)}  \Big\} + c_3 \sqrt{\frac{\wh{c}(s,a) \iota_{s,a}}{n^+(s,a)}} + c_4 \frac{\wt B \sqrt{S' \iota_{s,a}}}{n^+(s,a)}, \\
            &Q^{(i+1)}(s,a) \, \leftarrow \, \max\big\{ \wh{c}(s,a) \,+ \, \wt P_{s,a} V^{(i)}  ~ - \,b^{(i+1)}(s,a), ~ 0 \big\}, \\
            &V^{(i+1)}(s) \, \leftarrow \, \min_a Q^{(i+1)}(s,a).
        \end{align}\\
        Set $V^{(i + 1)}(g) \gets 0$ and $i \leftarrow i+1$. \\
        \If{$\norm{V^{(i)}}_{\infty} > \wt B$} {\label{second_hald_cond}
            \algocomment{~Second halting condition: \textup{\VISGO} range exceeds threshold} \\
            \Return Fail, $Q^{(i)},\ V^{(i)}$.
        }
        }
        \Return Success, $Q^{(i)},\ V^{(i)}$.
        \caption{Subroutine \VISGO}
        \label{plan}
\end{algorithm}

\end{document}